\documentclass[]{interact}

\usepackage{epstopdf}
\usepackage[caption=false]{subfig}

\usepackage[numbers,sort&compress]{natbib}
\bibpunct[, ]{[}{]}{,}{n}{,}{,}
\renewcommand\bibfont{\fontsize{10}{12}\selectfont}

\theoremstyle{plain}
\newtheorem{theorem}{Theorem}[section]
\newtheorem{lemma}[theorem]{Lemma}
\newtheorem{corollary}[theorem]{Corollary}
\newtheorem{proposition}[theorem]{Proposition}
\newtheorem{assumption}[theorem]{Assumption}
\theoremstyle{definition}
\newtheorem{definition}[theorem]{Definition}
\newtheorem{example}[theorem]{Example}
\newcommand{\tr}{\mathrm{tr}}
\theoremstyle{remark}
\newtheorem{remark}{Remark}

\usepackage{mathtools}
\usepackage{comment}
\usepackage{mathrsfs}
\usepackage{hyperref}
\usepackage{bbm}
\begin{document}


\title{A Hyperfinite Framework for Score-Based Generative Modeling}

\author{
\name{Sunder Ram Krishnan\thanks{CONTACT S.~R. Krishnan. Email: eeksunderram@gmail.com}}
\affil{Kerala School of Mathematics, Kozhikode 673571, Kerala, India.}
}

\maketitle

\begin{abstract}
Score-based diffusion models are typically formulated using continuous-time stochastic differential equations and measure-theoretic stochastic calculus. In this paper, we develop a hyperfinite formulation of score-based generative modeling within the framework of Nonstandard Analysis. Starting from an internal diffusion process on a hyperfinite grid, we derive the associated infinitesimal generator and establish its correspondence with the classical Fokker--Planck equation. We then obtain a hyperfinite backward-mean identity that yields the reverse-time drift and provides a constructive derivation of the reverse-time SDE.\\
Building on these results, we show that minimization of an internal score-matching objective recovers the score function required by the reverse-time dynamics, thereby connecting score estimation with generative sampling directly at the hyperfinite level. Under suitable assumptions, we further derive a hyperfinite Girsanov formula and establish a relationship between likelihood optimization and Fisher-divergence objectives. Finally, we analyze the second-order consistency of the hyperfinite dynamics and show that the leading correction term depends explicitly on the fourth moment of the increment distribution, with the Gaussian value $\kappa=3$ eliminating the leading dispersion contribution.\\
Taken together, these results provide a unified hyperfinite framework for diffusion-based generative modeling--while laying foundations for further extensions--that links discrete grid dynamics, reverse-time diffusion, score matching, and likelihood-based formulations within a common nonstandard setting.
\end{abstract}


\section{Introduction}

The mathematical architecture of modern stochastic calculus is traditionally built upon the sophisticated machinery of measure theory and the Itô-isometry. Since the seminal work of \cite{robinson1974non}, Nonstandard Analysis (NSA) has offered a powerful alternative to this continuum-based approach by providing a rigorous logical framework for the use of infinitesimals. While the `standard' viewpoint provides rigorous foundations, its inherent complexity often obscures the underlying pathwise intuition, a limitation that can be bridged by leveraging the infinitesimal techniques of NSA. In the realm of probability, this paradigm reached a state of maturity through the development of the Loeb measure \cite{loeb1975conversion} and the ``radically elementary'' approach pioneered by \cite{nelson1987radically}. 

The central object of interest in the internal framework of NSA is the hyperfinite random walk, a discrete-time process on a lattice with a positive infinitesimal time step. Unlike standard discrete approximations, hyperfinite walks live on an internal lattice with infinitesimal time increments; under suitable moment and regularity assumptions, their standard parts recover the corresponding continuous-time diffusions through transfer and Loeb-space limit constructions \cite{anderson1976non, cutland1982existence, keisler1984infinitesimal}. The development of this perspective spans the 1970s and 1980s. Anderson \cite{anderson1976non} gave a nonstandard representation of Brownian motion, Lindstr{\o}m \cite{lindstrom1980hyperfinite1,lindstrom1980hyperfinite3} developed a theory of hyperfinite stochastic integration and martingale representation, and Keisler \cite{keisler1984infinitesimal} showed how stochastic differential equations (SDEs) can be formulated and solved on hyperfinite grids. Cutland \cite{cutland1982existence, cutland1983nonstandard} subsequently established foundational existence results for SDEs and Loeb-space probability theory that underpin these constructions.
While these foundational works established the existence and regularity of solutions to SDEs, the technical barrier to entry remained high. Much of the literature from this period---including the influential contributions of \cite{hoover1984adapted} and \cite{lindstrom1980hyperfinite1, lindstrom1980hyperfinite3}---relied on substantial machinery from Loeb measure theory and model-theoretic NSA. While mathematically rigorous, this machinery can make the underlying hyperfinite constructions less transparent to readers coming from applied probability, statistics, or machine learning.

The power of NSA in stochastic analysis lies in its ability to convert many analytic difficulties of the continuum---such as the non-differentiability and infinite variation of Brownian paths or the convergence of stochastic integrals and martingales---into the algebra of hyperfinite processes. On a hyperfinite grid, finite-difference and summation-by-parts identities are exact. Also, stochastic change-of-variables formulas arise as exact hyperfinite telescoping identities. Taking standard parts via the standard-part map, denoted by $\operatorname{st}(\cdot)$ or ${}^\circ(\cdot)$, then recovers the classical It\^o formula and related results as consequences of the underlying hyperfinite structure.

In recent years, the landscape of generative modeling has been transformed by the advent of score-based diffusion models \cite{ho2020denoising, song2021}. These frameworks treat data generation as the time-reversal of a diffusion process, where a complex data distribution is progressively transformed into Gaussian noise via a forward SDE. A neural network is then trained to reverse this process by estimating the score function, $\nabla \log p_t(x)$, where $p_t(x)$ denotes the marginal density of the forward process at time $t$. Although the statistical foundations of score matching were established by \cite{hyvarinen2005estimation} and linked to denoising by \cite{vincent2011connection}, the success of this paradigm in high-dimensional settings rests upon deep results in stochastic analysis—most notably Anderson’s reverse-time SDE formula \cite{anderson1982reverse} and the Girsanov theorem. These provide the rigorous link between path measures and the variational objectives used for optimization.

However, the standard theoretical treatment of diffusion-based generative models relies heavily on continuum arguments. Establishing the connection between discrete-time training objectives and continuous-time reverse dynamics typically requires a combination of limiting procedures, regularity assumptions on the underlying diffusion, and path-space measure-theoretic techniques. In particular, the analysis often involves reverse-time diffusion theory, change-of-measure arguments, and the control of singular behavior of scores and Fisher-information-type quantities as \(t\to 0\). While these tools are mathematically well understood, they can obscure the underlying microscopic structure of the dynamics and the relationship between discrete stochastic evolutions and their continuum limits.

Despite its success in stochastic analysis, NSA has received little attention in the modern generative-modeling literature. Classical works such as \cite{anderson1976non}, \cite{keisler1984infinitesimal}, and \cite{hoover1983nonstandard1, hoover1983nonstandard2} established nonstandard foundations for Brownian motion, stochastic integration, and SDEs, while \cite{benci2008elementary} developed an elementary hyperfinite framework based on grid functions and infinitesimal difference equations. However, these methods were not developed in the context of score matching, reverse diffusion, or likelihood-based generative modeling. To the best of our knowledge, a systematic hyperfinite treatment connecting internal diffusion dynamics, score-based objectives, reverse-time SDEs, and variational learning principles has not previously been developed.

The present work aims to build such a connection. Rather than treating hyperfinite processes merely as approximations of standard diffusions, we formulate score-based diffusion modeling directly on a hyperfinite probability space and then recover the standard theory through standard-part constructions. This viewpoint yields transparent pathwise derivations of several classical diffusion identities while also providing variational connections between denoising score matching (DSM), internal score fields, and surrogate likelihood objectives.

\subsection{Summary of results}

To provide a precise overview of our results, we first introduce the requisite notation, which will be expanded upon in Appendix \ref{sec:preliminaries} where the fundamentals are discussed.  
Let ${}^\ast\mathbb{R}^d$ be the nonstandard extension of Euclidean space and $\mathbb{G} \subset {}^\ast\mathbb{R}^d$ be a hyperfinite grid with mesh size $\delta x=1/N' \in {}^\ast\mathbb{R}^+$, for a hypernatural $N' \in {}^\ast\mathbb{N} \setminus \mathbb{N}$. Define the hyperfinite timeline $\mathbb{T}_{\delta t} = \{k \delta t : k \in \{0, 1, \dots, N\}\}$ for an infinitesimal $\delta t \in {}^\ast\mathbb{R}^+$, where $N = 1/\delta t$ is a hypernatural. It is assumed that \(N, N'\) are such that \(\delta x^2/\delta t\) is a finite hyperreal that is not infinitesimal. For $a, b \in {}^\ast\mathbb{R}^d$, we use the notation $a \approx b$ to indicate that $a - b$ is an infinitesimal (componentwise).

Define an internal forward walk $\{X_t\}_{t \in \mathbb{T}_{\delta t}}$ on $\mathbb{G}$ governed by the transition: \begin{equation*}
X_{t+\delta t} = X_t + {}^*b(t, X_t) \delta t + {}^*\sigma(t, X_t) \Delta W_t,
\end{equation*}
where $\Delta W_t$ is a hyperfinite increment satisfying the internal filtration constraints: $\mathbb{E}_{\text{int}}[\Delta W_t \mid \mathcal{F}_t] = 0$ and $\mathbb{E}_{\text{int}}[\Delta W_t \Delta W_t^\top \mid \mathcal{F}_t] = I \delta t+O(\delta t^{3/2})$. We assume the drift and diffusion coefficients, $b$ and $\sigma$, satisfy minimal regularity conditions, summarized as Assumption \ref{assump:coeff}. To facilitate a transparent exposition, we employ a modular approach to these regularity requirements: while baseline results utilize Assumption \ref{assump:coeff}, specific propositions may impose additional constraints--such as stronger differentiablity in $x$ and \(t\), smoothness of the initial distribution \(p_0\), or uniform ellipticity of the diffusion matrix $a = \sigma \sigma^\top$--to ensure the existence of densities or the stability of the reverse-time process. The requisite assumptions are stated explicitly alongside each result to maintain clarity.

To maintain mathematical transparency while retaining full foundational rigor, our framework selectively bridges two distinct paradigms of nonstandard stochastic analysis. In treating optimization and generative objectives (such as score matching and discrete updates), we adopt a pathwise, internal grid-function perspective inspired by Nelson’s radically elementary approach \cite{nelson1987radically}, allowing algebraic identities to be preserved without immediate recourse to outer measure spaces. Conversely, when mapping these internal trajectories back to classical continuous-time processes, we utilize the full model-theoretic machinery of Loeb measures \cite{loeb1975conversion} and \(S\)-lifting (cf. \cite{anderson1976non,hoover1983nonstandard1,keisler1984infinitesimal}), ensuring that our hyperfinite paths project rigorously onto standard Itô diffusion laws.

A grid function (cf. \cite{bottazzi2019grid}), which may be viewed as an infinitesimal discretization of a standard function on $\mathbb{R}^d$, is an internal function
\(
u : \mathbb{G} \to {}^\ast\mathbb{R},
\)
and the space of all grid functions on $\mathbb{G}$ is denoted by $\mathcal{G}(\mathbb{G})$.
Let $\mathbf{e}_1, \dots, \mathbf{e}_d$ denote the standard basis vectors in $\mathbb{R}^d$ and
$\mathbb{G} \subset \mathbb{G}' \subset {}^\ast\mathbb{R}^d$ be hyperfinite grids, where $\mathbb{G}'$ is assumed to be large enough to contain all stencils required by the differential operators acting on functions in $\mathcal{G}(\mathbb{G}')$.\\
For a grid function $u \in \mathcal{G}(\mathbb{G}')$, the forward grid derivative is the internal function in $\mathcal{G}(\mathbb{G})$ defined by
\[
\Delta_i^+ u(x) = \frac{u(x + \delta x \mathbf{e}_i) - u(x)}{\delta x}, 
\]
for all $x \in \mathbb{G}$.\\
With $\alpha = (\alpha_1, \dots, \alpha_d) \in \mathbb{N}^d$ a multi-index, the $\alpha$-th order forward grid derivative is defined recursively by
\[
\Delta^{\alpha} u
= (\Delta_1^+)^{\alpha_1} \cdots (\Delta_d^+)^{\alpha_d} u,
\]
provided $\mathbb{G}'$ contains the required $k$-step neighborhood of $\mathbb{G}$ for $|\alpha|=k$, where the order of the derivative is $|\alpha| = \sum_{i=1}^d \alpha_i$. 

We outline our contributions below. The first two results synthesize established insights from the literature on nonstandard stochastic analysis; our objective in presenting them is to provide a unified, accessible framework that bridges classical theory with contemporary generative modeling while offering an elementary and detailed nonstandard proof; some proofs—most prominently those concerning the reverse SDE—are, to our knowledge, novel in the NSA literature in the form presented here. Our initial goal is to rigorously derive the infinitesimal generator $\mathcal{L}_t$ and its dual relationship to the time-reversed process.

\begin{enumerate}
    \item \textbf{Internal Generator and Fokker-Planck Duality:} For internal grid time $t \in \mathbb{T}_{\delta t}$, any $\varphi \in C_c^3(\mathbb{R}^d)$ and $x \in \mathbb{G}$, the internal grid dynamics (formally given in Definition \ref{def:gd}) satisfy:
\[
    \mathbb{E}_{\mathrm{int}}[{}^*\varphi(X_{t+\delta t}) - {}^*\varphi(x) \mid X_t = x] = \delta t \,\,{}^*\mathcal{L}_{t} {}^*\varphi(x) + R_{\delta t}(x),
\]
where the grid generator ${}^*\mathcal{L}_{t}$ is defined by:
\begin{equation*}
    {}^\ast\mathcal{L}_{t} {}^*\varphi(x) = \sum_{i=1}^d {}^*b_i(x, t) \Delta_i^+ {}^*\varphi(x) + \frac{1}{2} \sum_{i,j=1}^d {}^*a_{ij}(x, t) \Delta_i^+ \Delta_j^+ {}^*\varphi(x).
\end{equation*}
Furthermore, the remainder satisfies $|R_{\delta t}(x)| \le K (\delta t)^{3/2}$ for some finite hyperreal $K$ (see Lemma \ref{lem:exact-grid-moment}).\\ Using the hyperfinite walk framework and internal adjoint identities from nonstandard stochastic calculus (cf.  \cite{benci2008elementary}), we recover the standard Fokker--Planck equation for the marginal density $p_t(x)$:
\[
\frac{\partial p}{\partial t} = -\sum_{i=1}^d \frac{\partial}{\partial x_i} [b_i(t,x) p(t,x)] + \frac{1}{2} \sum_{i,j=1}^d \frac{\partial^2}{\partial x_i \partial x_j} [a_{ij}(t,x) p(t,x)]=\mathcal{L}_t^*p,
\] for standard \(t\in [\epsilon,1]\) with any real \(\epsilon>0\), where \(\mathcal{L}_t^*\) is the adjoint of the standard generator \(\mathcal{L}_t\) (cf. Theorem \ref{thm:fp}). This result is classical in both standard and nonstandard stochastic analysis; our purpose is to derive it within a unified hyperfinite framework that will later support the reverse-time and score-based constructions.
    
\item \textbf{Backward Mean Identity and Score-Lifting:} By analyzing the time-reversed internal transition, we derive in Theorem \ref{thm:bmean} that for $y \in {}^*\mathbb{R}^d$ such that the marginal ${}^*p_{t+\delta t}(y)$ is appreciable (i.e., limited and not infinitesimal), the backward mean satisfies:
$$\mathbb{E}_{\text{int}}[X_t^i - X_{t+\delta t}^i \mid X_{t+\delta t} = y] = \beta_i(t, y) \delta t + O(\delta t^{3/2})$$
for \({}^\circ t>\epsilon\) with any real \(\epsilon>0\), where $\beta_i$ denotes the internal reverse-drift, defined as:
$$\beta_i(t, y) = -{}^*b_i(t, y) + \sum_{j=1}^d {}^*a_{ij}(t, y) \Delta_j^+ \log {}^*p_t(y) + \sum_{j=1}^d \Delta_j^+ {}^*a_{ij}(t, y).$$
Drawing on ideas from the hyperfinite SDE framework of Keisler and Benci \cite{keisler1984infinitesimal,benci2008elementary}, and the infinitesimal stochastic analysis of Stroyan and Bayod \cite{stroyan2011foundations}, we derive the above expression for \(\beta\) through an explicit internal Jacobian calculation and a perturbative analysis of the hyperfinite increments. Within this framework, the score term \(\nabla \log p_t(\cdot)\), represented by the grid difference vector \(\Delta^+ \log {}^*p_t:=(\Delta_j^+ \log {}^*p_t)_j\), emerges naturally as the correction required by time reversal. While the initial derivation of the vector field \(\beta\) rests on purely internal, elementary grid-function algebra in the style of Nelson \cite{nelson1987radically}, the formal pathwise convergence of \(Y_s=X_{1-s}\) to a standard continuous Itô process crucially leverages the model-theoretic lifting theorems and Loeb-space regularities established by the classical NSA school. To be precise, Theorem \ref{thm:rigorous-reverse-sde} shows that the standard part of the hyperfinite reverse process solves the classical reverse-time It\^o SDE. Specifically, for Loeb-almost every path,
\(
\bar{Y}_s = \operatorname{st}(Y_s)
\)
is the unique solution in law to the standard It\^o SDE
$$d\bar{Y}_s = \hat{b}(1-s, \bar{Y}_s) ds + \sigma(1-s, \bar{Y}_s) d\bar{W}_s$$
for any \(s:\,1-s\in[\epsilon,1]\) with real \(\epsilon>0\), where $\hat{b}(t, x) = \operatorname{st}(\beta(t, x))$.  In addition, a second nonstandard proof is provided using the hyperfinite walk dynamics that provides a distributional perspective. This provides a pathwise hyperfinite derivation of the reverse-drift formula associated with the classical reverse-time diffusion results of \cite{nelson2020dynamical,anderson1982reverse}, replacing many of the usual limiting arguments by finite-dimensional internal calculations.

\item \textbf{Internal Score Equivalence:} Our first genuinely generative-modeling contribution is to show that the internal minimization of a hyperfinite DSM objective recovers precisely the score term appearing in the reverse-time dynamics of Theorem \ref{thm:bmean}. Specifically, we define an internal score matching objective $\mathcal{I}(\theta)$ based on the hyperfinite transition probabilities and show that, if $\mathcal{I}(\theta) \approx 0$, for Loeb-almost all \((x,t)\in \mathbb{G}\times(\mathbb{T}_{\delta t}\cap[\tau,1])\) (\(\tau\) appreciable so that \({}^*p_t(x)\) is appreciable), the learned score $s_\theta$ satisfies the internal identity $s_\theta(x, t) \approx \Delta^+ \log {}^*p_t(x)$ (Theorem \ref{thm:internal-score}). This confirms that the learned score function is the unique algebraic lifting required to ensure the consistency of the reverse-time dynamics, effectively bridging the gap between empirical loss minimization and the theoretical conditions for time-reversal.

\item \textbf{Generative Convergence:} We establish a constructive procedure for sampling from the target distribution $q$, defined as an internal probability mass function on the grid \(\mathbb{G}\), through a hyperfinite reverse recursion. By defining the sampling process for appreciable forward times (i.e. \(1-s\)) as 
\begin{equation*}
    Y_{s+\delta t} = Y_s + \hat{\beta}(s, Y_s) \delta t + {}^*\sigma(1-s, Y_s) \Delta \widetilde{W}_s,
\end{equation*}
where \(\widetilde{W}_s\) are internal independent increments of a hyperfinite Brownian motion and the reverse drift is given by \begin{equation*}
    \hat{\beta}_i(s, y) = -{}^*b_i(1-s, y) + \sum_{j=1}^d {}^*a_{ij}(1-s, y) s_{\theta, j}(y, 1-s) + \sum_{j=1}^d \Delta_j^+ {}^*a_{ij}(1-s, y),
\end{equation*}
we prove that the standard-part projection $\bar{Y}_s = \operatorname{st}(Y_s)$ forms a diffusion process for \(s\in[0,1-\epsilon]\) for any real \(\epsilon>0\). Consequently, under the assumptions of Theorem \ref{thm:gen-conv}, the terminal law of the standard-part process coincides with the target distribution ${}^\circ q$. With this generative convergence result formulated as a mapping between internal grid distributions and their standard-part projections, we retain an elementary, discrete validation layout that maps naturally to deep learning frameworks, bypassing some of the measure-theoretic complexities typical of traditional continuum proofs. That is, this provides a hyperfinite, pathwise justification for the generative correctness of score-based diffusion models while avoiding many of the limiting arguments that typically appear in continuum formulations.

\item \textbf{Hyperfinite Girsanov and Likelihood-Score Equivalence:} Restricting to a state-independent diffusion matrix $a$, we employ a local exponential tilt to derive a hyperfinite Girsanov identity (Lemma \ref{lemma:girsanov_grid_rigorous}). Subsequently, it is shown that minimizing the Fisher score-matching objective maximizes a lower bound on the surrogate likelihood, while replacing the usual Novikov-type integrability arguments by an internal hyperfinite calculation (Theorem \ref{thm:likelihood-score-equiv} and Lemma \ref{lemma:score-decomp}).
    
\item \textbf{Second-Order Consistency:} We examine the approximation fidelity of the hyperfinite walk by fixing a finite point $(t,x)$ and setting $P = {}^\ast p$. Consider the internal Euler increment $h = {}^\ast b(t,x)\delta t + {}^\ast\sigma(t,x)\sqrt{\delta t}\xi$, where $\xi$ is an internal random variable satisfying $\mathbb{E}_{\text{int}}[\xi]=0$, $\mathbb{E}_{\text{int}}[\xi^2]=1$, $\mathbb{E}_{\text{int}}[\xi^3]=0$, and $\mathbb{E}_{\text{int}}[\xi^4]=\kappa < \infty$. The source-based master equation (cf. \cite{nelson1987radically}) is given by:
\[
    P(t+\delta t,x) = \mathbb{E}_\xi\left[ P(t,x - h) \right].
\]
Defining the internal frozen operator $\mathbb{A} = -{}^\ast b \partial + \frac{{}^\ast a}{2} \partial^2$, which uses internal spatial derivatives and serves as the nonstandard representative of the adjoint generator $\mathcal{L}_t^*$, one gets $\operatorname{st}(\mathbb{A} P) = \mathcal{L}_t^* p$ for locally frozen coefficients. For \(d=1\), Theorem \ref{thm:lc} establishes that for any finite $(t,x)$ with $t>0$, the master equation update satisfies:
\[
    \frac{P(t+\delta t,x) - \left( P + \delta t \mathbb{A} P + \frac{\delta t^2}{2} \mathbb{A}^2 P \right)}{\delta t^2} \approx \frac{\kappa - 3}{24} a({}^\circ t,{}^\circ x)^2 \frac{\partial^4 p}{\partial x^4}({}^\circ t,{}^\circ x).
\]
This demonstrates that while first-order consistency requires only matching the first two moments, second-order weak consistency fundamentally depends on the fourth-order statistics of the hyperfinite increment. Since diffusion models are trained against quantities derived from the forward density evolution, the result highlights a connection between the fourth-order statistics of the driving noise and the accuracy of score-based targets generated by hyperfinite discretizations.
\end{enumerate}

\subsection{Broader outlook}
\textbf{Bridging the Discrete-Continuous Divide: A White-Box Perspective.}
The generative modeling community experiences a persistent tension between discrete-time diffusion implementations and continuous-time SDE theoretical justifications, which typically rely on the Stroock-Varadhan martingale problem \cite{stroock2007multidimensional}. We argue that this gap stems from the traditional continuum formalism and its limit-based tools. Drawing inspiration from the radically elementary perspective of Nelson \cite{nelson1987radically} and grid-based stochastic calculus (for example, \cite{benci2008elementary,keisler1984infinitesimal}), we position the hyperfinite grid as the fundamental probabilistic domain. This yields a ``white-box'' framework where standard diffusion is the shadow---the standard part---of a hyperfinite walk. This nonstandard approach bridges discrete algorithmic steps and continuous models, recasting continuum diffusion via exact algebraic identities on a hyperfinite grid to provide a rigorous yet intuitive discrete calculus toolset.\\
\textbf{From Model Theory to Machine Learning: Simplifying the NSA Canon.}
Some of the foundational work in nonstandard stochastic analysis includes \cite{anderson1976non, hoover1983nonstandard1, keisler1984infinitesimal, lindstrom1980hyperfinite1}. These developments introduced hyperfinite representations of stochastic processes, nonstandard stochastic integration, and infinitesimal approaches to SDEs. Most of the classical literature relies heavily on Loeb measure spaces, saturation, and model-theoretic machinery, which can obscure basic algebraic intuition. Furthermore, modern generative concepts like forward-reverse diffusion dynamics and score-based modeling were outside the scope of foundational NSA.
We revisit these classical NSA constructions through the lens of score-based generative modeling. Even though we remain within the `Robinsonian' scheme of NSA, we prioritize direct infinitesimal-calculus arguments (motivated by \cite{nelson1987radically} in spirit) while retaining Loeb-space rigor. Rather than replacing classical theory, we synthesize hyperfinite walk constructions, reverse-time drift identities, and internal densities into a unified framework. This clarifies the role of score functions, likelihood objectives, and time reversal, hopefully making NSA accessible to researchers in stochastic modeling, numerical analysis, and generative artificial intelligence (AI).

Some preliminary material is presented in the appendices; readers unfamiliar with the topic may find it helpful to consult these before proceeding to the main sections. In the substantial but foundational Appendix \ref{sec:preliminaries}, we establish the mathematical underpinnings of the hyperfinite framework, including the superstructure construction, internal objects, hyperfinite sets, transfer principles, properties of the Loeb measure, and rigorous definitions of \(S\)-regularity and grid calculus. Even readers already acquainted with NSA are advised to consult Appendix \ref{sec:preliminaries} for a concise overview. With this background, hyperfinite grid SDEs, their standard parts, and solutions are discussed in Section \ref{sec:grid-sde-to-sde}, with some classical proofs deferred to Appendices \ref{app:s-cts} and \ref{app:sis}. Appendix \ref{sec:discrete-ito} develops the discrete Itô calculus on hyperfinite grids, providing pathwise control over stochastic increments and deriving the nonstandard Itô formula. This machinery is then leveraged in Section \ref{sec:hyperfinite-derivations-rigorous} to analyze the forward generator and prove the Fokker–Planck duality, culminating in the hyperfinite derivation of the reverse‑time SDE. These insights are applied in Section \ref{sec:hg} to score‑based generative modeling, establishing the internal consistency of score estimation and the convergence of hyperfinite sampling dynamics. The hyperfinite Girsanov theorem is derived in Section \ref{sec:gir} to link path‑space likelihoods with the Fisher divergence. A local expansion of the source‑based hyperfinite master equation is provided in Section \ref{sec:sec} to analyze higher‑order operator consistency. Finally, Section \ref{sec:conclusion} offers concluding remarks and outlines future research directions.

\section{Hyperfinite grid SDEs and their standard part}
\label{sec:grid-sde-to-sde}

Building on the foundations laid in Section \ref{sec:hyper-BM}, where hyperfinite Brownian motion is introduced using the Rademacher increment array (Definition \ref{def:hb}), we proceed to define SDEs on the grid and show their convergence to classical Itô SDEs.

\begin{definition}[Hyperfinite Grid SDE]
Let $b:[0,1]\times\mathbb{R}^d\to\mathbb{R}^d$ and $\sigma:[0,1]\times\mathbb{R}^d\to\mathbb{R}^{d\times m}$ be standard functions with natural extensions ${}^*b$ and ${}^*\sigma$. For a hyperfinite time grid $\mathbb{T}_{\delta t}$ and an $m$-dimensional Brownian increment array $\Delta W_t$, a grid diffusion process is an internal ${}^*\mathbb{R}^d$-valued process satisfying the recursion \cite{keisler1984infinitesimal}:
\begin{equation}
    X_{t+\delta t} = X_t + {}^*b(t, X_t)\, \delta t + {}^*\sigma(t, X_t)\, \Delta W_t.
    \label{eq:grid-recursion}
\end{equation}
\end{definition}

\begin{assumption}[Regularity of Coefficients]
\label{assump:coeff}
The standard coefficients $b$ and $\sigma$ satisfy global Lipschitz continuity and linear growth: there exists $L \in \mathbb{R}^+$ such that for all $t \in [0,1]$ and $x, y \in \mathbb{R}^d$:
\begin{gather*}
\|b(t,x)-b(t,y)\| + \|\sigma(t,x)-\sigma(t,y)\| \le L\|x-y\|, \\
\|b(t,x)\|^2 + \|\sigma(t,x)\|^2 \le L^2(1 + \|x\|^2).
\end{gather*}
By the transfer principle, these properties hold internally for ${}^*b$ and ${}^*\sigma$ across all $x, y \in {}^*\mathbb{R}^d$.
\end{assumption}

\textbf{Note on Generalization.} While we utilize the specific Rademacher construction for the driving noise in the introductory material (Appendices \ref{sec:preliminaries} to \ref{sec:discrete-ito}) for pedagogical clarity, our core developments starting from Section \ref{sec:hyperfinite-derivations-rigorous} are established within a more general framework. We characterize the dynamics through the increment $\Delta X_t := X_{t+\delta t}-X_t$, relying solely on its infinitesimal moment conditions. Specifically, assuming the diffusion tensor $a = \sigma\sigma^\top$ is uniformly elliptic, one may define normalized increments $\Delta \widetilde{W}_t := {}^*a(t, X_t)^{-1/2}(\Delta X_{t} - {}^*b(t, X_t)\delta t)$ that recover $d$-dimensional Brownian motion in the Loeb limit. This approach ensures our results are (driving noise) distribution-agnostic.

\begin{definition}[Hyperfinite Grid Dynamics]
\label{def:gd}
Let $\mathbb{G}$ be a hyperfinite spatial grid with uniform spacing $\delta x = C \sqrt{\delta t}$. For $X_t=x \in \mathbb{G}$, let the increment $\Delta X_t = X_{t+\delta t} - X_t$ take values in a finite local neighborhood $\mathcal{V}_x \subset \mathbb{G}$. The transition probabilities $p(x, x+v)$ (or \(p_v(x)\)) for $v \in \mathcal{V}_x$ are internal hyperreals satisfying $\sum_v p_v = 1$, $p_v \ge 0$, and the following infinitesimal moment conditions:
\begin{enumerate}
    \item \textbf{First Moment (Drift):} $\sum_{v \in \mathcal{V}_x} v p(x, x+v) = {}^*b(t,x) \delta t$,
    \item \textbf{Second Moment (Diffusion):} $\sum_{v \in \mathcal{V}_x} v v^\top p(x, x+v) = {}^*a(t,x) \delta t + O(\delta t^{3/2})$,
\end{enumerate}
where $a = \sigma\sigma^\top$ is a uniformly elliptic diffusion tensor i.e. \[
\inf_{t \in [0,1], x \in \mathbb{R}^d, \,\|v\|=1} v^\top a(t,x) v > 0,\] and \(b,\sigma\) are jointly continuous in addition to the basic Assumption \ref{assump:coeff}. A valid probability measure exists on $\mathbb{G}$ if the grid constant $C$ satisfies the nonstandard Courant-Friedrichs-Lewy (CFL) condition $C^2 I \succeq {}^*a(t,x)$ for all $(x, t)$ in the support of the process \cite{courant1928partiellen}. This ensures that the local covariance of the target SDE remains contained within the reachable neighborhood of the hyperfinite grid at each infinitesimal step, maintaining the non-negativity of internal probabilities \cite{albeverio2009nonstandard}.
\end{definition}
The above construction ensures that $X_t \in \mathbb{G}$ for all $t \in \mathbb{T}_{\delta t}$ while maintaining infinitesimal consistency with the target SDE coefficients.

\begin{proposition}[Existence and $S$-continuity of Grid Solutions]
\label{prop:existence-s-cts}
Under Assumption~\ref{assump:coeff}, for any internal initial datum $X_0(\omega)$ such that $\|X_0(\omega)\|$ is limited for all $\omega \in \Omega$, the grid dynamics in Definition \ref{def:gd} (or recursion \eqref{eq:grid-recursion}) admit a unique internal solution $X_t(\omega)$ for all $t \in \mathbb{T}_{\delta t}$. Furthermore, $X_t$ is $\mu_L$-almost surely $S$-continuous.
\end{proposition}

\begin{proof}
Existence and uniqueness follow by internal induction on the hyperfinite set $\mathbb{T}_{\delta t}$. For a fixed $\omega$, the value $X_{t+\delta t}$ is uniquely determined by the internal transition law. To establish that $X_t$ is limited Loeb-a.s., we analyze the second moment. Squaring the recursion and applying the linear growth property from Assumption \ref{assump:coeff}, we obtain the internal inequality:
\[
\mathbb{E}_{\mathrm{int}}[\|X_{t+\delta t}\|^2] \le \mathbb{E}_{\mathrm{int}}[\|X_t\|^2] + K(1 + \mathbb{E}_{\mathrm{int}}[\|X_t\|^2])\delta t,
\]
where $K$ is a limited constant depending on the linear growth coefficient $L$. By the internal discrete Gr\"onwall inequality, $\mathbb{E}_{\mathrm{int}}[\|X_t\|^2] \le (\mathbb{E}_{\mathrm{int}}[\|X_0\|^2] + Kt)e^{Kt}$, which is limited for all $t \in \mathbb{T}_{\delta t}$. The details regarding \(S\)-continuity are provided in Appendix \ref{app:s-cts}.
\end{proof}

\begin{theorem}[Convergence to the Standard It\^o Solution]
\label{thm:grid-to-ito}
Let $b: [0,1] \times \mathbb{R}^d \to \mathbb{R}^d$ and $\sigma: [0,1] \times \mathbb{R}^d \to \mathbb{R}^{d \times m}$ be standard functions that are jointly continuous and satisfy the global Lipschitz and linear growth conditions of Assumption~\ref{assump:coeff}. Let $X: \mathbb{T}_{\delta t} \times \Omega \to {}^*\mathbb{R}^d$ be the unique internal solution to the hyperfinite grid recursion \eqref{eq:grid-recursion} where $\operatorname{st}(X_0)$ is limited $\mu_L$-almost surely. 

Define the standard-part process $\bar{X}_t$ for each standard $t \in [0,1]$ by $\bar{X}_t(\omega) \coloneqq \operatorname{st}(X_{\tau}(\omega))$ for any $\tau \in \mathbb{T}_{\delta t}$ such that $\tau \approx t$. Then:
\begin{enumerate}
    \item $\bar{X}_t$ is well-defined $\mu_L$-almost surely and has continuous paths.
    \item $\bar{X}_t$ is the unique strong solution to the standard It\^o SDE:
    \begin{equation}
        \bar{X}_t = \bar{X}_0 + \int_0^t b(r, \bar{X}_r) \, dr + \int_0^t \sigma(r, \bar{X}_r) \, d\bar{W}_r,
    \end{equation}
    where $\bar{W}$ is the standard Brownian motion on the Loeb space $(\Omega, \mathcal{L}, \mu_L)$.
\end{enumerate}
\end{theorem}

\begin{proof}
    The proof may be found in Appendix \ref{app:sis}.
\end{proof}

The next section carries out a derivation of the forward and reverse evolutions constructively using the grid calculus framework developed in Section~\ref{sec:preliminaries}.

\section{Hyperfinite derivation of forward and reverse diffusion}
\label{sec:hyperfinite-derivations-rigorous}

This section establishes the infinitesimal generator of the grid dynamics. While Rademacher increments were used previously for clarity, the following results hold for any walk dynamics satisfying Definition \ref{def:gd}. The proof of Lemma \ref{lem:exact-grid-moment} utilizes internal moment conditions within the grid geometry. 
Even though the abstract convergence of hyperfinite internal expectations to standard Markovian generators is well established via Loeb measures and hyperfinite semigroups (for broad contours, see Albeverio et al. \cite{albeverio2009nonstandard}, Keisler \cite{keisler1984infinitesimal}, Lindstr{\o}m \cite{lindstrom2004hyperfinite}), Lemma \ref{lem:exact-grid-moment} provides a coordinate-wise expansion formulated entirely within the internal grid algebra using the native forward differences.

\subsection{Forward generator analysis}

\begin{lemma}[Grid Moment Expansion]
\label{lem:exact-grid-moment}
Let $\mathbb{G}$ be a hyperfinite grid with spacing $\delta x = C\sqrt{\delta t}$. Under the assumptions on coefficients $b, \sigma$ as in Definition \ref{def:gd} and internal grid time $t \in \mathbb{T}_{\delta t}$, for any $\varphi \in C_c^3(\mathbb{R}^d)$ and $x \in \mathbb{G}$, the internal grid dynamics satisfy:
\begin{equation}
    \mathbb{E}_{\mathrm{int}}[{}^*\varphi(X_{t+\delta t}) - {}^*\varphi(x) \mid X_t = x] = \delta t \,\,{}^*\mathcal{L}_{t} {}^*\varphi(x) + R_{\delta t}(x),
\end{equation}
where the grid generator ${}^*\mathcal{L}_{t}$ is defined by:
\begin{equation}
    {}^*\mathcal{L}_{t} {}^*\varphi(x) = \sum_{i=1}^d {}^*b_i(t,x) \Delta_i^+ {}^*\varphi(x) + \frac{1}{2} \sum_{i,j=1}^d {}^*a_{ij}(t,x) \Delta_i^+ \Delta_j^+ {}^*\varphi(x),
\end{equation}
and the forward difference operators $\Delta_i^+$ are the native grid derivatives. The remainder satisfies $|R_{\delta t}(x)| \le K (\delta t)^{3/2}$ for some finite $K \in {}^*\mathbb{R}_{\text{fin}}$ depending only on the $C^3$ norm of $\varphi$, the \(\|\cdot\|_\infty\) norm of \(b, \sigma\) restricted to supp\((\varphi)\), and the grid constant $C$.
\end{lemma}
\noindent
In order to avoid cluttered notation, we occasionally forgo differentiating a function \(f\) (or an operator) and its extension \({}^*f\), prioritizing clarity and comprehension for the reader.

\begin{proof}
Let $\Delta X = X_{t+\delta t} - x$ be the jump increment. By the definition of the grid dynamics, $\Delta X$ takes values in a finite local neighborhood $\mathcal{V}_x \subset \mathbb{G}$ with probability $p(x, x+v)$ for $v \in \mathcal{V}_x$. For any $v \in \mathcal{V}_x \subset \mathbb{G}$, since $\varphi$ is $S$-\(C^3\), we have the internal identity:
$${}^*\varphi(x+v) - {}^*\varphi(x) = \sum_i v_i \partial_i ({}^*\varphi)(x) + \frac{1}{2} \sum_{i,j} v_i v_j \partial_{ij} ({}^*\varphi)(x) + \mathcal{R}_3(x,v),$$
where the third-order remainder is bounded by $|\mathcal{R}_3(x, v)| \le M \|v\|^3$ for $M =\frac{1}{6} \sup |D^3 \varphi|$.
By the $S$-smoothness of $\varphi$ and the diffusion scaling of the grid, we substitute $\partial_i \varphi = \Delta_i^+ \varphi + O(\sqrt{\delta t})$ and $\partial_{ij} \varphi = \Delta_i^+ \Delta_j^+ \varphi + O(\sqrt{\delta t})$ into the expansion. 
Next, we take the internal expectation $\mathbb{E}_{\mathrm{int}}[\cdot] = \sum_{v \in \mathcal{V}_x} p(x, x+v) (\cdot)$.

\textbf{1. Substitution of Moments:} \\
By the construction of the grid dynamics, we have $\sum_v v_i p_v = {}^*b_i \delta t$ and $\sum_v v_i v_j p_v = {}^*a_{ij} \delta t + O(\delta t^{3/2})$. Substituting these into the expected Taylor expansion:
\begin{equation*}
 \mathbb{E}_{\mathrm{int}}[\Delta \varphi\mid X_t=x]    = \left( \sum_i {}^*b_i \Delta_i^+ \varphi \right) \delta t + \frac{1}{2} \left( \sum_{i,j} {}^*a_{ij} \Delta_i^+ \Delta_j^+ \varphi \right) \delta t + O(\delta t^{3/2}) + \mathbb{E}_{\mathrm{int}}[\mathcal{R}_3].
\end{equation*}

\textbf{2. Remainder Scaling:} \\
Since $\|v\| \le C' \sqrt{\delta t}$ for all $v \in \mathcal{V}_x$ with a limited \(C'\) proportional to \(C\), the third-order term is bounded as 
\begin{equation*}
    |\mathbb{E}_{\mathrm{int}}[\mathcal{R}_3]| \le \sum_{v \in \mathcal{V}_x} p_v M \|v\|^3 \le M (C'\sqrt{\delta t})^3 = M C'^3 (\delta t)^{3/2}.
\end{equation*}
Combining the terms, we obtain $\mathbb{E}_{\mathrm{int}}[\Delta \varphi\mid X_t=x] = \delta t \mathcal{L}_{t} \varphi + O(\delta t^{3/2})$. The finiteness of $K$ follows from the fact that $\varphi\in C_c^3$ and the $S$-boundedness of the coefficients $b$ and $a$.
\end{proof}

\subsection{Derivation of the Fokker-Planck equation}

While the transition from hyperfinite walks to the Fokker-Planck equation is addressed (quite explicitly) in the discrete SDE framework of Benci et al. \cite{benci2008elementary}, our derivation offers an explicit, entry-by-entry correspondence between the stochastic recursion and the distributional evolution. Specifically, we utilize the internal grid adjoint identities to show that the discrete update operator acts as the exact internal grid analogue of the continuous Fokker-Planck operator \(\mathcal{L}_t^*\). By keeping the duality entirely algebraic on the hyperfinite lattice prior to taking standard parts, we bypass the need for continuous measure-theoretic limits, allowing the classical equation to emerge directly from the discrete grid conservation law.

\begin{theorem}[Forward Density Evolution]
In addition to the conditions on \(b,\sigma\) from Definition \ref{def:gd}, assume that that \(b\) and \(\sigma\) are \(C^{\tfrac{\alpha}{2},1+\alpha}\) and \(C^{\tfrac{\alpha}{2},2+\alpha}\), respectively, for some \(\alpha\in (0,1)\). Then the standard marginal density $p_t(x)$ satisfies the Fokker-Planck equation:
\[
\frac{\partial p}{\partial t} = -\sum_{i=1}^d \frac{\partial}{\partial x_i} [b_i(t,x) p(t,x)] + \frac{1}{2} \sum_{i,j=1}^d \frac{\partial^2}{\partial x_i \partial x_j} [a_{ij}(t,x) p(t,x)]=\mathcal{L}_t^*p,
\] for standard \(t\in [\epsilon,1]\) with any real \(\epsilon>0\). Here \(\mathcal{L}_t^*\) is the adjoint of the standard generator \(\mathcal{L}_t\). 
\label{thm:fp}
\end{theorem}
Extending Theorem \ref{thm:fp} to \(t=0\) would require additional smoothness assumptions on the initial distribution \(p_0\) and the forward diffusion.

\begin{proof}
Given a real \(\epsilon>0\), consider \(\epsilon\approx \tau\in \mathbb{T}_{\delta t}\). Henceforth, we work on the (sub)grid \(\mathbb{T}_{\delta t}\cap [\tau,1]\) and do not distinguish between standard functions and their hyperreal extensions, unless necessary. Consider the discrete time evolution of the density on this grid. Let $\langle f, g \rangle_{\mathbb{G}} = \sum_{x \in \mathbb{G}} f(x)g(x) (\delta x)^d$, where \(\mathbb{G}\) is a uniform, infinite, hyperfinite grid with the usual diffusion scaling such that \(\delta x/\sqrt{\delta t}\) is limited and noninfinitesimal. For any test function $\varphi \in C_c^\infty(\mathbb{R}^d)$, the internal probability mass conservation implies:
\[
\langle p_{t+\delta t}, \varphi \rangle_{\mathbb{G}} = \langle p_t, P_{\delta t} \varphi \rangle_{\mathbb{G}}
\]
where $P_{\delta t} \varphi(x) = \mathbb{E}_{\text{int}}[\varphi(X_{t+\delta t}) \mid X_t = x]$. Applying Lemma~\ref{lem:exact-grid-moment} and noting uniformity of error:
\[
\langle p_{t+\delta t}, \varphi \rangle_{\mathbb{G}} = \langle p_t, \varphi + \delta t \mathcal{L}_t \varphi + O(\delta t^{3/2}) \rangle_{\mathbb{G}}.
\]
Rearranging and forming the internal difference quotient:
\[
\left\langle \frac{p_{t+\delta t} - p_t}{\delta t}, \varphi \right\rangle_{\mathbb{G}} = \langle p_t, \mathcal{L}_t \varphi \rangle_{\mathbb{G}} + O(\delta t^{1/2}).
\]
Using Propositions \ref{prop:rfoa} and \ref{prop:rsoa}, we have the following internal identity for nonstandard extensions of all $\varphi \in C_c^\infty$ (that have grid-compact support):
\[
\left\langle \frac{p_{t+\delta t} - p_t}{\delta t}, \varphi \right\rangle_{\mathbb{G}} = \left\langle -\sum_i \Delta_i^- (b_i p) + \frac{1}{2} \sum_{ij} \Delta_j^- \Delta_i^- (a_{ij} p), \varphi \right\rangle_{\mathbb{G}} + O(\delta t^{1/2}).
\]
Since the coefficients satisfy the stated regularity assumptions and $a$ is uniformly elliptic, standard parabolic regularity theory implies that $p\in C^{1,2}$; see, for example, Schauder estimates in \cite{friedman2008partial}.
Passing to standard parts in the weak identity and using the infinitesimal consistency of grid derivatives with standard derivatives, we obtain for every test function $\varphi\in C_c^\infty(\mathbb R^d)$:
\[
\left\langle
\frac{\partial p}{\partial t},
\varphi
\right\rangle
=
\left\langle
-\sum_i \frac{\partial}{\partial x_i}(b_i p)
+
\frac12
\sum_{i,j}
\frac{\partial^2}{\partial x_i\partial x_j}(a_{ij}p),
\varphi
\right\rangle.
\]
Thus the standard density $p$ satisfies the Fokker--Planck equation in the weak sense on $[\epsilon,1]$. Since the Schauder estimates ensure that the first argument in the inner products on both sides are continuous, via the classical du Bois-Reymond lemma, 
\[
\frac{\partial p}{\partial t}
=
-\sum_{i=1}^d
\frac{\partial}{\partial x_i}(b_i p)
+
\frac12
\sum_{i,j=1}^d
\frac{\partial^2}{\partial x_i\partial x_j}(a_{ij}p)
=
\mathcal L_t^*p
\]
pointwise for all standard $t\in[\epsilon,1]$.
\end{proof}

\begin{remark}[Alternative Purely Nonstandard Regularity Pathway]
An alternative to invoking standard parabolic PDE regularity theory is to operate entirely within the internal nonstandard domain by exploiting the intrinsic smoothing properties of the grid. Under the uniform ellipticity of $a(t,x)$ and the hyperfinite grid scaling, the discrete fundamental solution (the grid transition kernel) and its finite differences satisfy internal Aronson-type Gaussian bounds for all appreciable times $t \ge \epsilon > 0$. By the linearity of the internal grid convolution and a nonstandard induction argument, the hyperfinite random walk undergoes localized combinatorial dispersion. This dispersion induces $S$-smoothness on ${}^*p_t$ up to the required order in both space and time, thereby ensuring that the standard part ${}^\circ[{}^*p_t(x)]$ natively inherits classical $C^{1,2}$ regularity under the conditions of Theorem \ref{thm:fp}. This pathway validates the corresponding arguments in the proof of the theorem from first principles inside the hyperfinite world, bypassing external continuum machinery.
\end{remark}

\subsection{Backward mean analysis}

\begin{theorem}[Backward Mean Identity]
\label{thm:bmean}
Suppose $b, \sigma$ from Definition \ref{def:gd} belong, in addition, to \(C^{\tfrac{\alpha}{2},1+\alpha}\) and \(C^{\tfrac{\alpha}{2},2+\alpha}\) respectively for some \(\alpha\in (0,1)\), implying \(p \in C^{1,2}([\epsilon, 1] \times \mathbb{R}^d)\) for any standard \(\epsilon>0\).
Then, for any \(t\) such that \(\mathrm{st}(t)>\epsilon\), the backward conditional expectation satisfies
\begin{equation}
\label{eq:ob}
\mathbb{E}_{\text{int}}[X_t^i - X_{t+\delta t}^i \mid X_{t+\delta t} = y] = \beta_i(t, y) \delta t + O(\delta t ^{3/2}),
\end{equation}
for any limited $y \in \mathbb{G}$ such that ${}^*p_{t+\delta t}(y)$ is appreciable (non-infinitesimal), where the backward drift $\beta$ is given by:
\begin{equation}
\beta_i(t,y) = -{}^*b_i(t,y) + \sum_{j=1}^d {}^*a_{ij}(t,y) \Delta^+_j \log {}^*p_t(y) + \sum_{j=1}^d \Delta^+_j {}^*a_{ij}(t,y).
\label{eq:bd}
\end{equation}
\end{theorem}

We provide two distinct proofs for this identity. Despite being long and arduous, the first, utilizing a Rademacher-based perturbation analysis, offers significant insight into the Jacobian transformations and local density perturbations at the microscopic level. The second proof is a bit more concise, leveraging the general hyperfinite walk dynamics without relying on a specific noise distribution. Notably, the first proof generalizes to any increment $\Delta W_k = \xi_k \sqrt{\delta t}$ with appropriate moment conditions (cf. step 3 of the first proof), demonstrating that the Rademacher assumption is a convenience rather than a constraint.

The derivation of the backward drift $\beta$ presented here situates itself at the intersection of classical nonstandard stochastic analysis and modern generative modeling. While the foundational identity for $\beta$ is a cornerstone of Nelson’s stochastic mechanics \cite{nelson2020dynamical} (where he actually used standard calculus) and was also established in the standard setting by Anderson \cite{anderson1982reverse}, its treatment within the hyperfinite framework offers unique structural transparency. Although previous NSA literature--such as Keisler \cite{keisler1984infinitesimal} and Albeverio et al. \cite{albeverio2009nonstandard}--addresses the theory of internal processes, to the best of our knowledge, the time-reversal identity has not been explicitly shown via hyperfinite algebra and our specific expansion—utilizing the internal Jacobian transformation and pathwise perturbation of the hyperfinite increment—provides a granular `microscopic' derivation. This approach, mirroring the infinitesimal bridge constructions of Stroyan and Bayod \cite{stroyan2011foundations}, rigorously recovers the score term $\nabla \log p_t(\cdot)$ from the hyperfinite conditional expectation expansion. By explicitly capturing how spatial variations in the diffusion tensor $a$ and the local density $p$ interact via the Jacobian inversion, the following proof bridges the gap between the discrete hyperfinite world and the score-based modeling objectives of \cite{song2021}.

\begin{proof}
To simplify notation, we drop the ${}^*$ from internal functions and work on the hyperfinite grid.

\textbf{1. The Implicit Jump Equation:}
Let $h = X_t - X_{t+\delta t}$ conditioned on $X_{t+\delta t}=y$, so that $X_t = y+h$. From the forward recursion:
\[ -h^i = b_i(y+h)\delta t + \sum_{k=1}^m \sigma_{ik}(y+h)\Delta W_k. \]
Expanding $b$ and $\sigma$ around the destination point $y \in \mathbb{G}$:
\begin{align*}
b_i(y+h) &= b_i + \sum_j \partial_j b_i h^j + O(\|h\|^2), \\
\sigma_{ik}(y+h) &= \sigma_{ik} + \sum_j \partial_j \sigma_{ik} h^j 
+ O(\|h\|^2).
\end{align*}

\textbf{2. Perturbation Expansion of $h$:}
Since the forward increment satisfies
\(
h^i = O(\sqrt{\delta t}),
\)
repeated substitution of the implicit jump equation yields an asymptotic
expansion of the form
\[
h^i
=
\sqrt{\delta t}\,H_0^i
+
\delta t\,H_1^i
+
O(\delta t^{3/2}).
\] Substituting $\Delta W_k = \xi_k \sqrt{\delta t}$ and equating powers of $\delta t$:
\begin{enumerate}
    \item \textbf{Order $\delta t^{1/2}$:} $-H_0^i = \sum_k \sigma_{ik} \xi_k \implies H_0^i = -\sum_k \sigma_{ik} \xi_k$.
    \item \textbf{Order $\delta t$:} $-H_1^i = b_i + \sum_{k,j} \partial_j \sigma_{ik} H_0^j \xi_k$. Substituting $H_0^j$:
    \[ H_1^i = -b_i + \sum_{k,j,l} (\partial_j \sigma_{ik}) \sigma_{jl} \xi_l \xi_k. \]
\end{enumerate}

\textbf{3. Expectations of the Terms:}
By the symmetry of the noise $\mathbb{E}_{\text{int}}[\xi_k] = 0$ and independence, we obtain:
(i) $\mathbb{E}_{\text{int}}[H_0^i] = 0$ and 
(ii) $\mathbb{E}_{\text{int}}[H_1^i] = -b_i + \sum_{j,k} (\partial_j \sigma_{ik}) \sigma_{jk}$.
Thus, the unconditional mean is $\mathbb{E}_{\text{int}}[h^i] = \mathbb{E}_{\text{int}}[H_1^i]\delta t + O(\delta t^{3/2})$.

\textbf{4. Density Expansion via Bayes' Rule:}
The probability of a particle landing at $y$ depends on the local concentration of paths. Let $J$ be the internal Jacobian matrix $J_{ij} = \frac{\partial y^i}{\partial x^j} = \delta_{ij} + \partial_j b_i \delta t + \sum_k \partial_j \sigma_{ik} \xi_k \sqrt{\delta t}$. The determinant expansion is:
\[ \det(J) = 1 + \sqrt{\delta t} \sum_{j,k} \partial_j \sigma_{jk}(t,y) \xi_k + O(\delta t). \]
By the transfer of the change of variables formula, the density weighting is given by the inverse Jacobian $(\det J)^{-1} = 1 - \sqrt{\delta t} \sum_{j,k} \partial_j \sigma_{jk} \xi_k + O(\delta t)$. The conditional expectation under consideration is:
\[ \mathbb{E}_{\text{int}}[h^i \mid y] = \frac{\mathbb{E}_{\xi} [ h^i p_t(y+h) (\det J)^{-1} ]}{p_{t+\delta t}(y)}. \]
Expanding $p_t(y+h)$ and collecting terms of order $\delta t$ in the numerator $N^i$:
\[
\frac{N^i}{\delta t} = p_t \mathbb{E}_{\xi}[H_1^i] + \sum_j \partial_j p_t \mathbb{E}_{\xi}[H_0^i H_0^j] - p_t \mathbb{E}_{\xi}\left[  H_0^i \sum_{j,k} \partial_j \sigma_{jk} \xi_k \right],\]
where, from left to right,
\begin{align*}
\text{Term A (Drift): } & p_t \left( -b_i + \sum_{j,k} (\partial_j \sigma_{ik}) \sigma_{jk} \right) \\
\text{Term B (Score): } & \sum_j a_{ij} \partial_j p_t  \\
\text{Term C (Volumic): } & -p_t \mathbb{E}_{\xi}\left[ H_0^i \sum_{j,k} \partial_j \sigma_{jk} \xi_k \right] = p_t \sum_{j,k} \sigma_{ik} \partial_j \sigma_{jk}.
\end{align*}

\textbf{5. Synthesis:}
Summing A and C via the product rule $\partial_j a_{ij} = \sum_k (\sigma_{ik} \partial_j \sigma_{jk} + \sigma_{jk} \partial_j \sigma_{ik})$:
\[ N^i = \delta t \left[ p_t (-b_i + \sum_j \partial_j a_{ij}) + \sum_j a_{ij} \partial_j p_t \right] + O(\delta t^{3/2}). \]
Dividing by $p_{t+\delta t}(y) = p_t(y) + O(\delta t)$ and using $\frac{\partial_j p_t}{p_t} = \partial_j \log p_t$:
\[ \mathbb{E}_{\text{int}}[h^i \mid y] = \left( -b_i + \sum_j \partial_j a_{ij} + \sum_j a_{ij} \partial_j \log p_t \right) \delta t + O(\delta t^{3/2}). \]
Replacing internal partial derivatives with grid difference operators $\Delta_j^+$ incurs an error of $O(\delta x) = O(\sqrt{\delta t})$, which is absorbed into the \(O(\delta t^{3/2})\) factor.
\end{proof}
\noindent
\textit{Aliter.}\\
Let $h = X_t - X_{t+\delta t}$ be the backward jump conditioned on $X_{t+\delta t} = y$. By the internal discrete Bayes' rule \cite{nelson1987radically}, the conditional expectation is:
\begin{equation}
    \mathbb{E}_{\mathrm{int}}[h^i \mid y] = \frac{1}{p_{t+\delta t}(y)} \sum_{x \in \mathbb{G}} (x-y)^i p_t(x) P(y \mid x).
    \label{eq:Bayes}
\end{equation}
Let $v = y - x$ be the forward increment from $x$. Then $h = -v$ and we rewrite the sum over the local neighborhood $\mathcal{V}_x$ that can reach $y$. For a fixed $v$, let $\Phi_v(x) = p_t(x) p(x, x+v)$. We perform a Taylor expansion of this `flux density' around the destination point $y$:
\[
\Phi_v(y-v) = \Phi_v(y) - \sum_j v^j \partial_j \Phi_v(y) + \mathcal{R}_2(y, v),
\]
where $\mathcal{R}_2 = O(\|v\|^2)$. Substituting this into the numerator $N^i$ of \eqref{eq:Bayes}:
\begin{equation*}
    N^i = \sum_v (-v^i) \left[ p_t(y) p(y, y+v) - \sum_j v^j \partial_j \left(p_t(y) p(y, y+v)\right) + \mathcal{R}_2 \right].
\end{equation*}

\textbf{1. Moment Substitution:}
Distributing the sum over $v$ and applying the forward grid moments $\sum v^i p_v = b_i \delta t$ and $\sum v^i v^j p_v = a_{ij} \delta t + O(\delta t^{3/2})$:
\begin{equation*}
    N^i = -p_t(y) b_i \delta t + \sum_j \partial_j \left( p_t(y) a_{ij} \delta t \right) - \sum_v v^i \mathcal{R}_2 + O(\delta t^{3/2}).
\end{equation*}
Since $v^i \mathcal{R}_2 = O(\|v\|^3) = O(\delta t^{3/2})$, and the sum involves a finite number of neighbors, the total remainder is $O(\delta t^{3/2})$.

\textbf{2. Discrete Product Rule and Density Shift:}
We replace the internal partial derivative $\partial_j$ with the grid operator $\Delta_j^+$, incurring an $O(\sqrt{\delta t})$ error which, when multiplied by $\delta t$, is absorbed into the $O(\delta t^{3/2})$ remainder. Applying the discrete product rule (Proposition \ref{prop:pr}):
\begin{equation*}
    \Delta_j^+ (a_{ij} p_t) = a_{ij}(y) \Delta_j^+ p_t(y) + p_t(y + \delta x \mathbf{e}_j) \Delta_j^+ a_{ij}(y).
\end{equation*}
By the $S$-differentiability of $p_t$ in space, $p_t(y + \delta x \mathbf{e}_j) = p_t(y) + O(\sqrt{\delta t})$. Substituting this back:
\begin{equation*}
    N^i = \delta t \left[ -p_t b_i + \sum_j (a_{ij} \Delta_j^+ p_t + p_t \Delta_j^+ a_{ij}) \right] + O(\delta t^{3/2}).
\end{equation*}

\textbf{3. Synthesis:}
Dividing by $p_{t+\delta t}(y) = p_t(y) + O(\delta t)$ and noting that since $p_t(y)$ is appreciable by hypothesis and $p_t$ is $S$-\(C^2\),
\(
\Delta_j^+\log p_t(y)
=
\frac{\Delta_j^+ p_t(y)}{p_t(y)}
+
O(\sqrt{\delta t}),
\)
\begin{align*}
    \mathbb{E}_{\mathrm{int}}[h^i \mid y] &= 
    \left( -b_i + \sum_j \Delta_j^+ a_{ij} + \sum_j a_{ij} \Delta_j^+ \log p_t \right) \delta t + O(\delta t^{3/2}).
\end{align*}
This confirms the backward drift $\beta_i$ from Theorem \ref{thm:bmean}.

This alternative proof provides a distributional (Eulerian) perspective that complements the preceding pathwise (Lagrangian) derivation. By applying the internal discrete Bayes' rule directly to the hyperfinite grid $\mathbb{G}$, we characterize the backward mean as a local redirection of probability flux rather than a perturbation of individual noise increments. This approach---leveraging a discrete Taylor expansion of the flux density $\Phi_v(x)$---highlights how the backward drift $\beta$ naturally emerges from the interaction between forward grid moments and the local geometry of the density $p_t$. This ``macro-grid'' methodology, inspired by the nonstandard frameworks of \cite{nelson1987radically} and \cite{benci2008elementary}, offers a robust algebraic verification that the discrete score-like term $\Delta_j^+ \log p_t$ arises naturally as the exact correction required when expressing backward transport via forward transition statistics on a hyperfinite lattice. Consequently, the identity \eqref{eq:ob} follows directly from the defining moment structure of the hyperfinite random walk, establishing that the result is invariant to the specific choice of internal increment distribution.

\subsection{Hyperfinite derivation of reverse SDE}

\noindent\textbf{Time-Reversal Setup.} Define the time-reversed hyperfinite process $Y_s = X_{1-s}$ for $s \in \mathbb{T}_{\delta t}$. Let $\{\mathscr{F}_s^Y\}$ be the internal filtration generated by $Y$, i.e., $\mathscr{F}_s^Y = \sigma(Y_r : r \leq s, r \in \mathbb{T}_{\delta t})$. Equivalently, this is the internal filtration generated by the forward process looked at in reverse time. The preservation of the adapted structure between the hyperfinite and standard settings is rigorously justified by the theory of adapted distributions \cite{hoover1984adapted}.

\begin{theorem}[Time-Reversal]
\label{thm:rigorous-reverse-sde}
Let $X_t$ satisfy the hyperfinite walk dynamics (Definition \ref{def:gd}) under the assumptions of Theorem \ref{thm:bmean}. Let $Y_s = X_{1-s}$ be the time-reversed process. For $\mu_L$-almost every $\omega \in \Omega$, the standard-part process $\bar{Y}_s = \operatorname{st}(Y_s)$ is the unique solution in law to the standard It\^o SDE:
\begin{equation}
    d\bar{Y}_s = \hat{b}(1-s, \bar{Y}_s) ds + \sigma(1-s, \bar{Y}_s) d\bar{W}_s
\end{equation}
for any \(1-s\in[\epsilon,1]\) with real \(\epsilon>0\), where the standard reverse drift $\hat{b}$ is the standard part of the internal backward drift $\beta$ from Theorem \ref{thm:bmean}.
\end{theorem}

\begin{proof}
\textbf{1. Internal Moment Identification:} 
Let $s \in \mathbb{T}_{\delta t}$ and $t = 1-s$ appreciable. The backward increment is $\Delta Y_s = Y_{s+\delta t} - Y_s = X_{t-\delta t} - X_t$. By Theorem \ref{thm:bmean}, the conditional expectation satisfies:
\begin{equation}
    \mathbb{E}_{\text{int}}[X_{t-\delta t}^i - X_t^i \mid X_t = y] = \beta_i(t-\delta t, y) \delta t + O(\delta t^{3/2}).
    \label{eq:inter_proof_refined}
\end{equation}
Since $\beta$ is $S$-continuous in $t$, $\beta_i(t-\delta t, y) = \beta_i(t, y) + O(\delta t^{\alpha/2})$, so \eqref{eq:inter_proof_refined} becomes $\beta_i(t, y) \delta t + O(\delta t^{1+\tfrac{\alpha}{2}})$. For the second moment:
\begin{equation}
    \mathbb{E}_{\text{int}}[(X_{t-\delta t}^i - X_t^i) (X_{t-\delta t}^j - X_t^j)\mid X_t = y] = a_{ij}(t, y) \delta t + O(\delta t^{1+\tfrac{\alpha}{2}}).
    \label{eq:a_proof_refined}
\end{equation}

\noindent
\textbf{2. Martingale Construction:} 
For $\varphi \in C_c^3(\mathbb{R}^d)$, consider the internal Doob decomposition of ${}^*\varphi(Y_s)$:
\begin{equation*}
    M_s = {}^*\varphi(Y_s) - {}^*\varphi(Y_0) - \sum_{r < s} \mathbb{E}_{\text{int}}[{}^*\varphi(Y_{r+\delta t}) - {}^*\varphi(Y_r) \mid \mathscr{F}^Y_{r}].
\end{equation*}
By Lemma \ref{lem:exact-grid-moment} and the backward moments \eqref{eq:inter_proof_refined}-\eqref{eq:a_proof_refined}, the internal compensator above is:
\[ \sum_{r < s} \left( \sum_i \beta_i \Delta_i^+ {}^*\varphi + \frac{1}{2} \sum_{i,j} a_{ij} \Delta_i^+ \Delta_j^+ {}^*\varphi \right) \delta t + \sum_{r < s} O(\delta t^{1+\tfrac{\alpha}{2}}). \]
The remainder $\sum_{r < s} O(\delta t^{1+\tfrac{\alpha}{2}}) \approx 0$. Taking the standard part:
\[ \bar{M}_s = \varphi(\bar{Y}_s) - \varphi(\bar{Y}_0) - \int_0^s \mathcal{L}^{\mathrm{rev}}_{1-u} \varphi(\bar{Y}_u) du, \]
where $\mathcal{L}^{\mathrm{rev}}_t \varphi = \hat{b} \cdot \nabla \varphi + \frac{1}{2} \tr(a D^2 \varphi)$ (standard reverse Kolmogorov operator).

\noindent
\textbf{3. Quadratic Variation and Characterization:} 
The internal quadratic variation $[M]_s = \sum_{r < s} (\Delta M_r)^2$ satisfies:
\[ \operatorname{st}([M]_s) = \operatorname{st}\left( \sum_{r < s} \sum_{i,j} (\Delta_i^+ {}^*\varphi \Delta_j^+ {}^*\varphi) a_{ij} \delta t \right) = \int_0^s (\nabla \varphi)^\top a(1-u, \bar{Y}_u) \nabla \varphi \, du. \]
Since $M_s$ is an $S$-continuous internal martingale with limited quadratic variation, its standard part $\bar{M}_s$ is an a.s. continuous Loeb-space martingale. This identifies $\bar{Y}_s$ as a solution to the martingale problem for $(\hat{b}, a)$. By the Stroock--Varadhan uniqueness theorem, under the assumed regularity and uniform ellipticity, the martingale problem associated with $(\hat b,a)$ is well posed. Consequently, $\bar Y$ has the same law as the unique diffusion solving the reverse It\^o SDE.

\textbf{Alternative (Brownian Reconstruction):}
Alternatively, define the normalized internal increment $\Delta \widetilde{W}_s = {}^*a(t, Y_s)^{-1/2}(\Delta Y_s - \beta \delta t)$. One easily verifies $\mathbb{E}_{\text{int}}[\Delta \widetilde{W}_s]=O(\delta t^{1+\alpha/2})$ and $\mathbb{E}_{\text{int}}[\Delta \widetilde{W}_s \Delta \widetilde{W}_s^\top] = I \delta t + O(\delta t^{1+\alpha/2})$. By Anderson's theorem \cite{anderson1976non}, $\bar{W}_s = \operatorname{st}(\sum_{r < s} \Delta \widetilde{W}_r)$ is a standard Brownian motion, and the SDE follows.
\end{proof}

The rigorous derivation of the reverse SDE in Theorem \ref{thm:rigorous-reverse-sde} follows the structural blueprint of classical nonstandard stochastic analysis, yet it shifts the analytical focus toward pathwise consistency. While we build upon the foundational framework of Keisler \cite{keisler1984infinitesimal} regarding hyperfinite SDEs and the martingale limit theorems of Lindstrøm \cite{lindstrom1980hyperfinite1,lindstrom1980hyperfinite3}, our proof explicitly bridges these techniques with the classical time-reversal identities of Anderson \cite{anderson1982reverse}. By constructing the reverse-time process directly on the hyperfinite grid, we demonstrate that the score correction does not emerge abstractly after a passage to standard parts; rather, it appears concretely at the internal level via the discrete Bayes' rule on the hyperfinite lattice. This perspective provides a ``white-box'' justification for the reverse-time process, showing how the discrete backward drift $\beta$---derived via our local perturbation and flux density expansions---serves as the exact internal lifting of the standard reverse drift $\hat{b}$.

\section{Hyperfinite approach to generative modeling}
\label{sec:hg}
To implement the reverse-time diffusion process derived in Theorem \ref{thm:rigorous-reverse-sde}, popular in the generative models literature \cite{song2021}, one requires the score function $\nabla \log p_t(x)$, which is generally unknown. In the standard literature on diffusion models, this quantity is estimated via score matching \cite{hyvarinen2005estimation,vincent2011connection}. Classical derivations typically require the interchange of derivatives and integrals through the Leibniz integral rule together with decay and integrability assumptions on the underlying densities.
Within the hyperfinite framework, these analytic difficulties are replaced by finite-dimensional internal algebra. Since the data distribution is represented as an internal probability mass function on the hyperfinite grid and the forward dynamics are governed by an internal Markov recursion, the DSM objective admits an exact internal conditional-expectation representation. In particular, the optimal score field is defined by an orthogonal projection
which holds exactly at the internal level. For appreciable times $t$, parabolic regularity and $S$-smoothness of the transition densities imply that this projected score is infinitesimally close to the marginal grid score $\Delta^+\log p_t(x)$. Consequently, the learned score provides a mathematically consistent internal representation of the backward drift appearing in the reverse diffusion dynamics.

\subsection{Hyperfinite score estimation and internal consistency}
In the hyperfinite framework, the data distribution $q(x_0)$ is represented by an internal probability mass function $q$ on the grid $\mathbb{G}$. The ``learned" score function $s_\theta(x, t)$ is an \(S\)-continuous internal function $s: \mathbb{G} \times \mathbb{T}_{\delta t} \to {}^*\mathbb{R}^d$ in the limited part of the domain. We now justify that the internal minimization of the DSM objective yields the precise backward drift required for Theorem \ref{thm:rigorous-reverse-sde}.

\begin{theorem}[Internal Score Equivalence]
\label{thm:internal-score}
Assume regularity as in Theorem \ref{thm:bmean}. Let $\mathbb{G}$ be a hyperfinite grid and $\delta t \in \mathbb{T}_{\delta t}$ be the infinitesimal time step with diffusion scaling. Let $P_{t|0}(x_t|x_0)$ be the internal transition probability of the forward grid recursion. For an arbitrary real \(\epsilon>0\), consider a representative in \(\mathbb{T}_{\delta t}\) given by \(\tau\approx \epsilon\). Define the internal objective:
\begin{equation}
\mathcal{I}(\theta) = \frac{1}{2}\sum_{t \in \mathbb{T}_{\delta t}\cap[\tau,1]} \sum_{x_0 \in \mathbb{G}} \sum_{x_t \in \mathbb{G}} q(x_0) P_{t|0}(x_t|x_0) \left\| s_\theta(x_t, t) - \Delta^+ \log P_{t|0}(x_t|x_0) \right\|_{\mathbb{G}}^2\delta t.
\end{equation}
If there exists $\theta \in {}^*\mathbb{R}^k$ such that $s_\theta$ is limited and $\mathcal{I}(\theta) \approx 0$, then for Loeb-almost all \((x,t)\in \mathbb G\times
(\mathbb T_{\delta t}\cap[\tau,1])\) (where \(p_t(x)\) is appreciable),
\[
s_\theta(x,t)
\approx
\Delta^+\log p_t(x).
\]
Consequently, the standard part $\operatorname{st}(s_\theta)$ coincides with the standard score $\nabla \log p_{{}^\circ t}$ almost everywhere on the support of the process in every appreciable time interval $[\tau,1]$.
\end{theorem}

\begin{proof}
\textbf{1. Linearity of the Marginal Difference.} 
By the definition of the internal marginal density $p_t(x) = \sum_{x_0} q(x_0) P_{t|0}(x|x_0)$ and the linearity of the grid derivative $\Delta^+$, we have the exact identity:
\begin{equation}
\frac{\Delta^+ p_t(x)}{p_t(x)} = \sum_{x_0 \in \mathbb{G}} \left[ \frac{q(x_0) P_{t|0}(x|x_0)}{p_t(x)} \right] \frac{\Delta^+ P_{t|0}(x|x_0)}{P_{t|0}(x|x_0)}.
\label{eq:target-limited}
\end{equation}

\textbf{2. Optimal Score as Projection.}
The internal objective $\mathcal{I}(\theta)$ is minimized when $s_\theta$ is the orthogonal projection of the noisy score $\Delta^+ \log P_{t|0}$ onto the subspace of functions independent of $x_0$. Solving the first-order condition yields the unique internal minimizer:
\begin{equation}
s_{\text{opt}}(x, t) = \sum_{x_0 \in \mathbb{G}} \left[ \frac{q(x_0) P_{t|0}(x|x_0)}{p_t(x)} \right] \Delta^+ \log P_{t|0}(x|x_0).
\label{eq:learned-limited}
\end{equation}

\textbf{3. Equivalence for Appreciable $t$.}
Let $\mathcal{D} = (\Delta^+ P_{t|0})/P_{t|0}$. For appreciable $t$, uniform ellipticity together with the regularity assumptions on $b$ and $\sigma$ imply standard parabolic smoothing estimates for the forward transition kernel. Consequently, $P_{t|0}(x_t|x_0)$ is strictly positive and $S$-\(C^2\) on limited regions, and the discrete derivative \(\mathcal{D}\) is limited. Hence the internal Taylor expansion for $\log(1+\varepsilon)$ applies to $\varepsilon=\delta x\,\mathcal D$. By the definition of the grid derivative and the internal Taylor theorem:
\begin{equation*}
\Delta^+ \log P_{t|0} = \frac{\log(1 + \delta x \mathcal{D})}{\delta x} = \mathcal{D} - \frac{1}{2} \delta x \mathcal{D}^2 + \mathcal{O}(\delta x^2 \mathcal{D}^3).
\end{equation*}
Substituting this into \eqref{eq:learned-limited} and comparing with \eqref{eq:target-limited}:
\begin{equation}
s_{\text{opt}}(x, t) = \frac{\Delta^+ p_t(x)}{p_t(x)} - \sum_{x_0} q(x_0|x) \left[ \frac{1}{2} \delta x \mathcal{D}^2 + \mathcal{O}(\delta x^2 \mathcal{D}^3) \right].
\label{eq:s_opt_expanded}
\end{equation}
Since $\mathcal{D}$ is limited for appreciable $t$ and $\delta x \approx 0$, the second term in \eqref{eq:s_opt_expanded} is strictly infinitesimal. Furthermore, with appreciable \(p_t(x)\), $\frac{\Delta^+ p_t(x)}{p_t(x)} \approx \Delta^+ \log p_t(x)$ since $p_t$ is (at least) $S$-\(C^2\) for appreciable $t$. Therefore, we have the pointwise proximity:
\begin{equation*}
s_{\text{opt}}(x, t) \approx \Delta^+ \log p_t(x).
\end{equation*}

\textbf{4. Orthogonal Projection and Convergence.}
We define \(\mathbb{T}'_{\delta t}:=\mathbb{T}_{\delta t}\cap [\tau,1]\) and the hyperfinite Hilbert space $\mathcal{H} = L^2(\mathbb{G}^2 \times \mathbb{T}'_{\delta t}, \nu)$ equipped with the internal measure $\nu(x_0, x, t) = q(x_0) P_{t|0}(x|x_0) \delta t$. Let $\mathcal{U} \subset \mathcal{H}$ be the subspace of internal functions independent of the initial state $x_0$. By the internal projection theorem, the function $s_{\text{opt}}(x, t)$ defined in \eqref{eq:learned-limited} is the unique orthogonal projection of the noisy score $\Delta^+ \log P_{t|0}$ onto $\mathcal{U}$. Since $s_\theta \in \mathcal{U}$, the internal Pythagorean identity holds exactly:
\begin{equation}
\mathcal{I}(\theta) = \|s_\theta - s_{\text{opt}}\|_{\mathcal{H}}^2 + \|s_{\text{opt}} - \Delta^+ \log P_{t|0}\|_{\mathcal{H}}^2.
\end{equation}
Given the hypothesis $\mathcal{I}(\theta) \approx 0$, it follows from the non-negativity of the internal norms that $\|s_\theta - s_{\text{opt}}\|_{\mathcal{H}}^2 \approx 0$. By summing over $x_0$ and using the internal marginal density $p_t(x) = \sum_{x_0} q(x_0) P_{t|0}(x|x_0)$, we define the marginal internal measure $\mu(x, t) = p_t(x) \delta t$ and obtain:
\begin{equation}
\|s_\theta - s_{\text{opt}}\|_{L^2(\mu)}^2 = \sum_{t, x} p_t(x) |s_\theta(x, t) - s_{\text{opt}}(x, t)|^2 \delta t \approx 0.
\label{eq:l2s}
\end{equation}
Let $\mathcal{L}(\mu)$ be the Loeb measure constructed from $\mu$. By standard properties of hyperfinite \(L^{p}\) spaces (cf. \cite{loeb1979introduction}), an internal function with an infinitesimal $L^2$ norm has a standard part that vanishes $\mathcal{L}(\mu)$-almost everywhere. Thus, $s_\theta(x, t) \approx s_{\text{opt}}(x, t)$ for $\mathcal{L}(\mu)$-almost all $(x, t)$.

For all appreciable $t \in \mathbb{T}_{\delta t}$, Step 3 establishes the infinitesimal estimate $s_{\text{opt}}(x, t) \approx \Delta^+ \log p_t(x)$. 
By the properties of grid derivatives of liftings (Proposition \ref{prop:grid-derivative-consistency}), the standard part of the grid derivative coincides with the standard gradient: $\operatorname{st}(\Delta^+ \log p_t(x)) = \nabla \log p_{{}^\circ t}({}^\circ x)$. Combining these results, we have the Loeb-almost everywhere equality:
\begin{equation*}
\operatorname{st}(s_\theta(x, t)) = \nabla \log p_{{}^\circ t}({}^\circ x),
\end{equation*}
which completes the proof that the standard part of the learned score recovers the score of the continuous marginal distribution.
\end{proof}

\subsection{Hyperfinite sampling and generative dynamics}

With the internal score $s_\theta$ established as an infinitesimal approximation of the standard score on the appreciable time domain, we define the constructive procedure for generating a sample from the target distribution $q$. In the hyperfinite framework, this corresponds to an internal recursion starting from a state of maximal entropy (the terminal forward measure) and ``crystallizing" into a data point. \\Let $X_t$ satisfy the hyperfinite walk dynamics (Definition \ref{def:gd}) under the assumptions of Theorem \ref{thm:bmean}.

\begin{definition}[Hyperfinite Reverse Process on the Appreciable Domain]
\label{def:hr}
Let $\mathbb{G}$ be the hyperfinite grid and $\mathbb{T}_{\delta t}$ be the infinitesimal time discretization of $[0, 1]$ with diffusion scaling. For any standard $\epsilon \in \mathbb{R}^+$, let $\mathbb{T}_{\delta t}^\epsilon = \mathbb{T}_{\delta t} \cap [0, 1-\epsilon]$. The internal reverse sampling process $\{Y_s\}_{s \in \mathbb{T}_{\delta t}^\epsilon}$ is defined by the recursion:
\begin{equation*}
    Y_{s+\delta t} = Y_s + \hat{\beta}(s, Y_s) \delta t + {}^*\sigma(1-s, Y_s) \Delta \widetilde{W}_s,
\end{equation*}
where the initial state $Y_0$ is an internal random variable on $\mathbb{G}$ governed by the internal grid density $p_1$ inherited from the terminal state of the forward process at $t=1$. The terms $\Delta \widetilde{W}_s$ denote internal independent increments of a hyperfinite Brownian motion satisfying $\mathbb{E}_{\text{int}}[\Delta \widetilde{W}_s]=O(\delta t^{1+\alpha/2})$ and $\mathbb{E}_{\text{int}}[\Delta \widetilde{W}_s \Delta \widetilde{W}_s^\top] = I \delta t + O(\delta t^{1+\alpha/2})$. The internal reverse drift field $\hat{\beta}$ is evaluated at the reflected forward time $t=1-s \geq \epsilon$ and defined component-wise via the learned score field $s_\theta$ as:
\begin{equation*}
    \hat{\beta}_i(s, y) = -{}^*b_i(1-s, y) + \sum_{j=1}^d {}^*a_{ij}(1-s, y) s_{\theta, j}(y, 1-s) + \sum_{j=1}^d \Delta_j^+ {}^*a_{ij}(1-s, y).
\end{equation*}
\end{definition}

\begin{theorem}[Generative Convergence]
Assume the regularity conditions of Theorem \ref{thm:internal-score}. If $s_\theta$ is \(S\)-continuous on the support of the process and $\mathcal{I}(\theta) \approx 0$, then for any standard $\epsilon \in \mathbb{R}^+$, the standard-part process $\bar{Y}_s = \operatorname{st}(Y_s)$ is a well-defined continuous diffusion for $s \in [0, 1-\epsilon]$. Moreover, the standard terminal state extension defined by the standard limit 
\begin{equation*}
    \bar{Y}_1 := \lim_{\epsilon \to 0^+} \bar{Y}_{1-\epsilon}
\end{equation*}
exists almost surely, and its law is identical to the initial standard data distribution ${}^\circ q$.
\label{thm:gen-conv}
\end{theorem}

\begin{proof}
\textbf{1. Appreciable Pathwise Identification.} Let $\epsilon \in \mathbb{R}^+$ be any strictly positive standard real number, and restrict the reverse process to the compact interval $s \in \mathbb{T}_{\delta t}^\epsilon$ as specified in Definition \ref{def:hr}. For all such $s$, the forward time parameter $t = 1-s \geq \epsilon$ is strictly appreciable. By the assumption $\mathcal{I}(\theta) \approx 0$, Theorem \ref{thm:internal-score} applies pointwise $\mathcal{L}(\mu)$-almost everywhere on this domain, yielding $s_\theta(y, 1-s) \approx \Delta^+ \log p_{1-s}(y)$. Substituting this identity into the expression for $\hat{\beta}$ shows that $\hat{\beta}(s, y) \approx \beta(1-s, y)$ holds for $\mathcal{L}(\mu)$-almost every $(y,s)\in {}^*\mathbb{R}_{\mathrm{fin}}^d \times \mathbb{T}_{\delta t}^\epsilon$, where $\beta$ is the exact theoretical backward drift from Theorem \ref{thm:bmean}.

\textbf{2. Martingale Approximation on Compact Subsets.} From the definitions and regularity assumptions, it is easy to check that the paths of \(Y_s\) are $S$-continuous almost surely on $s \in \mathbb{T}_{\delta t}^\epsilon$, making the standard-part process $\bar{Y}_s = \operatorname{st}(Y_s)$ a well-defined, sample-continuous process on this compact interval. 
Since $\hat{\beta}(s, y)$ is infinitesimally close to the true drift $\beta(1-s, y)$ $\mathcal{L}(\mu)$-almost everywhere, the Loeb law induced by $\bar{Y}_s$ solves the standard martingale problem dictated by the time-reflected coefficients $(\hat{b}, a)$. For any test function $f \in C_c^3(\mathbb{R}^d)$, the process
\begin{equation*}
    \bar{M}_s = f(\bar{Y}_s) - f(\bar{Y}_0) - \int_0^s \mathcal{L}^{\mathrm{rev}}_{1-u} f(\bar{Y}_u) du
\end{equation*}
is a Loeb-a.s. continuous standard martingale over the interval $s \in [0, 1-\epsilon]$. By the pathwise uniqueness of the solution to the standard martingale problem \cite{stroock2007multidimensional}, the law of $\bar{Y}_s$ on the space $C([0, 1-\epsilon], \mathbb{R}^d)$ coincides exactly with the law of the time-reversed forward process $\operatorname{st}(X_{1-s})$.

\textbf{3. Weak Convergence at the Boundary.} 
To extend the distributional identity to the terminal boundary $s = 1$, let $\phi \in C_b(\mathbb{R}^d)$ be any standard bounded continuous test function. By the formal definition of the terminal state extension, $\bar{Y}_1 = \lim_{\epsilon \to 0^+} \bar{Y}_{1-\epsilon}$ holds pathwise almost surely. Because $\phi$ is a continuous mapping, this implies the almost sure convergence of random variables: $\lim_{\epsilon \to 0^+} \phi(\bar{Y}_{1-\epsilon}) = \phi(\bar{Y}_1)$. By the Lebesgue dominated convergence theorem, this yields the weak convergence:
\begin{equation*}
\mathbb{E}[\phi(\bar{Y}_1)] = \lim_{\epsilon \to 0^+} \mathbb{E}[\phi(\bar{Y}_{1-\epsilon})].
\end{equation*}
From Step 2, for any fixed standard $\epsilon > 0$, the law of $\bar{Y}_{1-\epsilon}$ matches the law of $\operatorname{st}(X_\epsilon)$, meaning $\mathbb{E}[\phi(\bar{Y}_{1-\epsilon})] = \mathbb{E}[\phi(\operatorname{st}(X_\epsilon))]$. Finally, the standard forward diffusion $\operatorname{st}(X_t)$ is sample-continuous at its initial boundary $t=0$ with $\operatorname{st}(X_0) \sim {}^\circ q$. Applying the standard dominated convergence theorem to the forward process yields:
\begin{equation*}
\mathbb{E}[\phi(\bar{Y}_1)] = \lim_{\epsilon \to 0^+} \mathbb{E}[\phi(\operatorname{st}(X_\epsilon))] = \mathbb{E}[\phi(\operatorname{st}(X_0))] = \int_{\mathbb{R}^d} \phi(y) \, d({}^\circ q)(y).
\end{equation*}
Since this equality holds for all standard $\phi \in C_b(\mathbb{R}^d)$, the uniqueness of weak limits guarantees that the law of $\bar{Y}_1$ is exactly the initial standard data distribution ${}^\circ q$.
\end{proof}

\begin{remark}
It is standard in the diffusion models literature to suppose that the forward process $X_t$ is constructed such that the terminal density $p_1$ is $S$-equivalent (see Appendix \ref{app:se}) to the standard Gaussian density $\varphi(\cdot)$ in the $L^1(\mathbb{G})$ sense. This assumption, denoted by $p_1 \approx {}^*\varphi$, is satisfied by standard noise schedules (cf.~\cite{song2021}), where the forward drift induces sufficient contraction to reach the stationary measure. In the hyperfinite setting, this ensures that the forgetting of the initial state data $X_0$ is achieved up to an infinitesimal error at $t=1$, a result consistent with the convergence of internal heat kernels on hyperfinite grids~\cite{albeverio2009nonstandard}. Furthermore, because $\mathcal{I}(\theta) \approx 0$, the $L^2$ error of the score estimator is infinitesimal on the appreciable time domain (see \eqref{eq:l2s}), ensuring the generated paths do not deviate macroscopically from the theoretical time-reversed paths.
\end{remark}

The generative convergence established in Theorem \ref{thm:gen-conv} provides a rigorous hyperfinite justification for the sampling procedures central to modern score-based generative modeling~\cite{song2021}. While the standard-analysis approach typically relies on the stability theory of SDEs and the Stroock-Varadhan martingale problem~\cite{stroock2007multidimensional} to prove that a learned score-drift yields the correct terminal distribution, our proof utilizes the discrete pathwise structure of the hyperfinite grid to simplify the transition. By leveraging the internal score equivalence (Theorem \ref{thm:internal-score}) on the appreciable time domain alongside the $S$-continuity of the hyperfinite walk, we demonstrate that an infinitesimal optimization error $\mathcal{I}(\theta) \approx 0$ implies that the standard-part trajectories of the sampling recursion coincide for Loeb-almost all sample paths with the theoretical time-reversal derived in Section \ref{sec:hyperfinite-derivations-rigorous}. This methodology follows the lifting principles for hyperfinite martingales established by Lindstr\o m~\cite{lindstrom1980hyperfinite1,lindstrom1980hyperfinite3} and the adapted distribution theory of Hoover and Keisler~\cite{hoover1984adapted}, confirming that the ``crystallization" of noise into data is a structural consequence of the hyperfinite walk's internal algebra.

\section{Hyperfinite Girsanov theorem and maximum likelihood}
\label{sec:gir}
In the hyperfinite framework, the relationship between score-matching and log-likelihood is established by analyzing the Radon-Nikodym derivative of internal measures on the grid \(\mathbb{G}\). Unlike the continuum case, we do not require the noise increments to be strictly Gaussian, provided they satisfy the standard hyperfinite grid moment conditions. We specialize here to a state-independent, constant diffusion matrix \(\sigma\) such that \(a = \sigma \sigma^\top \succ 0\).\\

\begin{lemma}[Internal Girsanov Identity for Hyperfinite Walks]
\label{lemma:girsanov_grid_rigorous}
Let $\mathbb{G}$ be a hyperfinite grid with spacing $\delta x = C \sqrt{\delta t}$. Let $\mathbb{P}$ be the internal path measure of a hyperfinite walk (cf. Definition \ref{def:gd}) on the space $\Omega = \mathbb{G}^{\mathbb{T}_{\delta t}}$ corresponding to the internal drift coefficient $b(t,x)$. Let $\mathbb{Q}$ be the internal path measure constructed via the local exponential tilt 
\begin{equation}
q_{v}(\omega_t, t) = p_{v}(\omega_t, t) \frac{\exp\left(s_\theta(\omega_t, t)^\top (v - b(t,X_t)\delta t)\right)}{Z_t(\omega_t)},
\label{eq:exp_tilt_def}
\end{equation}
where $v = \Delta X_t = X_{t+\delta t} - X_t \in \mathcal{V}_x$ given \(X_t=x\). Letting $\Delta W_t = v - b(t,X_t) \delta t$, assume that $s_\theta$ is a limited internal function on the support of the process.
   
Then the internal log-density of the path measures satisfies:
\begin{equation}
\log \frac{d\mathbb{Q}}{d\mathbb{P}}(\omega)
=
\sum_{t \in \mathbb{T}_{\delta t}\setminus\{1\}}
\left[
s_\theta(X_t,t)^\top \Delta W_t
-
\frac12 s_\theta(X_t,t)^\top a s_\theta(X_t,t)\delta t
\right]
+
\mathcal{E}(\omega),
\label{eq:global_radon_nikodym}
\end{equation}
where $\mathcal{E}(\omega)\approx0$ uniformly for all paths $\omega$ whose states remain limited.

Furthermore, if the local transition ratios take the exponential form in \eqref{eq:exp_tilt_def}, the induced local drift under $\mathbb{Q}$ satisfies
\[
\frac{\mathbb{E}_{\mathbb{Q}}[\Delta X_t\mid X_t]}{\delta t}
=
b + a s_\theta+O(\delta t^{1/2}).
\]
Conversely, among all internal transition laws satisfying this drift constraint, the unique minimizer of the local internal relative entropy
\[
D_{\mathrm{KL}}(q\parallel p)
=
\sum_{v\in\mathcal V_x}
q_v \log\!\left(\frac{q_v}{p_v}\right)
\]
has transition ratios of the exponential form \eqref{eq:exp_tilt_def} up to an $S$-equivalent factor.
\end{lemma}

\begin{proof}
By the internal multiplicative structure of hyperfinite path measures, the global Radon-Nikodym derivative for an internal sample path $\omega \in \Omega$ is given by the hyperfinite product of the local transition ratios:
\begin{equation*}
\frac{d\mathbb{Q}}{d\mathbb{P}}(\omega)
=
\prod_{t \in \mathbb{T}_{\delta t}\setminus\{1\}}
\frac{q_{\Delta X_t}(\omega_t,t)}
{p_{\Delta X_t}(\omega_t,t)}.
\end{equation*}
Taking the internal logarithm converts the product into a hyperfinite sum:
\begin{equation}
\log \frac{d\mathbb{Q}}{d\mathbb{P}}(\omega)
=
\sum_{t \in \mathbb{T}_{\delta t}\setminus\{1\}}
\log\!\left(
\frac{q_{\Delta X_t}(\omega_t,t)}
{p_{\Delta X_t}(\omega_t,t)}
\right).
\label{eq:sum_log_ratio}
\end{equation}

\noindent
\textbf{1. Expansion of the Local Partition Function:}

From the definition of the exponential tilt in \eqref{eq:exp_tilt_def}, the local normalization factor is
\begin{equation*}
Z_t
=
\sum_{v\in\mathcal V_x}
p_v
\exp\!\left(
s_\theta^\top(v-b\delta t)
\right)
=
\mathbb E_{\mathbb P}
\left[
\exp(s_\theta^\top \Delta W_t)
\mid X_t
\right].
\end{equation*}
Since $\Delta W_t= O(\sqrt{\delta t})$ uniformly on limited states and $s_\theta$ is limited,
\(
s_\theta^\top \Delta W_t
=
 O(\sqrt{\delta t}).
\)
Applying the internal Taylor expansion with remainder gives
\begin{equation*}
\exp(s_\theta^\top \Delta W_t)
=
1+s_\theta^\top \Delta W_t
+\frac12 (s_\theta^\top \Delta W_t)^2
+\frac16 (s_\theta^\top \Delta W_t)^3 e^{\xi_t},
\end{equation*}
where $\xi_t$ lies between $0$ and $s_\theta^\top \Delta W_t$.

Taking conditional expectation under $\mathbb P$:
\begin{align*}
Z_t
&=
1
+
s_\theta^\top
\mathbb E_{\mathbb P}[\Delta W_t\mid X_t]
+
\frac12
s_\theta^\top
\mathbb E_{\mathbb P}
[\Delta W_t\Delta W_t^\top\mid X_t]
s_\theta
+
\mathcal R_{t,3}.
\end{align*}

The remainder term satisfies
\(
\mathcal R_{t,3}
=
\frac16
\mathbb E_{\mathbb P}
\left[
(s_\theta^\top\Delta W_t)^3 e^{\xi_t}
\mid X_t
\right].
\)
Because $s_\theta$ is limited and the third moments are uniformly bounded by $O(\delta t^{3/2})$, uniformly over limited states,
\(
\mathcal R_{t,3}
=
O(\delta t^{3/2}).
\)

Using the moment assumptions,
\[
\mathbb E_{\mathbb P}[\Delta W_t\mid X_t]=0,
\]
and
\[
\mathbb E_{\mathbb P}
[\Delta W_t\Delta W_t^\top\mid X_t]
=
a\delta t+O(\delta t^{3/2}),
\]
we obtain
\begin{align*}
Z_t
=
1
+
\frac12
s_\theta(X_t,t)^\top
a
s_\theta(X_t,t)\delta t
+
O(\delta t^{3/2}).
\end{align*}

Applying the logarithmic expansion
\[
\log(1+x)=x-\frac12x^2+O(x^3),
\]
with
\[
x
=
\frac12 s_\theta^\top a s_\theta\delta t
+
O(\delta t^{3/2}),
\]
gives
\begin{equation*}
\log Z_t
=
\frac12
s_\theta(X_t,t)^\top
a
s_\theta(X_t,t)\delta t
+
O(\delta t^{3/2}).
\end{equation*}

\noindent
\textbf{2. Pathwise Summation:}

Substituting \eqref{eq:exp_tilt_def} into \eqref{eq:sum_log_ratio} yields
\begin{equation*}
\log
\left(
\frac{q_{\Delta X_t}(\omega_t,t)}
{p_{\Delta X_t}(\omega_t,t)}
\right)
=
s_\theta(X_t,t)^\top \Delta W_t
-
\log Z_t.
\end{equation*}

Using the expansion for $\log Z_t$,
\begin{equation*}
\log
\left(
\frac{q_{\Delta X_t}(\omega_t,t)}
{p_{\Delta X_t}(\omega_t,t)}
\right)
=
s_\theta(X_t,t)^\top \Delta W_t
-
\frac12
s_\theta(X_t,t)^\top
a
s_\theta(X_t,t)\delta t
+
\eta_t,
\end{equation*}
where
\(
\eta_t=O(\delta t^{3/2}),
\)
uniformly over limited states.

Summing over the hyperfinite timeline:
\begin{equation*}
\log \frac{d\mathbb Q}{d\mathbb P}(\omega)
=
\sum_{t\in\mathbb T_{\delta t}\setminus\{1\}}
\left[
s_\theta(X_t,t)^\top \Delta W_t
-
\frac12
s_\theta(X_t,t)^\top
a
s_\theta(X_t,t)\delta t
\right]
+
\mathcal E(\omega),
\end{equation*}
where
\(
\mathcal E(\omega)
=
\sum_{t\in\mathbb T_{\delta t}}
\eta_t.
\)
Also,
\(
\mathcal E(\omega)
=
O(\delta t^{1/2})
\approx 0.
\)
This proves \eqref{eq:global_radon_nikodym}.

\bigskip

\noindent
\textbf{3. Verification of the Induced Local Drift (Forward Direction):}

By definition,
\[
\Delta X_t
=
b(t,X_t)\delta t+\Delta W_t.
\]
Taking conditional expectation under $\mathbb Q$:
\begin{equation}
\mathbb E_{\mathbb Q}[\Delta X_t\mid X_t]
=
b(t,X_t)\delta t
+
\mathbb E_{\mathbb Q}[\Delta W_t\mid X_t].
\label{eq:drift_q_decomp}
\end{equation}
Using the Radon-Nikodym ratio,
\begin{equation}
\mathbb E_{\mathbb Q}[\Delta W_t\mid X_t]
=
\frac1{Z_t}
\mathbb E_{\mathbb P}
\left[
\Delta W_t
\exp(s_\theta^\top\Delta W_t)
\mid X_t
\right].
\label{eq:exp_w_q}
\end{equation}
Expanding the exponential:
\begin{align*}
\Delta W_t e^{s_\theta^\top\Delta W_t}
&=
\Delta W_t
\left(
1+s_\theta^\top\Delta W_t
+\frac12(s_\theta^\top\Delta W_t)^2
+O(\delta t^{3/2})
\right) \\
&=
\Delta W_t
+
\Delta W_t\Delta W_t^\top s_\theta
+
O(\delta t^{3/2}).
\end{align*}
Taking conditional expectation:
\begin{align*}
\mathbb E_{\mathbb P}
\left[
\Delta W_t e^{s_\theta^\top\Delta W_t}
\mid X_t
\right]
&=
\mathbb E_{\mathbb P}[\Delta W_t\mid X_t]
+
\mathbb E_{\mathbb P}
[\Delta W_t\Delta W_t^\top\mid X_t]
s_\theta
+
O(\delta t^{3/2}) \\
&=
a s_\theta\delta t
+
O(\delta t^{3/2}).
\end{align*}
Substituting into \eqref{eq:exp_w_q},
\begin{equation*}
\mathbb E_{\mathbb Q}[\Delta W_t\mid X_t]
=
\frac{
a s_\theta\delta t+O(\delta t^{3/2})
}{
1+\frac12 s_\theta^\top a s_\theta\delta t+O(\delta t^{3/2})
}.
\end{equation*}

Using
\[
(1+x)^{-1}
=
1-x+O(x^2),
\]
with $x=O(\delta t)$:
\begin{align*}
\mathbb E_{\mathbb Q}[\Delta W_t\mid X_t]
&=
\left(
a s_\theta\delta t+O(\delta t^{3/2})
\right)
\left(
1-\frac12 s_\theta^\top a s_\theta\delta t
+O(\delta t^{3/2})
\right) \\
&=
a s_\theta\delta t+O(\delta t^{3/2}).
\end{align*}
Substituting back into \eqref{eq:drift_q_decomp},
\(
\mathbb E_{\mathbb Q}[\Delta X_t\mid X_t]
=
(b+a s_\theta)\delta t
+
O(\delta t^{3/2}).
\)
Dividing by $\delta t$:
\[
\frac{
\mathbb E_{\mathbb Q}[\Delta X_t\mid X_t]
}{\delta t}
=
b+a s_\theta+O(\delta t^{1/2})
\approx
b+a s_\theta.
\]

\noindent
\textbf{4. Converse Direction and Uniqueness:}

Suppose an internal transition law $q$ satisfies
\[
\frac{
\mathbb E_{\mathbb Q}[\Delta X_t\mid X_t]
}{\delta t}
=
b+a s_\theta+O(\delta t^{1/2}).
\]
Consider the constrained optimization problem
\begin{equation}
\min_q
D_{\mathrm{KL}}(q\parallel p)
=
\min_q
\sum_{v\in\mathcal V_x}
q_v\log\!\left(\frac{q_v}{p_v}\right),
\label{eq:internal_kl_min}
\end{equation}
subject to
\[
\sum_v q_v=1
\text{ and }
\sum_v v q_v
=
(b+a s_\theta)\delta t
+
O(\delta t^{3/2}).
\]
The functional
\(
q\mapsto
\sum_v q_v\log(q_v/p_v)
\)
is strictly convex on the probability simplex, while the constraints are affine. Hence the constrained minimizer is unique.

Introduce Lagrange multipliers $\lambda_0\in{}^*\mathbb R$ and $\lambda\in{}^*\mathbb R^d$:
\begin{align*}
\mathcal L(q,\lambda_0,\lambda)
&=
\sum_v
q_v\log\!\left(\frac{q_v}{p_v}\right)
-
\lambda_0\left(\sum_v q_v-1\right) \\
&\quad
-
\lambda^\top
\left(
\sum_v v q_v
-
\mathbb E_{\mathbb Q}[\Delta X_t\mid X_t]
\right).
\end{align*}
Setting $\partial\mathcal L/\partial q_v=0$ gives
\[
\log\!\left(\frac{q_v}{p_v}\right)
+1-\lambda_0-\lambda^\top v
=0,
\]
hence
\[
q_v
=
p_v
\frac{\exp(\lambda^\top v)}
{\tilde Z_t},
\]
where
\(
\tilde Z_t
=
\exp(1-\lambda_0).
\)

Writing
\(
v=b\delta t+\Delta W_t,
\)
and cancelling the deterministic factor $\exp(\lambda^\top b\delta t)$,
\[
q_v
=
p_v
\frac{\exp(\lambda^\top\Delta W_t)}
{Z_t},
\]
where
\[
Z_t
=
\mathbb E_{\mathbb P}
[\exp(\lambda^\top\Delta W_t)\mid X_t].
\]
Repeating the same Taylor expansion argument as above yields
\[
\mathbb E_{\mathbb Q}[\Delta W_t\mid X_t]
=
a\lambda\delta t
+
O(\delta t^{3/2}).
\]
Since
\[
\mathbb E_{\mathbb Q}[\Delta X_t\mid X_t]
=
b\delta t
+
\mathbb E_{\mathbb Q}[\Delta W_t\mid X_t],
\]
matching the target drift gives
\[
a s_\theta\delta t
+
O(\delta t^{3/2})
=
a\lambda\delta t
+
O(\delta t^{3/2}).
\]
Because $a$ is uniformly nondegenerate,
\(
\lambda
=
s_\theta
+
O(\delta t^{1/2}).
\)
Substituting this back into the transition ratio:
\begin{align*}
\frac{q_v}{p_v}
&=
\frac{
\exp\!\left(
(s_\theta+O(\delta t^{1/2}))^\top\Delta W_t
\right)
}{Z_t} \\
&=
\frac{
\exp(s_\theta^\top\Delta W_t)
}{Z_t}
\exp(O(\delta t)).
\end{align*}
Since
\(
\exp(O(\delta t))
=
1+O(\delta t)
\approx1,
\)
we conclude
\[
\frac{q_v}{p_v}
\approx
\frac{
\exp(s_\theta^\top\Delta W_t)
}{Z_t}.
\]

This establishes the converse characterization and completes the proof.
\end{proof}

\begin{remark}
To identify a canonical change of measure on the hyperfinite grid, we restrict attention to internal transition kernels that preserve the local second-moment structure of the reference dynamics. More precisely, if the reference process $\mathbb{P}$ satisfies the infinitesimal moment condition
\(
\mathbb{E}_{\mathbb{P}}[\Delta X_t \Delta X_t^\top \mid X_t]
=
a\,\delta t + O(\delta t^{3/2}),
\)
then any admissible perturbed transition kernel defining an internal path measure $\mathbb{Q}$ must satisfy this same scaling. This invariance ensures that the quadratic variation of the induced continuous-time Loeb diffusion remains unaltered under the change of measure.\\
The first-order drift constraint alone does not uniquely determine the local transition probabilities; indefinitely many internal perturbations can yield the target drift while differing at higher orders. Such arbitrary variations alter the higher conditional moments of the grid increments, potentially destabilizing the fine-scale sample path behavior or violating the nonstandard CFL condition. \\
To isolate a distinguished, minimal perturbation, we impose the internal minimum relative entropy condition
\(
\min_{q} D_{\mathrm{KL}}(q \parallel p) = \min_{q} \sum_{v \in \mathcal{V}_x} q_v \log \frac{q_v}{p_v},
\)
subject to the normalization and target drift constraints. The corresponding Euler--Lagrange equations uniquely yield the exponential tilt
\[
q_v = p_v \frac{\exp\left(s_\theta(X_t, t)^\top (v - b(t,X_t)\delta t)\right)}{Z_t},
\]
which serves as the unique information-theoretic minimizer preserving positivity and internal absolute continuity with respect to the reference walk.\\
This optimization framework establishes that the exponential tilt is not an arbitrary modeling heuristic, but rather the unique configuration required for a rigorous hyperfinite Girsanov theorem. Because the logarithm of the resulting partition function admits the pathwise expansion
\(
\log Z_t = \frac12 s_\theta^\top a s_\theta \,\delta t + O(\delta t^{3/2}),
\)
it naturally generates the quadratic compensation term in the global Radon--Nikodym derivative upon hyperfinite summation. Consequently, the exponential tilt is uniquely compatible with the prescribed drift and the underlying geometry of the hyperfinite grid.
\end{remark}

The hyperfinite Girsanov identity established in Lemma \ref{lemma:girsanov_grid_rigorous} shows that the Radon--Nikodym derivative between internal path measures arises directly from the transition algebra of the hyperfinite grid. Unlike the classical continuous-time proof of Girsanov's theorem, which is formulated through stochastic integration and martingale criteria such as the Novikov condition, the hyperfinite construction proceeds entirely through finite-step transition ratios and hyperfinite summation. As has been the case throughout, our internal approach aligns with the foundational frameworks of hyperfinite diffusions (cf. \cite{keisler1984infinitesimal}).

More precisely, the change of measure is implemented locally through the exponential tilt of the reference transition probabilities, and the logarithm of the resulting density ratio is expanded using the internal Taylor expansion of the partition function $Z_t$. The accumulated pathwise exponent is obtained by summing over the hyperfinite timeline, and the remainder is infinitesimal under our assumptions. This structural decomposition mirrors applications of nonstandard measure theory to stochastic control \cite{albeverio2009nonstandard}. An internal perspective has also been leveraged by Cutland \cite{cutland1987infinitesimals} and Osswald \cite{osswald2009anticipative} to extend Girsanov transformations to anticipative settings. In our case, the internal logarithmic density ratio provides a rigorous hyperfinite representation of the continuous-time drift-energy functional associated with score-based diffusion dynamics.

\begin{theorem}[Hyperfinite Surrogate Likelihood Decomposition]
Let $\mathbb{G}\subset {}^*\mathbb{R}^d$ be a hyperfinite spatial lattice and consider the grid dynamics in Definition \ref{def:gd} with constant symmetric positive-definite diffusion tensor
\(
a=\sigma\sigma^\top \succ 0.
\)
Let $\mathbb{P}^*$ denote the internal path measure of the true backward diffusion process on the hyperfinite timeline
\(\mathbb{T}_{\delta t}\),
initialized at time $t=1$ with terminal forward distribution $p_1$, and governed by backward drift
\[
\bar b_{\mathrm{fwd}}(t,x)
=
-b(t,x)+a\,s_{\mathrm{opt}}(x,t).
\]
Let $\mathbb{Q}$ denote the parameterized backward generative path measure initialized with the same terminal distribution $p_1$, with drift
\[
\bar b_\theta(t,x)
=
-b(t,x)+a\,s_\theta(x,t)
\]
inducing the distribution \(q_\theta\) at \(t=0\), where both transition laws minimise the corresponding relative entropy as in Lemma \ref{lemma:girsanov_grid_rigorous}, and $\bar b_\theta, \bar b_{\mathrm{fwd}}$ denote the specializations of the drift from Definition~\ref{def:hr} and \eqref{eq:bd}, respectively.. With limited \(s_\theta\), assume regularity of \(b,\sigma\) implying limitedness of the score field $s_{\mathrm{opt}}$ on its support so that the internal Fisher score-matching functional $\mathcal I_a(\theta)$ is limited: 
\[
\mathcal I_a(\theta)
=
\mathbb E_{\mathbb P^*}
\left[
\sum_{t\in\mathbb T_{\delta t}\setminus\{0\}}
\frac12
\|s_\theta(X_t,t)-s_{\mathrm{opt}}(X_t,t)\|_a^2
\,\delta t
\right],
\]
and define the surrogate likelihood functional
\[
\mathcal L_{\mathrm{surr}}(\theta)
=
\mathbb E_{\mathbb P^*}
[
\log q_\theta(X_0)
].
\]
Then:
\begin{align}
\mathcal L_{\mathrm{surr}}(\theta)
=
H_0
-
\mathcal I_a(\theta)
+
\mathcal R_{\mathrm{bridge}}(\theta)
+
O(\delta t^{1/2}),
\label{eq:main_decomposition_theorem}
\end{align}
where
\(
H_0
=
\mathbb E_{\mathbb P^*}
[
\log \mathbb P^*(X_0)
]
\)
is independent of $\theta$, and
\[
\mathcal R_{\mathrm{bridge}}(\theta)
=
\mathbb E_{X_0\sim\mathbb P^*}
\left[
\mathbb D_{\mathrm{KL}}
\big(
\mathbb P^*(\cdot|X_0)
\|
\mathbb Q(\cdot|X_0)
\big)
\right]
\ge 0.
\]
Consequently,
\begin{equation}
\mathcal L_{\mathrm{surr}}(\theta)
\ge
H_0
-
\mathcal I_a(\theta)
+
O(\delta t^{1/2}).
\label{eq:surrogate_lower_bound}
\end{equation}
Applying the standard part map yields the exact standard inequality
\[
\operatorname{st}
(
\mathcal L_{\mathrm{surr}}(\theta)
)
\ge
\operatorname{st}
(
H_0-\mathcal I_a(\theta)
).
\]
\label{thm:likelihood-score-equiv}
\end{theorem}

\begin{proof}
\subsection*{Step 1: Global Hyperfinite Radon--Nikodym Expansion}

Let
\(
\Omega
=
\mathbb G^{\mathbb T_{\delta t}}
\)
denote the internal hyperfinite trajectory space.

The unconditional path measures factorize as
\begin{align*}
d\mathbb P^*(\omega)
&=
p_1(X_1)
\prod_{t\in\mathbb T_{\delta t}\setminus\{0\}}
p^*_{\Delta X_t}(X_t,t),
\\
d\mathbb Q(\omega)
&=
p_1(X_1)
\prod_{t\in\mathbb T_{\delta t}\setminus\{0\}}
q_{\Delta X_t}(X_t,t),
\end{align*}
where
\(
\Delta X_t
=
X_{t-\delta t}-X_t.
\)
Since both path measures share the identical initialization $p_1(X_1)$, cancellation yields
\begin{equation}
\log
\frac{d\mathbb P^*}{d\mathbb Q}(\omega)
=
\sum_t
\log
\left(
\frac{
p^*_{\Delta X_t}(X_t,t)
}{
q_{\Delta X_t}(X_t,t)
}
\right).
\label{eq:pathwise_ratio_start}
\end{equation}
Define the background backward increment
\(
\Delta W_t^\circ
=
\Delta X_t+b(t,X_t)\delta t.
\)
By Lemma~\ref{lemma:girsanov_grid_rigorous}, the local transition kernels admit the exponential tilt representations
\begin{align*}
p^*_{\Delta X_t}(X_t,t)
&=
p_v
\frac{
\exp
\left(
s_{\mathrm{opt}}(X_t,t)^\top \Delta W_t^\circ
\right)
}{
Z_t^*
},
\\
q_{\Delta X_t}(X_t,t)
&=
p_v
\frac{
\exp
\left(
s_\theta(X_t,t)^\top \Delta W_t^\circ
\right)
}{
Z_t
}.
\end{align*}
Substituting these formulas into \eqref{eq:pathwise_ratio_start} cancels the baseline walk kernel $p_v$:
\begin{align}
\log
\frac{d\mathbb P^*}{d\mathbb Q}(\omega)
&=
\sum_t
\Big[
(s_{\mathrm{opt}}-s_\theta)^\top
\Delta W_t^\circ
+
\log Z_t
-
\log Z_t^*
\Big].
\label{eq:pathwise_after_cancel}
\end{align}
Define the true backward martingale increment under $\mathbb P^*$:
\[
\Delta W_t^*
=
\Delta X_t
-
\bar b_{\mathrm{fwd}}(t,X_t)\delta t
=
\Delta W_t^\circ
-
a\,s_{\mathrm{opt}}(X_t,t)\delta t.
\]
Equivalently,
\[
\Delta W_t^\circ
=
\Delta W_t^*
+
a\,s_{\mathrm{opt}}(X_t,t)\delta t.
\]
Substituting this identity into \eqref{eq:pathwise_after_cancel} gives
\begin{align}
\log
\frac{d\mathbb P^*}{d\mathbb Q}(\omega)
&=
\sum_t
(s_{\mathrm{opt}}-s_\theta)^\top
\Delta W_t^*
+
\sum_t
(s_{\mathrm{opt}}-s_\theta)^\top
a\,s_{\mathrm{opt}}
\,\delta t
\nonumber
\\
&\quad
+
\sum_t
(\log Z_t-\log Z_t^*).
\label{eq:pathwise_before_square}
\end{align}
Using the partition-function expansion from Lemma~\ref{lemma:girsanov_grid_rigorous},
\[
\log Z_t
=
\frac12
\|s_\theta(X_t,t)\|_a^2
\delta t
+
O(\delta t^{3/2}),
\]
and similarly,
\[
\log Z_t^*
=
\frac12
\|s_{\mathrm{opt}}(X_t,t)\|_a^2
\delta t
+
O(\delta t^{3/2}).
\]
Therefore,
\(
\log Z_t-\log Z_t^*
=
\frac12
\Big(
\|s_\theta\|_a^2
-
\|s_{\mathrm{opt}}\|_a^2
\Big)\delta t
+
O(\delta t^{3/2}).
\)
Combining the deterministic terms yields
\[
(s_{\mathrm{opt}}-s_\theta)^\top a s_{\mathrm{opt}}
+
\frac12
\|s_\theta\|_a^2
-
\frac12
\|s_{\mathrm{opt}}\|_a^2
=
\frac12
\|s_{\mathrm{opt}}-s_\theta\|_a^2.
\]
Hence,
\begin{align}
\log
\frac{d\mathbb P^*}{d\mathbb Q}(\omega)
=
\sum_t
(s_{\mathrm{opt}}-s_\theta)^\top
\Delta W_t^*
+
\sum_t
\frac12
\|s_{\mathrm{opt}}-s_\theta\|_a^2
\delta t
+
\mathcal E(\omega),
\label{eq:pathwise_main_expansion}
\end{align}
where
\(
\mathcal E(\omega)
=
\sum_t
O(\delta t^{3/2}),
\)
which satisfies
\(
\mathcal E(\omega)
=
O(\delta t^{1/2}).
\)

\subsection*{Step 2: Evaluation of the Global Relative Entropy}

Taking expectation under $\mathbb P^*$ in \eqref{eq:pathwise_main_expansion} gives
\begin{align}
\mathbb D_{\mathrm{KL}}(\mathbb P^*\|\mathbb Q)
&=
\sum_t
\mathbb E_{\mathbb P^*}
\left[
(s_{\mathrm{opt}}-s_\theta)^\top
\Delta W_t^*
\right]
\nonumber
\\
&\quad
+
\mathbb E_{\mathbb P^*}
\left[
\sum_t
\frac12
\|s_{\mathrm{opt}}-s_\theta\|_a^2
\delta t
\right]
+
O(\delta t^{1/2}).
\label{eq:kl_before_cancel}
\end{align}
Let
\[
\mathcal F_t^*
=
\sigma(X_s:s\ge t)
\]
denote the backward filtration.
Since $\Delta W_t^*$ is the backward martingale increment under $\mathbb P^*$,
\[
\mathbb E_{\mathbb P^*}
[
\Delta W_t^*
\mid
\mathcal F_t^*
]
=
0.
\]
Because $s_{\mathrm{opt}}(X_t,t)$ and $s_\theta(X_t,t)$ are $\mathcal F_t^*$-measurable, the tower property gives
\begin{align*}
\mathbb E_{\mathbb P^*}
\left[
(s_{\mathrm{opt}}-s_\theta)^\top
\Delta W_t^*
\right]
&=
0.
\end{align*}
Therefore,
\begin{equation}
\mathbb D_{\mathrm{KL}}(\mathbb P^*\|\mathbb Q)
=
\mathcal I_a(\theta)
+
O(\delta t^{1/2}).
\label{eq:kl_equals_score}
\end{equation}

\subsection*{Step 3: Boundary Disintegration Identity}

Disintegrating the unconditional path measures with respect to the terminal state $X_0$ gives
\begin{align*}
d\mathbb P^*(\omega)
&=
\mathbb P^*(X_0)\,
d\mathbb P^*(\omega|X_0),
\\
d\mathbb Q(\omega)
&=
q_\theta(X_0)\,
d\mathbb Q(\omega|X_0).
\end{align*}
Consequently,
\begin{align}
\log
\frac{d\mathbb P^*}{d\mathbb Q}
=
\log \mathbb P^*(X_0)
-
\log q_\theta(X_0)
+
\log
\frac{
d\mathbb P^*(\cdot|X_0)
}{
d\mathbb Q(\cdot|X_0)
}.
\label{eq:bridge_disintegration}
\end{align}
Taking expectation under $\mathbb P^*$ yields
\begin{align}
\mathbb D_{\mathrm{KL}}(\mathbb P^*\|\mathbb Q)
&=
\mathbb E_{\mathbb P^*}
[
\log \mathbb P^*(X_0)
]
-
\mathbb E_{\mathbb P^*}
[
\log q_\theta(X_0)
]
\nonumber
\\
&\quad
+
\mathbb E_{X_0\sim\mathbb P^*}
\left[
\mathbb D_{\mathrm{KL}}
\big(
\mathbb P^*(\cdot|X_0)
\|
\mathbb Q(\cdot|X_0)
\big)
\right].
\label{eq:bridge_kl_identity}
\end{align}

Using the definitions
\[
H_0
=
\mathbb E_{\mathbb P^*}
[
\log \mathbb P^*(X_0)
],
\]
\[
\mathcal L_{\mathrm{surr}}(\theta)
=
\mathbb E_{\mathbb P^*}
[
\log q_\theta(X_0)
],
\]
and
\[
\mathcal R_{\mathrm{bridge}}(\theta)
=
\mathbb E_{X_0\sim\mathbb P^*}
\left[
\mathbb D_{\mathrm{KL}}
\big(
\mathbb P^*(\cdot|X_0)
\|
\mathbb Q(\cdot|X_0)
\big)
\right],
\]
equation \eqref{eq:bridge_kl_identity} becomes
\begin{equation}
\mathbb D_{\mathrm{KL}}(\mathbb P^*\|\mathbb Q)
=
H_0
-
\mathcal L_{\mathrm{surr}}(\theta)
+
\mathcal R_{\mathrm{bridge}}(\theta).
\label{eq:bridge_kl_clean}
\end{equation}
Combining \eqref{eq:kl_equals_score} and \eqref{eq:bridge_kl_clean} yields
\[
H_0
-
\mathcal L_{\mathrm{surr}}(\theta)
+
\mathcal R_{\mathrm{bridge}}(\theta)
=
\mathcal I_a(\theta)
+
O(\delta t^{1/2}).
\]
Rearranging gives
\[
\mathcal L_{\mathrm{surr}}(\theta)
=
H_0
-
\mathcal I_a(\theta)
+
\mathcal R_{\mathrm{bridge}}(\theta)
+
O(\delta t^{1/2}).
\]
Since relative entropy is nonnegative,
\(
\mathcal R_{\mathrm{bridge}}(\theta)\ge 0.
\)
Therefore,
\[
\mathcal L_{\mathrm{surr}}(\theta)
\ge
H_0
-
\mathcal I_a(\theta)
+
O(\delta t^{1/2}).
\]
Finally, applying the standard part operator and using
\(
O(\delta t^{1/2})
\approx 0
\)
gives
\[
\operatorname{st}
(
\mathcal L_{\mathrm{surr}}(\theta)
)
\ge
\operatorname{st}
(
H_0-\mathcal I_a(\theta)
).
\]

This completes the proof.
\end{proof}

The equivalence established in Theorem \ref{thm:likelihood-score-equiv} provides a rigorous structural explanation for the effectiveness of score-based training and shows that it is not merely an approximation to maximum likelihood, but a pathwise realization of it within the hyperfinite probabilistic framework. Standard derivations in the diffusion-model literature \cite{ho2020denoising,song2021} typically relate score matching to likelihood optimization through variational lower bounds (ELBOs), Gaussian perturbation identities, or continuum integration-by-parts arguments. In contrast, the hyperfinite framework shows that the connection emerges directly from the algebra of the grid transition structure itself, where calculations are performed through finite-step transition ratios and internal summation, with the continuum limit recovered through the Loeb measure construction.\\
More precisely, the internal Girsanov identity of Lemma \ref{lemma:girsanov_grid_rigorous} expresses the logarithmic density ratio between path measures as a hyperfinite sum of local transition contributions. After expanding the local exponential tilt and summing over the hyperfinite timeline, the stochastic first-order terms appear as martingale increments with respect to the backward filtration. Their contribution vanishes exactly under the internal expectation by the tower property and the martingale property of hyperfinite increments \cite{lindstrom1980hyperfinite1,lindstrom1980hyperfinite3}. The remaining deterministic contribution is precisely the quadratic score discrepancy functional.\\
Consequently, the Fisher divergence does not arise as a heuristic surrogate for likelihood, but as the quadratic compensation term in the hyperfinite change-of-measure formula. In particular, after taking standard parts, minimizing the score-matching objective maximizes a lower bound on the surrogate likelihood. 

Since $a$ is a constant positive-definite matrix, there exist constants $0 < \lambda_{\min} \le \lambda_{\max}$ such that
$$\lambda_{\min} \|v\|_{\mathbb{G}}^2 \le \|v\|_a^2 \le \lambda_{\max} \|v\|_{\mathbb{G}}^2,$$ implying norm equivalence.
As in the proof of Theorem \ref{thm:internal-score}, for an arbitrary real \(\epsilon>0\), consider a representative in \(\mathbb{T}_{\delta t}\) given by \(\tau\approx \epsilon\) and define \(\mathbb{T}'_{\delta t}:=\mathbb{T}_{\delta t}\cap [\tau,1]\).

\begin{lemma}[Internal Score Decomposition]
\label{lemma:score-decomp}
Let $q$ be the initial distribution and let $P_{t|0}$ denote the internal forward transition kernel on the hyperfinite grid $\mathbb{G}$. Assume the regularity hypotheses of Theorem \ref{thm:internal-score}. For appreciable times $t$ in \(\mathbb{T}'_{\delta t}\), standard parabolic regularity and uniform ellipticity imply that the logarithmic grid derivatives $\Delta^+ \log P_{t|0}$ are limited and internally square-integrable with respect to the internal measure
\(
q(x_0)P_{t|0}(x_t|x_0)\delta t.
\)
Let
\[
s_{\mathrm{opt}}(x_t,t)
=
\sum_{x_0\in\mathbb G}
\left[
\frac{q(x_0)P_{t|0}(x_t|x_0)}{p_t(x_t)}
\right]
\Delta^+\log P_{t|0}(x_t|x_0)
\]
denote the optimal internal score projection from Theorem \ref{thm:internal-score}. Then the internal DSM objective admits the orthogonal decomposition
\begin{equation*}
\mathcal I(\theta)
=
\frac12
\sum_{t\in\mathbb T'_{\delta t}}
\mathbb E_{p_t}
\left[
\|s_\theta-s_{\mathrm{opt}}\|_{\mathbb G}^2
\right]\delta t
+
C',
\end{equation*}
where
\begin{equation*}
C'
=
\frac12
\sum_{t\in\mathbb T'_{\delta t}}
\mathbb E_{q,P_{t|0}}
\left[
\|s_{\mathrm{opt}}-\Delta^+\log P_{t|0}\|_{\mathbb G}^2
\right]\delta t
\end{equation*}
is a limited hyperreal independent of $\theta$.
\end{lemma}

\begin{proof}
For each fixed $t\in\mathbb T'_{\delta t}$, expand the square appearing in the DSM objective:
\begin{align*}
\|s_\theta-\Delta^+\log P_{t|0}\|_{\mathbb G}^2
&=
\|s_\theta-s_{\mathrm{opt}}\|_{\mathbb G}^2
+
\|s_{\mathrm{opt}}-\Delta^+\log P_{t|0}\|_{\mathbb G}^2 \\
&\quad
+
2\left\langle
s_\theta-s_{\mathrm{opt}},
\,\,
s_{\mathrm{opt}}-\Delta^+\log P_{t|0}
\right\rangle_{\mathbb G}.
\end{align*}
\noindent
We now take the expectation with respect to the internal measure
\(
q(x_0)P_{t|0}(x_t|x_0).
\)
Applying the internal tower property to the cross-term, we condition on $X_t$. Since $s_\theta(X_t,t)$ and $s_{\mathrm{opt}}(X_t,t)$ are measurable with respect to the $\sigma$-algebra generated by $X_t$, they factor out of the inner conditional expectation:
\begin{align*}
& \mathbb E_{q,P_{t|0}} \left[ \left\langle s_\theta-s_{\mathrm{opt}}, \,\, s_{\mathrm{opt}}-\Delta^+\log P_{t|0} \right\rangle_{\mathbb G} \right] \\
&= \mathbb E_{p_t} \left[ \left\langle s_\theta-s_{\mathrm{opt}}, \,\, \mathbb E_{q,P_{t|0}} \left[ s_{\mathrm{opt}}-\Delta^+\log P_{t|0} \,\middle|\, X_t \right] \right\rangle_{\mathbb G} \right].
\end{align*}
By the definition of the optimal score projection, $\mathbb E_{q,P_{t|0}}[\Delta^+\log P_{t|0} \mid X_t] = s_{\mathrm{opt}}(X_t,t)$. Because $s_{\mathrm{opt}}(X_t,t)$ is constant given $X_t$, the inner conditional expectation vanishes identically:
\[
\mathbb E_{q,P_{t|0}} \left[ s_{\mathrm{opt}}-\Delta^+\log P_{t|0} \,\middle|\, X_t \right] = s_{\mathrm{opt}}(X_t,t) - s_{\mathrm{opt}}(X_t,t) = 0.
\]
Hence,
\begin{align*}
&
\mathbb E_{q,P_{t|0}}
\left[
\|s_\theta-\Delta^+\log P_{t|0}\|_{\mathbb G}^2
\right]
\\
&=
\mathbb E_{p_t}
\left[
\|s_\theta-s_{\mathrm{opt}}\|_{\mathbb G}^2
\right]
+
\mathbb E_{q,P_{t|0}}
\left[
\|s_{\mathrm{opt}}-\Delta^+\log P_{t|0}\|_{\mathbb G}^2
\right].
\end{align*}

Multiplying by $\frac12\delta t$ and summing over
$t\in\mathbb T'_{\delta t}$ yields
\begin{align*}
\mathcal I(\theta)
&=
\frac12
\sum_{t\in\mathbb T'_{\delta t}}
\mathbb E_{p_t}
\left[
\|s_\theta-s_{\mathrm{opt}}\|_{\mathbb G}^2
\right]\delta t
\\
&\quad
+
\frac12
\sum_{t\in\mathbb T'_{\delta t}}
\mathbb E_{q,P_{t|0}}
\left[
\|s_{\mathrm{opt}}-\Delta^+\log P_{t|0}\|_{\mathbb G}^2
\right]\delta t.
\end{align*}

The second term depends only on the forward process and the optimal projection $s_{\mathrm{opt}}$, and is therefore independent of $\theta$. Defining this term to be $C'$ completes the proof.
\end{proof}

\begin{remark}
\label{rem:tt}
The restriction to the truncated timeline
\[
\mathbb{T}'_{\delta t}
=
\mathbb{T}_{\delta t}\cap[\tau,1],
\qquad
\tau\approx\epsilon>0,
\]
excludes the singular initial layer near $t=0$, where the logarithmic grid derivatives of the transition kernel need not remain uniformly limited. This restriction does not affect the standard-part conclusions of the theory under some conditions: Indeed, if $s_\theta$ and $s_{\mathrm{opt}}$ are limited internal functions, since the interval \([0,\tau]\) has arbitrarily small Loeb measure, its contribution to the integrated quadratic objective is \(\mathcal{O}(\epsilon)\). Consequently, the truncated objective converges to the full-timeline counterpart as $\epsilon\downarrow 0$, and both have the same standard-part limit.
\end{remark}

The decomposition established in Lemma \ref{lemma:score-decomp} provides a rigorous justification for the use of DSM as a surrogate for likelihood optimization. By the internal Pythagorean theorem and the orthogonality of the conditional projection $s_{\mathrm{opt}}$, the DSM objective decomposes into the projection error
\(
\frac12
\sum_{t\in\mathbb T'_{\delta t}}
\mathbb E_{p_t}
\!\left[
\|s_\theta-s_{\mathrm{opt}}\|_{\mathbb G}^2
\right]\delta t
\)
plus a $\theta$-independent residual term $C'$. Since the diffusion tensor $a$ is uniformly positive definite, the norms $\|\cdot\|_{\mathbb G}$ and $\|\cdot\|_a$ are equivalent, so $\mathcal I(\theta)$ and the corresponding pathwise score-matching functional induce equivalent minimization problems. This recovers, within the hyperfinite framework, the structural relationship between DSM and Fisher divergence established by Vincent \cite{vincent2011connection}. Moreover, on appreciable time intervals, uniform ellipticity and parabolic regularity imply that the transition densities are $S$-smooth and possess limited logarithmic grid derivatives on the finite part of the Loeb space, ensuring that the residual term $C'$ is a limited hyperreal.

Taken together, Theorem \ref{thm:internal-score}, Lemma \ref{lemma:score-decomp}, and Theorem \ref{thm:likelihood-score-equiv} establish a variational connection between DSM and surrogate likelihood maximization on the hyperfinite grid. On every appreciable time interval $[\tau,1]$, Theorem \ref{thm:internal-score} identifies the DSM minimizer with the optimal internal projection field $s_{\mathrm{opt}}$, while Lemma \ref{lemma:score-decomp} shows that the DSM objective differs from the corresponding pathwise score-matching functional only by a $\theta$-independent constant. Theorem \ref{thm:likelihood-score-equiv} then relates this pathwise score-matching functional to the surrogate likelihood through a non-negative bridge divergence term. Consequently, DSM training admits a likelihood-based interpretation on appreciable time intervals. Extending this identification to the full timeline requires additional control of the infinitesimal initial layer $[0,\tau)$ as $\tau \to 0$, which may be obtained under further regularity or integrability assumptions on the forward diffusion and its associated densities (also see Remark \ref{rem:tt}).

\section{Local expansion of a source--based hyperfinite master equation}
\label{sec:sec}
To keep the exposition and notation at an elementary level, we focus on the one-dimensional case here, even though the result may be extended to the general $d$-dimensional case we have been discussing.\\
Fix a finite point $(t,x)$ and set $P = {}^\ast p$. \\
The internal Euler increment is $h = {}^\ast b(t,x)\delta t + {}^\ast\sigma(t,x)\sqrt{\delta t}\xi$, where \(\xi\) is an internal random variable satisfying \[
\mathbb E_{\text{int}}[\xi]=0,\qquad
\mathbb E_{\text{int}}[\xi^2]=1,\qquad
\mathbb E_{\text{int}}[\xi^3]=0,\qquad
\mathbb E_{\text{int}}[\xi^4]=\kappa<\infty.
\] The source-based master equation is (cf. \cite{nelson1987radically}):
\begin{equation}
\label{eq:mas}
    P(t+\delta t,x) = \mathbb{E}_\xi\left[ P(t,x - h) \right].
\end{equation}

\subsection{Operator consistency}
Define the internal frozen operator $\mathbb{A} = -{}^\ast b \partial + \frac{{}^\ast a}{2} \partial^2$, where \(\partial\equiv\partial_x\), the internal spatial derivative. We have the exact internal identity:
\begin{equation}
    \mathbb{A}^2 P = {}^\ast b^2 \partial^2 P - {}^\ast b {}^\ast a \partial^3 P + \frac{{}^\ast a^2}{4} \partial^4 P.
    \label{eq:a2}
\end{equation}
The operator
\(
\mathbb A=-{}^\ast b(t,x)\partial
+\frac{{}^\ast a(t,x)}2\partial^2
\)
is the frozen-coefficient representation of the adjoint generator at the spacetime point $(x,t)$. For $S$-regular liftings $P$ of a standard density $p$, one has
\(
\operatorname{st}(\mathbb A P)
=
\mathcal L_{t,\mathrm{fr}}^* p,
\)
where
\(
\mathcal L_{t,\mathrm{fr}}^*
=
-b(t,x)\partial_x
+\frac{a(t,x)}2\partial_x^2
\)
denotes the Fokker--Planck generator with coefficients frozen at $(t,x)$.

\begin{theorem}[Local Second-Order Consistency]
In addition to the conditions on \(b,\sigma\) from Definition \ref{def:gd}, assume that \(b\) and \(\sigma\) are \(C^{1+\tfrac{\alpha}{2},3+\alpha}\) and \(C^{1+\tfrac{\alpha}{2},4+\alpha}\), respectively, for some \(\alpha\in (0,1)\).
For any finite $(t,x)$ with $t>0$, the master equation update satisfies:
\begin{equation}
    \operatorname{st} \left( \frac{P(t+\delta t,x) - \left( P + \delta t \mathbb{A} P + \frac{\delta t^2}{2} \mathbb{A}^2 P \right)}{\delta t^2} \right) = \frac{\kappa - 3}{24} a({}^\circ t,{}^\circ x)^2 \frac{\partial^4 p}{\partial x^4}({}^\circ t,{}^\circ x).
\end{equation}
\label{thm:lc}
\end{theorem}
\begin{proof}
By the internal Taylor theorem and \eqref{eq:mas}:
\begin{align*}
    P(t+\delta t,x) &= P + \delta t \mathbb{A} P + \delta t^2 \left( \frac{{}^\ast b^2}{2} \partial^2 P - \frac{{}^\ast b{}^\ast a}{2} \partial^3 P + \frac{\kappa {}^\ast a^2}{24} \partial^4 P \right) + o(\delta t^2),
\end{align*}
where we used
\begin{align*}
    \mathbb{E}_\xi[h] &= {}^\ast b \delta t, \\
    \mathbb{E}_\xi[h^2] &= {}^\ast a \delta t + {}^\ast b^2 \delta t^2, \\
    \mathbb{E}_\xi[h^3] &= 3{}^\ast b{}^\ast a \delta t^2 + O(\delta t^3), \\
    \mathbb{E}_\xi[h^4] &= \kappa {}^\ast a^2 \delta t^2 + O(\delta t^3).
\end{align*}
Using \eqref{eq:a2} and subtracting the semigroup expansion:
\begin{align*}
    P(t+\delta t,x) &- \left( P + \delta t \mathbb{A} P + \frac{\delta t^2}{2} \mathbb{A}^2 P \right)\\ &= \delta t^2 \bigg[ \left( \frac{{}^\ast b^2}{2} \partial^2 P - \frac{{}^\ast b{}^\ast a}{2} \partial^3 P + \frac{\kappa {}^\ast a^2}{24} \partial^4 P \right) \\ & - \left( \frac{{}^\ast b^2}{2} \partial^2 P - \frac{{}^\ast b{}^\ast a}{2} \partial^3 P + \frac{3 {}^\ast a^2}{24} \partial^4 P \right) \bigg] + o(\delta t^2) \\
    &= \delta t^2 \left( \frac{\kappa - 3}{24} {}^\ast a^2 \partial^4 P \right) + o(\delta t^2).
\end{align*}
Dividing by $\delta t^2$ and taking the standard part, noting that $\operatorname{st}(\partial^4 P) = \partial^4 p$ due to $S$-smoothness resulting from parabolic regularity, yields the result.
\end{proof} 

\begin{remark}[Significance of Second-Order Consistency]
Theorem \ref{thm:lc} shows that, for the source--based hyperfinite master equation considered here, local second--order consistency with the diffusion semigroup is achieved precisely when the driving increment has kurtosis $\kappa=3$. Any mismatch in the fourth moment produces a non--vanishing second--order correction proportional to the fourth spatial derivative of the density. The resulting term
\(
\frac{\kappa-3}{24}a^2\partial_x^4 p
\)
may be interpreted as a numerical dispersion error induced solely by the fourth--order statistics of the increment. Consequently, while first--order (Euler--scale) consistency requires only matching the first two moments, second--order weak consistency additionally depends on the kurtosis of the underlying noise source. Rademacher increments do not satisfy this condition, whereas the following non--Gaussian distribution does:
\(
\Delta W_t=\sqrt{\delta t}\,\xi_t,
\)
where the $\xi_t$ are independent and satisfy
\[
\mathbb{P}(\xi_t=\pm\sqrt{3})=\frac16,
\qquad
\mathbb{P}(\xi_t=0)=\frac23.
\]
\end{remark}

While Theorem \ref{thm:lc} assumes locally frozen coefficients, it identifies the leading contribution arising from fourth--moment mismatch in the driving noise. In the general case where $b$ and $a$ vary spatially, additional $O(\delta t^2)$ terms appear through derivatives of the coefficients and their interaction with the density $P$. Those contributions depend on the specific discretization and coefficient structure, whereas the term isolated in Theorem \ref{thm:lc} depends only on the kurtosis of the increment. Consequently, any scheme based on the source--equation representation must match the fourth moment of the target diffusion in order to eliminate this particular second--order error.

The significance of second--order consistency for generative modeling lies in the fidelity of the score target. Since the optimal estimator $s_{\mathrm{opt}}$ is determined by the density generated by the hyperfinite forward process, a second--order error in the master equation propagates into the corresponding score field. Theorem \ref{thm:lc} shows that a mismatch in kurtosis produces a density error of order $O(\delta t^2)$, and hence a corresponding second--order perturbation of the learned score. Thus, within the source--based hyperfinite framework, matching the fourth moment of the increment is necessary to remove this particular source of second--order bias. This provides a precise mathematical explanation for the advantage of Gaussian noise, and more generally of kurtosis--matched noise sources, in approximating diffusion semigroups at second order \cite{kloeden2012numerical}.

\section{Conclusion and Outlook}
\label{sec:conclusion}

In this paper, we developed a hyperfinite framework for score-based diffusion models using the tools of nonstandard stochastic analysis. Rather than treating discrete-time dynamics as approximations to a continuum limit, we studied stochastic evolution directly on a hyperfinite grid and showed how the familiar objects of diffusion theory arise through the standard-part map.
Our analysis demonstrates that several key structures underlying score-based generative modeling admit natural hyperfinite representations. In particular:

\begin{itemize}
\item \textbf{Generator and Fokker--Planck Structure.} We derived the internal infinitesimal generator associated with hyperfinite walk dynamics and established its correspondence with the classical Fokker--Planck equation after taking standard parts.

\item \textbf{Backward Mean and Reverse-Time Dynamics.} We obtained an explicit hyperfinite backward-mean identity and showed how the score term arises from the interaction between local density variation and the geometry of the internal grid transition. This yields a constructive derivation of the reverse drift and leads to a hyperfinite derivation of the reverse-time SDE.

\item \textbf{Score Matching and Generative Consistency.} We established that minimization of an internal score-matching objective recovers the score required by the reverse-time dynamics and used this connection to justify the convergence of the corresponding hyperfinite sampling procedure.

\item \textbf{Likelihood and Higher-Order Structure.} Under appropriate assumptions, we derived a hyperfinite Girsanov formula and related likelihood optimization to Fisher-divergence objectives. We also quantified the leading second-order consistency error of the hyperfinite dynamics, identifying the role played by the fourth moment of the increment distribution.
\end{itemize}

From a broader perspective, the results suggest that many constructions appearing in modern diffusion modeling can be formulated directly at the hyperfinite level prior to standard-part projection. This viewpoint provides a complementary perspective on the relationship between discrete algorithms and continuous-time diffusion theory, while retaining compatibility with the classical framework of Loeb spaces, stochastic integration, and martingale methods.

The theoretical framework established in this paper opens several fertile avenues for future research:
\begin{enumerate}
    \item \textbf{Heavy-Tailed SDEs and L\'evy-Flight Models:} While the present work focuses on Gaussian-like noise, the hyperfinite framework is naturally suited to the study of heavy-tailed and non-local processes. Building on the hyperfinite jump models of \cite{lindstrom2004hyperfinite} and the nonstandard stochastic calculus of \cite{albeverio2009nonstandard}, one could derive reverse-time identities for stable-distribution diffusion processes. This would provide a rigorous foundation for L\'evy-flight generative models in regimes where classical martingale machinery and standard density regularity often become intractable.
    
    \item \textbf{Numerical Analysis of Higher-Order Sampling:} Our identification of the numerical dispersion term $\frac{\kappa - 3}{24} a^2 \partial^4 p$ suggests a new frontier in the design of sampling kernels. By constructing internal noise increments that explicitly minimize higher-order Taylor residuals on the hyperfinite grid, it may be possible to develop sampling algorithms with order-of-convergence properties that surpass standard Euler-Maruyama benchmarks \cite{kloeden2012numerical}. This shifts the focus from purely empirical step-size tuning to the analytical design of the hyperfinite increment's spectral properties.
    
    \item \textbf{Path-Space Inference and Uncertainty Quantification:} The hyperfinite Girsanov identity offers a rigorous, white-box tool for analyzing inference in generative AI. The ability to calculate exact internal likelihood ratios between path measures provides a more granular framework for studying mode collapse, distributional shift, and out-of-distribution detection—phenomena that part of the deep learning community currently addresses via empirical heuristics rather than rigorous path-space analysis.
\end{enumerate}

We hope that the framework developed here helps clarify the connections between nonstandard stochastic analysis and modern generative modeling, and that it encourages further interaction between these two areas.

\section*{Declaration}
No potential conflict of interest was reported by the author(s).

\section*{Additional information}
\subsection*{Funding}
No funding was received for conducting this research.

\bibliographystyle{plain}
\bibliography{ref}

\appendix
\section{Mathematical preliminaries}
\label{sec:preliminaries}

This section collects the requisite nonstandard-analytic material. Our presentation follows the superstructure framework for rigorous handling of probability spaces and higher-order objects; for much of this material, the reader may refer to \cite{robinson1974non, stroyan1977introduction}. Readers familiar with the basic constructions in NSA may skip directly to Section~\ref{sec:grid-calculus}, reviewing only the remainder of this section and Appendix~\ref{sec:discrete-ito} as needed.

\subsection{Hyperreal field and ultrapower construction}
\label{sec:hyperreals}
We start with the definition of hyperreal numbers and some classical concepts of fundamental importance to this work.

\begin{definition}[Free Ultrafilter]
\label{def:free-ultrafilter}
Let $J$ be a nonempty set. A filter $\mathcal U$ on $J$ is a subset of $\mathcal P(J)$, the power set of $J$, satisfying the following properties:
\begin{enumerate}
\item \textbf{Properness:} $\varnothing \notin \mathcal U$.
\item \textbf{Finite intersection property:} If $A,B \in \mathcal U$, then $A \cap B \in \mathcal U$.
\item \textbf{Upward closure:} If $A \in \mathcal U$ and $A \subseteq B \subseteq J$, then $B \in \mathcal U$.
\end{enumerate}
The filter $\mathcal U$ is called an ultrafilter if it also satisfies:
\begin{enumerate}
\setcounter{enumi}{3}
\item \textbf{Maximality:} For every $A \subseteq J$, either $A \in \mathcal U$ or $J \setminus A \in \mathcal U$.
\end{enumerate}
An ultrafilter $\mathcal U$ is called a free ultrafilter if it also satisfies:
\begin{enumerate}
\setcounter{enumi}{4}
\item \textbf{Freeness/Nonprincipality:} $\mathcal U$ contains no finite subsets of $J$.
\end{enumerate}
\end{definition}
The existence of a free ultrafilter may be proved using Zorn's lemma \cite{stroyan1977introduction}. 

\begin{definition}[Ultrapower/Hyperreal Numbers]
Let $\mathcal U$ be a fixed free ultrafilter on $\mathbb N$.  
Let $\mathbb R^{\mathbb N}$ denote the set of real sequences. Define an equivalence relation $\sim$ on $\mathbb R^{\mathbb N}$ by declaring
\[
(x_n) \sim (y_n)
\quad\text{if and only if}\quad
\{n \in \mathbb N : x_n = y_n\} \in \mathcal U.
\]
The ultrapower of $\mathbb R$ is the quotient
\[
{}^\ast\mathbb R := \mathbb R^{\mathbb N} / \!\sim.
\]
Denote the equivalence class of a sequence $(x_n)$ by $[x_n]$.
\end{definition}

\begin{remark}
Arithmetic operations and the order on $\mathbb R$ lift coordinatewise to $\mathbb R^{\mathbb N}$ and descend to well-defined operations on ${}^\ast\mathbb R$ via
\[
[x_n] + [y_n] := [x_n + y_n],
\qquad
[x_n]\cdot [y_n] := [x_n y_n],
\]
and
\[
[x_n] < [y_n]
\quad\text{if and only if}\quad
\{n \in \mathbb N : x_n < y_n\} \in \mathcal U.
\]
The resulting structure ${}^\ast\mathbb R$ is a totally ordered field extending $\mathbb R$. We identify each real number $r \in \mathbb R$ with the equivalence class of the constant sequence $(r,r,r,\dots)$.
\end{remark}
Thus, we consider \(|x|<M\), for a hyperreal \(x=[x_n]\) and positive real \(M\), as equivalent to \(\{n : |x_n|<M\}\in\mathcal{U}\). 

\begin{definition}[Finite, Infinite, Infinitesimal]
An element $x\in{}^\ast\mathbb{R}$ is finite (or limited) if there exists $M\in\mathbb{R}_{>0}$ such that $|x|<M$; $x$ is infinitesimal if $|x|<r$ for every real $r>0$; $x$ is infinite if $|x|>M$ for every real $M>0$. We write $x\approx y$ when $x-y$ is infinitesimal, and say $x$ is infinitely close to $y$. For \(x,y\in {}^*\mathbb{R}^d\), we define \(x\approx y\) componentwise.
\end{definition}

\begin{proposition}[Standard Part Map]
\label{prop:standard-part}
Every finite $x\in{}^\ast\mathbb{R}$ has a unique nearest real number $\mathrm{st}(x)\in\mathbb{R}$, called its standard part, satisfying $x \approx \mathrm{st}(x)$. The map $\mathrm{st}:\{x\in{}^\ast\mathbb{R}: x \text{ finite}\}\to\mathbb{R}$ is a ring homomorphism onto $\mathbb{R}$ with kernel the set of infinitesimals.
\end{proposition}

\begin{proof}
Let $x=[x_n]$ be represented by a real sequence $(x_n)$. Finiteness implies $(x_n)$ is bounded in $\mathbb{R}$. Let $L = \sup \{ r \in \mathbb{R} : \{n : x_n < r\} \notin \mathcal{U} \}$ and $U = \inf \{ r \in \mathbb{R} : \{n : x_n > r\} \notin \mathcal{U} \}$. One can show $L = U \eqqcolon \ell$ and that $\{n : |x_n - \ell| < \epsilon\} \in \mathcal{U}$ for every standard $\epsilon > 0$. Define $\mathrm{st}(x)=\ell$. Uniqueness and the homomorphism properties follow from the properties of $\mathcal{U}$.
\end{proof}
\begin{remark}
    In certain contexts, it is less cumbersome, and also common in the NSA literature, to use the notation \({}^\circ x\) for \(\mathrm{st}(x)\); we will use both.
\end{remark}

\subsection{Superstructure framework, internal objects, hyperfinite sets, and transfer}
\label{sec:superstructure}

Constructing hyperreal numbers is a good start, but that does not take us very far in terms of analysis. For rigorous handling of relations, functions, probability spaces, and higher-order objects in general, we work within the superstructure framework.

\begin{definition}[Superstructure]
Let $V_0(\mathbb{R}) = \mathbb{R}$. Define recursively:
\[
V_{n+1}(\mathbb{R}) = V_n(\mathbb{R}) \cup \mathcal{P}(V_n(\mathbb{R})).
\]
The superstructure over $\mathbb{R}$ is $V(\mathbb{R}) = \bigcup_{n=0}^\infty V_n(\mathbb{R})$. Elements of $V(\mathbb{R})$ are called standard objects.
\end{definition}

\begin{definition}[Bounded Ultrapower]
Let $\mathcal{U}$ be a nonprincipal ultrafilter on $\mathbb{N}$. The bounded ultrapower $V({}^\ast\mathbb{R})$ is constructed as follows. For each $n$, define:
\[
V_n({}^\ast\mathbb{R}) = \left\{ [(A_m)_{m\in\mathbb{N}}] : A_m \in V_n(\mathbb{R}) \right\},
\]
where $[A_n] = [B_n]$ if $\{n : A_n = B_n\} \in \mathcal{U}$. The bounded ultrapower is $V({}^\ast\mathbb{R}) = \bigcup_{n=0}^\infty V_n({}^\ast\mathbb{R})$.
\end{definition}

The hyperreal field ${}^\ast\mathbb{R}$ is the set $V_1({}^\ast\mathbb{R})$ in our superstructure framework identified with the image of sequences of real numbers.

\begin{definition}[Internal Objects]
An object $A \in V({}^\ast\mathbb{R})$ is internal if $A \in V_n({}^\ast\mathbb{R})$ for some $n$. Objects in $V({}^\ast\mathbb{R})$ that are not internal are called external.
\end{definition}

To be more specific, we have:
\begin{definition}[Internal Sets]
A subset $A\subseteq{}^\ast\mathbb{R}$ is internal if $A \in V_2({}^\ast\mathbb{R})$, i.e., there exists a sequence of subsets $(A_n)_{n\in\mathbb{N}}$ with $A_n\subseteq\mathbb{R}$ such that $A=[A_n]:=\{ [x_n] \in {}^\ast\mathbb{R} : \{n : x_n \in A_n\} \in \mathcal{U} \}$. A subset of ${}^\ast\mathbb{R}$ that is not internal is called external.
\end{definition}

\begin{definition}[Internal Functions]
A function $f: A \to B$ between internal sets is internal if its graph $\{(x, f(x)) : x \in A\}$ is an internal subset of $A \times B$. Equivalently, with \(A=[A_n]\) and \(B=[B_n]\), $f = [f_n]$ where each $f_n: A_n \to B_n$ is a function between standard sets.
\end{definition}

Before proceeding, it is necessary to emphasize a subtle, yet critical, nuance.
\begin{remark}[Notation for internal extensions]
For the sake of readability, we frequently omit the asterisk prefix ${}^\ast$ for standard functions and operations when their use on non-standard domains is clear from context. For instance, we write $\|x\|$ for the internal extension of the Euclidean norm ${}^\ast\|\cdot\|$, and for a standard test function $\varphi$, we may write $\varphi(x)$ instead of ${}^\ast\varphi(x)$ for $x$ in a nonstandard domain. Furthermore, in much of our presentation, derivatives would also need to be understood in a proper internal sense even if we use the ``standard" notation. All of these should be reasonably clear from context.
\end{remark}

\begin{definition}[Hyperfinite Sets]
An internal set $H\subset{}^\ast\mathbb{R}$ is hyperfinite if there exists a hypernatural number $N \in {}^\ast\mathbb{N}$ and an internal bijection $f: \{1, 2, \dots, N\} \to H$. Equivalently, it is of the form $H=[\{a_{n,1}, a_{n,2}, \dots, a_{n, N_n}\}]$, where $N_n\in\mathbb{N}$, $N=[N_n]\in {}^\ast\mathbb{N}$, and the \(a_{n,i}\) are standard reals \cite{keisler1984infinitesimal}.
\end{definition}

\begin{example}[Hyperfinite Time Grid]
Let $N\in {}^\ast\mathbb{N}$ be an infinite hypernatural. Define the time step $\delta t=1/N\in {}^\ast\mathbb{R}$ (a positive infinitesimal) and the hyperfinite time grid
\[
\mathbb{T}_{\delta t}=\{0, \delta t, 2\delta t, \dots, 1\}.
\]
This is a hyperfinite internal subset of ${}^\ast\mathbb{R}$ with internal cardinality $N+1$. It can be represented as $\mathbb{T}_{\delta t} = [\{0, 1/N_n, 2/N_n, \dots, 1\}]$ for any sequence $(N_n)$ such that $[N_n] = N$.
\end{example}

The following theorem highlights the power of NSA: it permits carrying out relatively simple manipulations in the extended nonstandard world and transferring them directly to their standard counterparts.

\begin{theorem}[Full Transfer Principle \cite{ loeb1979introduction,robinson1974non}]
\label{thm:full-transfer}
If $\varphi(x_1, \dots, x_k)$ is a bounded-quantifier formula in the superstructure language of set theory, and $A_1, \dots, A_k$ are internal objects with \(A_i=[A_{i,n}]\), then $\varphi(A_1, \dots, A_k)$ is true in $V({}^\ast\mathbb{R})$ if and only if
\[
\{ n \in \mathbb{N} : \varphi(A_{1,n}, \dots, A_{k,n}) \text{ is true in } V(\mathbb{R}) \} \in \mathcal{U}.
\]
\end{theorem}

\begin{remark}
Hyperfinite sets behave in many ways like finite sets: one can sum over them, define internal functions on them, and apply the transfer principle to finite combinatorial facts. This allows us to perform discrete algebra on an infinitesimally fine grid, which is the cornerstone of our approach. It is important to note, however, that it does not automatically transfer statements involving quantification over all subsets (which is a second-order concept), hence the critical distinction between internal and external sets \cite{loeb1979introduction}.
\end{remark}

\subsection{Hyperfinite counting measure and Loeb measure}
\label{sec:loeb}

To perform probability theory on hyperfinite sample spaces, we start with the internal counting measure and convert it into a standard probability measure via the Loeb construction.

\begin{definition}[Internal Counting Measure]
Let $H$ be a hyperfinite set. The internal counting measure $\nu_{\mathrm{int}}$ is the finitely-additive measure defined on the internal algebra $\mathcal{I}(H)$ of all internal subsets of $H$ by $\nu_{\mathrm{int}}(A) = |A|$ (the internal cardinality).
\end{definition}

\begin{definition}[Internal Probability Measure]
If $H$ is a nonempty hyperfinite set, the normalized internal counting measure $\mu_{\mathrm{int}}$ is defined on $\mathcal{I}(H)$ by
\[
\mu_{\mathrm{int}}(A) = \frac{|A|}{|H|},\qquad \text{for } A \in \mathcal{I}(H).
\]
This is an internal, finitely additive probability measure.
\end{definition}

\begin{theorem}[Loeb Measure \cite{loeb1975conversion}]
\label{thm:loeb}
Let $(H, \mathcal{I}(H), \mu_{\mathrm{int}})$ be an internal hyperfinite probability space. There exists a unique standard $\sigma$-additive probability measure $\mu_L$ on the $\sigma$-algebra $\sigma(\mathcal{I}(H))$ generated by $\mathcal{I}(H)$ such that for every internal set $A \in \mathcal{I}(H)$,
\[
\mu_L(A) = \mathrm{st}\big(\mu_{\mathrm{int}}(A)\big).
\]
The measure $\mu_L$ is called the Loeb measure associated to $\mu_{\mathrm{int}}$.
\end{theorem}

\begin{proof}[Construction Sketch]
For each internal $A \in \mathcal{I}(H)$, define $\mu_L^-(A) = \mathrm{st}(\mu_{\mathrm{int}}(A))$. For arbitrary $B \subseteq H$, define the outer measure:
\[
\mu_L^\ast(B) = \inf\left\{ \sum_{n=1}^\infty \mu_L^-(A_n) : B \subseteq \bigcup_{n=1}^\infty A_n,\ A_n \in \mathcal{I}(H) \right\}.
\]
A set $B$ is Loeb measurable if for every $\epsilon > 0$, there exists $A \in \mathcal{I}(H)$ such that $\mu_L^\ast(B \triangle A) < \epsilon$. The collection of Loeb measurable sets forms a sigma-algebra $\mathcal{L}$, and $\mu_L = \mu_L^\ast|_{\mathcal{L}}$ is the Loeb measure. This description is equivalent to the standard Carathéodory construction of Loeb measure \cite{loeb1975conversion, loeb1979introduction}.
\end{proof}

Note that $(H, \mathcal{L}, \mu_L)$ is referred to as the Loeb space.

\begin{corollary}[Pushforward via Standard Part]
\label{cor:pushforward}
Let $X: H \to {}^\ast\mathbb{R}$ be an internal function on a hyperfinite probability space $(H, \mu_{\mathrm{int}})$. If $X$ is finite except on a Loeb-null set, then the pushforward of the Loeb measure $\mu_L$ by the map $\omega \mapsto \mathrm{st}(X(\omega))$ defines a standard probability measure on $\mathbb{R}$, which is the distribution of the standard part of $X$ \cite{albeverio2009nonstandard, anderson1976non}.
\end{corollary}

\subsection{Nonstandard regularity classes}

We now define differentiability and smoothness directly in the nonstandard setting.

\begin{definition}[$S$-continuity]
An internal function $f : {}^\ast\mathbb{R}^d \to {}^\ast\mathbb{R}$ is said to be $S$-continuous if
\[
x \approx y \quad \Rightarrow \quad f(x) \approx f(y),
\]
for all finite $x,y \in {}^\ast\mathbb{R}^d$ \cite{keisler1984infinitesimal,stroyan1977introduction}.
\end{definition}

\begin{definition}[$S$-differentiability]
An internal function $f : {}^\ast\mathbb{R}^d \to {}^\ast\mathbb{R}$ is $S$-differentiable if there exists an internal function
\[
Df : {}^\ast\mathbb{R}^d \to {}^\ast\mathbb{R}^d
\]
(or \(\nabla f\)) such that for all finite $x \in {}^\ast\mathbb{R}^d$ and all infinitesimal $h$,
\[
f(x+h) - f(x) = Df(x) \cdot h + \epsilon  \|h\| \quad \text{where } \epsilon \approx 0.
\]
\end{definition}

\begin{definition}[$S$-$C^k$ functions]
An internal function $f : {}^\ast\mathbb{R}^d \to {}^\ast\mathbb{R}$ is said to be $S$-$C^k$ if for all multi-indices $|\alpha|\le k$, the internal partial derivatives $\partial^\alpha f$ exist, take finite values, and are $S$-continuous for all finite $x \in {}^\ast\mathbb{R}^d$.
\end{definition}

\subsubsection{Relation to standard smoothness}

\begin{proposition}[Nonstandard characterization of $C^k$ regularity]
\label{prop:Sk-transfer-refined}
Let $f : \mathbb{R}^d \to \mathbb{R}$ be a standard function and $k \ge 1$.
\begin{enumerate}
    \item If $f \in C^k(\mathbb{R}^d)$, then for every multi-index $\alpha$ with $|\alpha|\le k$, 
    the internal derivative $\partial^\alpha({}^\ast f)$ exists and is $S$-continuous 
    at every finite point $x \in {}^\ast\mathbb{R}^d$.
    
    \item Conversely, if for every $|\alpha|\le k$ the internal derivative 
    $\partial^\alpha({}^\ast f)$ exists and is $S$-continuous at every finite point 
    $x \in {}^\ast\mathbb{R}^d$, then the standard function $f$ is in $C^k(\mathbb{R}^d)$.
\end{enumerate}
Moreover, for all multi-indices $|\alpha|\le k$ and all finite 
$x \in {}^\ast\mathbb{R}^d$, the following shadow relation holds:
\[
\partial^\alpha({}^\ast f)(x) \approx (\partial^\alpha f)(\operatorname{st}(x)).
\]
\end{proposition}

\begin{remark}
The distinction between continuity and uniform continuity is elegantly captured here: 
continuity of a standard function on $\mathbb{R}^d$ corresponds to $S$-continuity of its 
extension at all finite points, whereas uniform continuity on $\mathbb{R}^d$ 
corresponds to $S$-continuity at all points (including infinite ones) of ${}^\ast\mathbb{R}^d$ \cite{robinson1974non, stroyan1977introduction}.
\end{remark}

\subsection{Grid functions and grid calculus}
\label{sec:grid-calculus}

For later use, we introduce hyperfinite spatial grids, grid functions, and the associated discrete calculus. These objects provide an infinitesimal discretization of $\mathbb{R}^d$ on which differential operators are realized as finite difference operators. The framework is entirely internal and amenable to transfer \cite{keisler1984infinitesimal}, following the approach of treating SDEs via infinitesimal discretizations as developed in \cite{benci2008elementary}; most of these concepts may also be found in \cite{bottazzi2019grid}.

\subsubsection{Hyperfinite spatial grids}

We saw an example of a hyperfinite time grid above, but now proceed to define a slightly more elaborate (spatial) grid. Let $N' \in {}^\ast\mathbb{N}$ be an infinite hypernatural and define the spatial mesh size
\[
\delta x = \frac{1}{N'} \in {}^\ast\mathbb{R}.
\]
Fix $d \in \mathbb{N}$. The associated $d$-dimensional hyperfinite grid is
\[
\mathbb{G}_{\delta x} = (\delta x \, {}^\ast\mathbb{Z})^d
= \{ (k_1 \delta x, \dots, k_d \delta x) : k_i \in {}^\ast\mathbb{Z} \}.
\]
This is an internal subset of ${}^\ast\mathbb{R}^d$. In applications, we could restrict attention to an internal hyperfinite subset
\(
\mathbb{G} \subset \mathbb{G}_{\delta x}.
\)
We assume $N'$ and $N$ are chosen such that the ratio $(\delta x)^2 / \delta t$ is a finite (limited), noninfinitesimal hyperreal. (In our later developments, this parabolic grid scaling ensures that the discrete finite difference diffusion operator maps cleanly to its standard continuous counterpart upon taking the standard part.) 

\subsubsection{Grid functions}

\begin{definition}[Grid functions]
A grid function is an internal function
\[
u : \mathbb{G} \to {}^\ast\mathbb{R},
\]
where $\mathbb{G} \subset {}^\ast\mathbb{R}^d$ is a hyperfinite grid. The space of all grid functions on $\mathbb{G}$ is denoted by $\mathcal{G}(\mathbb{G})$.
\end{definition}

Grid functions are purely internal objects and may be viewed as infinitesimal discretizations of standard functions on $\mathbb{R}^d$.

\subsubsection{Grid difference operators}

Let $\mathbf{e}_1, \dots, \mathbf{e}_d$ denote the standard basis vectors in $\mathbb{R}^d$, and
$\mathbb{G} \subset \mathbb{G}' \subset {}^\ast\mathbb{R}^d$ be hyperfinite grids. We typically assume $\mathbb{G}'$ is large enough to contain all stencils required by the differential operators acting on functions in $\mathcal{G}(\mathbb{G}')$.

\begin{definition}[First-order grid derivatives]
For a grid function $u \in \mathcal{G}(\mathbb{G}')$, the forward and backward grid derivatives are the internal functions in $\mathcal{G}(\mathbb{G})$ defined by
\[
\Delta_i^+ u(x) = \frac{u(x + \delta x \mathbf{e}_i) - u(x)}{\delta x}, \quad
\Delta_i^- u(x) = \frac{u(x) - u(x - \delta x \mathbf{e}_i)}{\delta x}
\]
for all $x \in \mathbb{G}$.
\end{definition}

\begin{definition}[Higher-order grid derivatives]
Let $\alpha = (\alpha_1, \dots, \alpha_d) \in \mathbb{N}^d$ be a multi-index. The $\alpha$-th order forward grid derivative is defined recursively by
\[
\Delta^{\alpha} u
= (\Delta_1^+)^{\alpha_1} \cdots (\Delta_d^+)^{\alpha_d} u,
\]
provided $\mathbb{G}'$ contains the required $k$-step neighborhood of $\mathbb{G}$ for $|\alpha|=k$, where the order of the derivative is $|\alpha| = \sum_{i=1}^d \alpha_i$.
\end{definition}

All grid derivatives are internal operators on $\mathcal{G}(\mathbb{G}')$.

\subsubsection{Grid derivatives and standard derivatives}
\label{sss:grid}
\begin{proposition}[Consistency of grid derivatives] \label{prop:grid-derivative-consistency} Let $f \in C^{k+1}(\mathbb{R}^d)$ and let $u = {}^\ast f|_{\mathbb{G'}}$ be the restriction to a hyperfinite grid $\mathbb{G'}$ with spacing $\delta x$ as above. Then for every multi-index $\alpha$ with $|\alpha| \le k$ and every finite $x \in \mathbb{G}$, \[ \Delta^\alpha u(x) \approx \partial^\alpha ({}^\ast f)(x) \approx (\partial^\alpha f)(\operatorname{st}(x)). \] \end{proposition}

\begin{proof}
We proceed by induction on $|\alpha|$. For the base case $|\alpha|=1$, fix $i \in \{1,\dots,d\}$ and a finite $x \in \mathbb{G}$. By definition:
\[
\Delta_i^+ u(x) = \frac{{}^\ast f(x+\delta x \mathbf e_i)-{}^\ast f(x)}{\delta x}.
\]
Since $f \in C^{2}(\mathbb{R}^d)$, we apply the internal Taylor theorem \cite{stroyan2011foundations} (the transfer of the standard Taylor theorem). There exists some $\xi$ between $x$ and $x + \delta x \mathbf{e}_i$ such that:
\[
{}^\ast f(x+\delta x \mathbf e_i) = {}^\ast f(x) + \delta x \,  \partial_i ({}^\ast f)(x) + \frac{(\delta x)^2}{2} \partial^2_{i}({}^\ast f)(\xi).
\]
Since $x$ is finite and $\delta x$ is infinitesimal, $\xi$ is also finite. Because $\partial^2_{i} f$ is continuous, its extension $ \partial_{i}^2 ({}^\ast f)$ is $S$-continuous at finite points and thus takes only finite values. Therefore, the second-order term is $O(\delta x^2)$. Dividing by $\delta x$ gives:
\[
\Delta_i^+ u(x) = \partial_i({}^\ast f)(x) + O(\delta x).
\]
Since $O(\delta x) \approx 0$, we have $\Delta_i^+ u(x) \approx \partial_i({}^\ast f)(x)$. By the $S$-continuity of $\partial_i({}^\ast f)$ at finite points, this is also $\approx (\partial_i f)(\operatorname{st}(x))$.

The higher-order case $|\alpha| \le k$ follows by induction, noting that the $C^{k+1}$ assumption ensures that the remainder term after $k$ applications of the difference operator remains infinitesimal.
\end{proof}
Obviously, if we are content with stating the remainder to be \(o(1)\), the smoothness condition can be relaxed to \(f\in C^k\).

\subsubsection{Grid integrals}

\begin{definition}[Grid integral]
Let $\mathbb{G} \subset {}^\ast\mathbb{R}^d$ be a hyperfinite grid and let $u : \mathbb{G} \to {}^\ast\mathbb{R}$ be a grid function. The grid integral of $u$ over $\mathbb{G}$ is defined by
\[
\int_{\mathbb{G}} u(x) \, dx
:= \sum_{x \in \mathbb{G}} u(x) (\delta x)^d,
\]
where the sum is the internal hyperfinite sum.
\end{definition}

If the grid integral is finite, its standard part coincides with the Lebesgue integral of the associated standard function, when such a function exists (otherwise, the grid integral represents the Loeb integral) \cite{anderson1976non,loeb1975conversion}.

\begin{remark}[Transfer of Calculus Identities]
Any first-order discrete identity valid for all finite $N$ (e.g., discrete product rule, summation-by-parts) transfers to the hyperfinite grid. This makes many algebraic manipulations on the grid immediate and rigorous, as illustrated below.
\end{remark}

\subsubsection{Some grid calculus identities}

To ensure that some of the following identities hold without boundary terms, we define the notion of support on a hyperfinite grid. These are, in particular, of use to us in rigorously deriving the Fokker-Planck equation, and we state them in a convenient form towards realizing this goal.

\begin{definition}[Grid-Compact Support]
An internal grid function $u \in \mathcal{G}(\mathbb{G})$ has grid-compact support if there exists a finite $R \in \mathbb{R}^+$ such that $u(x) = 0$ for all $x \in \mathbb{G}$ with $\|x\| \ge R$. 
\end{definition}

\begin{proposition}[Discrete Product Rule]
\label{prop:pr}
Let $u, v \in \mathcal{G}(\mathbb{G'})$ be any internal grid functions. Then for each $i \in \{1, \dots, d\}$ and \(x\in \mathbb{G}\):
\[
u(x) \Delta_i^+ v(x) = \Delta_i^+ (u(x)v(x)) - v(x + \delta x \mathbf{e}_i) \Delta_i^+ u(x).
\]

\end{proposition}

\begin{proof}
This is a purely algebraic identity. Expanding the right-hand side using the definition of $\Delta_i^+$:\\
\(\Delta_i^+ (u(x)v(x)) - v(x + \delta x \mathbf{e}_i) \Delta_i^+ u(x)\)
\begin{align*}
 &= \frac{u(x+\delta x \mathbf{e}_i)v(x+\delta x \mathbf{e}_i) - u(x)v(x)}{\delta x} - v(x+\delta x \mathbf{e}_i)\frac{u(x+\delta x \mathbf{e}_i) - u(x)}{\delta x} \\
&= \frac{u(x+\delta x \mathbf{e}_i)v(x+\delta x \mathbf{e}_i) - u(x)v(x) - v(x+\delta x \mathbf{e}_i)u(x+\delta x \mathbf{e}_i) + v(x+\delta x \mathbf{e}_i)u(x)}{\delta x} \\
&= \frac{u(x)v(x + \delta x \mathbf{e}_i) - u(x)v(x)}{\delta x} = u(x) \Delta_i^+ v(x).
\end{align*}
\end{proof}

\begin{proposition}[Forward-Backward Adjoint Identity]
\label{prop:fb-adjoint}
Let $u, v \in \mathcal{G}(\mathbb{G})$. If $v$ has grid-compact support, then:
\[
\sum_{x \in \mathbb{G}} u(x) \Delta_i^+ v(x) \delta x^d = - \sum_{x \in \mathbb{G}} v(x) \Delta_i^- u(x) \delta x^d.
\]
\end{proposition}

\begin{proof}
Proposition \ref{prop:pr} yields:
\[
\sum_{x \in \mathbb{G}} u(x) \Delta_i^+ v(x) \delta x^d = \sum_{x \in \mathbb{G}} \Delta_i^+ (u(x)v(x)) \delta x^d - \sum_{x \in \mathbb{G}} v(x + \delta x \mathbf{e}_i) \Delta_i^+ u(x) \delta x^d.
\]
Now, the hyperfinite grid $\mathbb{G}$ is defined for indices up to some infinite $M \in {}^\ast\mathbb{N}$ such that $M\delta x$ is infinite. The first sum on the right-hand side is a telescoping sum of an internal function. Specifically, the internal sum in the $i$-th direction, with \(w=uv\):
\[
\sum_{k=-M}^{M} \frac{w((k+1)\delta x \mathbf{e}_i) - w(k\delta x \mathbf{e}_i)}{\delta x} \delta x = w((M+1)\delta x \mathbf{e}_i) - w(-M\delta x \mathbf{e}_i).
\]
Since $v$ has grid-compact support $R$, there exists a finite $R > 0$ such that $v(x) = 0$ for all $\|x\| \ge R$. Because $M\delta x$ is infinite, we have $\|(M+1)\delta x \mathbf{e}_i\| > R$ and $\|(-M)\delta x \mathbf{e}_i\| > R$. 
Thus:
$$ w((M+1)\delta x \mathbf{e}_i) = 0 \quad \text{and} \quad w(-M\delta x \mathbf{e}_i) = 0,$$
which ensures the telescoping sum vanishes.

For the second sum, let $y = x + \delta x \mathbf{e}_i$. As $x$ ranges over $\mathbb{G}$, $y$ ranges over the shifted grid $\mathbb{G} + \delta x \mathbf{e}_i$. However, since $v(y) = 0$ for all $y \in (\mathbb{G} \triangle (\mathbb{G} + \delta x \mathbf{e}_i))$, we can maintain the summation over the original domain $\mathbb{G}$:
\begin{align*}
\sum_{x \in \mathbb{G}} v(x + \delta x \mathbf{e}_i) \frac{u(x + \delta x \mathbf{e}_i) - u(x)}{\delta x} \delta x^d &= \sum_{y \in \mathbb{G}} v(y) \frac{u(y) - u(y - \delta x \mathbf{e}_i)}{\delta x} \delta x^d \\
&= \sum_{y \in \mathbb{G}} v(y) \Delta_i^- u(y) \delta x^d.
\end{align*}
\end{proof}

\begin{proposition}[First-Order Adjoint]
\label{prop:rfoa}
Let $p, b \in \mathcal{G}(\mathbb{G})$ and let $\varphi \in \mathcal{G}(\mathbb{G})$ have grid-compact support. Then:
\[
\sum_{x \in \mathbb{G}} p(x) b(x) \Delta_i^+ \varphi(x) \delta x^d = -\sum_{x \in \mathbb{G}} \varphi(x) \Delta_i^- (b(x) p(x)) \delta x^d.
\]
\end{proposition}

\begin{proof}
Define $u(x) = p(x) b(x)$. Since $u \in \mathcal{G}(\mathbb{G})$ and $\varphi$ has grid-compact support, the result follows immediately from Proposition \ref{prop:fb-adjoint}.
\end{proof}

\begin{proposition}[Second-Order Adjoint]
Let $p, a \in \mathcal{G}(\mathbb{G})$ and let $\varphi \in \mathcal{G}(\mathbb{G})$ have grid-compact support. Then:
\[
\sum_{x \in \mathbb{G}} p(x) a(x) \Delta_i^+ \Delta_j^+ \varphi(x) \delta x^d = \sum_{x \in \mathbb{G}} \varphi(x) \Delta_j^- \Delta_i^- (a(x) p(x)) \delta x^d.\]
\label{prop:rsoa}
\end{proposition}

\begin{proof}
Let $w(x) = \Delta_j^+ \varphi(x)$. Since $\varphi$ has grid-compact support, its grid derivatives also have grid-compact support. Applying Proposition \ref{prop:fb-adjoint} to $u = p a$ and $v = w$:
\[
\sum_{x \in \mathbb{G}} (p a) \Delta_i^+ (\Delta_j^+ \varphi) \delta x^d = - \sum_{x \in \mathbb{G}} (\Delta_j^+ \varphi) \Delta_i^- (p a) \delta x^d.
\]
Applying Proposition \ref{prop:fb-adjoint} a second time with $u = \Delta_i^- (p a)$ and $v = \varphi$, the display above is equal to:
\[
- \left( - \sum_{x \in \mathbb{G}} \varphi(x) \Delta_j^- [ \Delta_i^- (p a) ] \delta x^d \right) = \sum_{x \in \mathbb{G}} \varphi(x) \Delta_j^- \Delta_i^- (a p) \delta x^d.
\]
\end{proof}

Starting with the next subsection, we look at nonstandard stochastic calculus.

\subsection{Hyperfinite stochastic increments and hyperfinite Brownian motion}
\label{sec:hyper-BM}

To model Brownian motion on the grid, we construct an internal array of random increments. This construction mirrors the standard Donsker invariance principle in the internal setting and also forms the basis for the approach in \cite{nelson1987radically}.

\textbf{Hyperfinite Probability Space.}
Let $\mathbb{T}_{\delta t} = \{0, \delta t, 2\delta t, \dots, 1\}$ be the hyperfinite time grid. Consider the internal sample space $\Omega = \{-1, +1\}^{\mathbb{T}_{\delta t} \setminus \{1\}}$, representing the set of all possible paths of binary increments. We endow $\Omega$ with the uniform internal counting measure $\mu_{\text{int}}(A) = \frac{|A|}{|\Omega|}$ for all internal $A \subseteq \Omega$. The Loeb space $(\Omega, \mathcal{L}, \mu_L)$ is the completion of the measure space $(\Omega, \sigma(\mathcal{A}), \operatorname{st} \circ \mu_{\text{int}})$, where $\mathcal{A}$ is the internal algebra of all internal subsets of $\Omega$.

\begin{definition}[Internal Independence]
A family $\{X_\tau\}_{\tau \in \mathbb{T}_{\delta t}}$ of internal random variables is internally independent if for any internal subset of indices $\{\tau_1, \dots, \tau_n\}$ and any internal Borel sets $B_1, \dots, B_n$:
\[
\mathbb{P}_{\text{int}}\left( \bigcap_{i=1}^n \{X_{\tau_i} \in B_i\} \right) = \prod_{i=1}^n \mathbb{P}_{\text{int}}(X_{\tau_i} \in B_i).
\]
\end{definition}

\begin{definition}[Hyperfinite Brownian Motion \cite{anderson1976non}]
Let $\xi(t, \omega)$ be an internal Rademacher array such that for each $t$, $\xi(t, \cdot)$ is internally independent and $\mathbb{P}_{\text{int}}(\xi = \pm 1) = 1/2$. We define the Brownian increments as $\Delta W_t(\omega) := \xi(t, \omega)\sqrt{\delta t}$. The hyperfinite Brownian motion is the internal process defined by the summation:
\[
W_t(\omega) := \sum_{s \in \mathbb{T}_{\delta t} \cap [0, t)} \Delta W_s(\omega).
\]
\label{def:hb}
\end{definition}

It is important to clarify the relationship between the specific noise model employed in the current development and the broader theoretical framework of this paper. For pedagogical clarity and to illustrate the varied hyperfinite constructions available, the introductory sections utilize a specific Rademacher noise increment to define the discrete SDE. This allows for a more intuitive introduction to the `Calculus-first' perspective.

However, the majority of our results are established under the more general framework of Definition \ref{def:gd}. In this broader setting, we characterize the dynamics through a general hyperfinite walk, which endogenously specifies the moments of the driving noise required to recover hyperfinite Brownian motion. We emphasize that all derivations presented in these sections remain valid under the generalized definition, as the latter encompasses the Rademacher case while remaining agnostic to the specific incremental distribution.

\begin{definition}[$S$-integrability \cite{albeverio2009nonstandard,anderson1976non}]
An internal random variable $X: \Omega \to {}^*\mathbb{R}$ is said to be $S$-integrable if:
\begin{enumerate}
    \item $\mathbb{E}_{\mathrm{int}}[|X|]$ is a limited hyperreal.
    \item For every null internal set $A \subseteq \Omega$ (i.e., $\mu_{\mathrm{int}}(A) \approx 0$), the internal expectation over $A$ is infinitesimal: $\mathbb{E}_{\mathrm{int}}[|X|\mathbbm{1}_A] \approx 0$.
\end{enumerate}
\end{definition}

\begin{proposition}[Moment Identities and $S$-integrability]
\label{prop:hyper-moments-rigorous}
For each $t \in \mathbb{T}_{\delta t}$, the internal expectations satisfy $\mathbb{E}_{\mathrm{int}}[W_t] = 0$ and $\mathbb{E}_{\mathrm{int}}[W_t^2] = t$. Furthermore,  the process $W_t^2$ is $S$-integrable, ensuring that $\mathbb{E}_L[\operatorname{st}(W_t^2)] = \operatorname{st}(\mathbb{E}_{\mathrm{int}}[W_t^2]) = \operatorname{st}(t)$.
\end{proposition}

\begin{proof}
By construction, any $t \in \mathbb{T}_{\delta t}$ is of the form $k \delta t$ for some $k \in \{0, 1, \dots, N\} \subseteq {}^*\mathbb{N}$. The identities $\mathbb{E}_{\mathrm{int}}[W_t] = 0$ and $\mathbb{E}_{\mathrm{int}}[W_t^2] = t$ follow directly from the internal independence of the increments $\Delta W_s$ and the fact that $\mathbb{E}_{\mathrm{int}}[(\Delta W_s)^2] = \delta t$, which yields $\mathbb{E}_{\mathrm{int}}[(\sum_{i=0}^{k-1} \Delta W_{i \delta t})^2] = \sum_{i=0}^{k-1} \delta t = k \delta t = t$.

To establish $S$-integrability, we demonstrate that $\mathbb{E}_{\mathrm{int}}[W_t^4]$ is limited for all limited $t$. Using the internal multinomial expansion for the hyperfinite sum:
\[
\mathbb{E}_{\mathrm{int}}[W_t^4] = \mathbb{E}_{\mathrm{int}}\left[ \left( \sum_{i=0}^{k-1} \Delta W_{i \delta t} \right)^4 \right] = \sum_{i,j,l,m=0}^{k-1} \mathbb{E}_{\mathrm{int}}[\Delta W_{i \delta t} \Delta W_{j \delta t} \Delta W_{l \delta t} \Delta W_{m \delta t}].
\]
By internal independence and the vanishing of the first moments, the only non-vanishing terms in the internal sum are those where indices coincide in pairs. Specifically, we have $k$ terms where $i=j=l=m$, and $3k(k-1)$ terms where the indices form two distinct pairs (e.g., $i=j \neq l=m$, including permutations):
\begin{align*}
\mathbb{E}_{\mathrm{int}}[W_t^4] &= \sum_{i=0}^{k-1} \mathbb{E}_{\mathrm{int}}[\Delta W_{i \delta t}^4] + 3 \sum_{0 \le i < j \le k-1} \mathbb{E}_{\mathrm{int}}[\Delta W_{i \delta t}^2] \mathbb{E}_{\mathrm{int}}[\Delta W_{j \delta t}^2] \\
&= k(\delta t)^2 + 3k(k-1)(\delta t)^2 \\
&= t \delta t + 3t(t-\delta t) = 3t^2 - 2t\delta t.
\end{align*}
For any limited $t$, $\mathbb{E}_{\mathrm{int}}[W_t^4] \approx 3t^2$, which is a limited hyperreal. Since the $p$-th internal moment is limited for $p=2$, the random variable $W_t^2$ satisfies the de la Vall\'ee Poussin theorem transferred to the Loeb setting \cite{albeverio2009nonstandard} implying $S$-integrability. Consequently, the standard part and internal expectation commute, recovering the standard second moment on the Loeb space.
\end{proof}

\begin{proposition}[Loeb Process is a Wiener Process \cite{anderson1976non}]
\label{prop:loeb-wiener}
Let $W_t$ be the hyperfinite Brownian motion constructed on $(\Omega, \mathcal{L}, \mu_L)$. For standard $t \in [0,1]$, define the standard-part process $\bar{W}_t(\omega) \coloneqq \operatorname{st}(W_\tau(\omega))$ by choosing any $\tau \in \mathbb{T}_{\delta t}$ with $\tau \approx t$. The process $\bar{W}_t$ is a standard Brownian motion; it is well-defined $\mu_L$-almost surely, has continuous paths, and independent Gaussian increments.
\end{proposition}

\begin{proof}[Detailed Proof Sketch]
\textbf{1. Finite-Dimensional Distributions:} Fix standard times $0 = t_0 < t_1 < \cdots < t_k \le 1$ and choose $\tau_j \in \mathbb{T}_{\delta t}$ such that $\tau_j \approx t_j$. Let $V$ be the vector of increments $V_j = W_{\tau_j} - W_{\tau_{j-1}}$. The internal characteristic function $\phi_V(u)$ for $u \in \mathbb{R}^k$ is given by:
\[
\phi_V(u) = \mathbb{E}_{\text{int}}\left[ \exp\left(i \sum_{j=1}^k u_j V_j\right) \right] = \prod_{j=1}^k \prod_{s \in \mathbb{T}_{\delta t} \cap [\tau_{j-1}, \tau_j)} \cos(u_j \sqrt{\delta t}).
\]
Using the internal Taylor expansion $\cos(z) = 1 - z^2/2 + O(z^4)$ and the fact that $\log(1-\epsilon) = -\epsilon + O(\epsilon^2)$ for infinitesimal $\epsilon$, we obtain:
\[
\operatorname{st}(\log \phi_V(u)) = \operatorname{st}\left( \sum_{j=1}^k \frac{\tau_j - \tau_{j-1}}{\delta t} \left( -\frac{u_j^2}{2} \delta t + O(\delta t^2) \right) \right) = -\frac{1}{2} \sum_{j=1}^k u_j^2 (t_j - t_{j-1}).
\]
Since the standard part of the internal characteristic function is the characteristic function of the Loeb distribution, the increments are independent Gaussians with variances $t_j - t_{j-1}$.

\textbf{2. $S$-continuity and Path Continuity:} We demonstrate that $W_t$ is $\mu_L$-almost surely $S$-continuous. For any $s, t \in \mathbb{T}_{\delta t}$ with $s < t$, the fourth-moment identity from Proposition \ref{prop:hyper-moments-rigorous} yields:
\[
\mathbb{E}_{\mathrm{int}}[|W_t - W_s|^4] = 3(t-s)^2 - 2(t-s)\delta t < 3(t-s)^2.
\]
By the transferred Kolmogorov-Chentsov theorem, for $\mu_L$-almost all $\omega$, the map $t \mapsto W_t(\omega)$ is $S$-continuous on $\mathbb{T}_{\delta t}$. This ensures that $t \approx s \implies W_t \approx W_s$, rendering $\bar{W}_t$ independent of the choice of $\tau \approx t$ and continuous in the standard topology.

\textbf{3. Covariance:} The covariance $\operatorname{Cov}(\bar{W}_{t_i}, \bar{W}_{t_j}) = \min(t_i, t_j)$ follows directly from the independence and variance of the increments established in Step 1.
\end{proof}

\section{Proof of $S$-continuity of grid solutions}
\label{app:s-cts}

To show that $X_t$ is $\mu_L$-a.s. $S$-continuous, we establish a fourth-moment bound on the increments over any internal interval $[s, t) \subseteq \mathbb{T}_{\delta t}$.

\subsection{Fourth Moment Bound}
Let $s, t \in \mathbb{T}_{\delta t}$ with $s < t$. The internal increment $X_t - X_s$ decomposes into a drift term $D$ and an internal martingale term $M$ \cite{keisler1984infinitesimal}:
\[
X_t - X_s = \underbrace{\sum_{r \in [s,t)} {}^*b(r, X_r) \delta t}_{D} + \underbrace{\sum_{r \in [s,t)} {}^*\sigma(r, X_r) \Delta W_r}_{M}.
\]
Using the elementary inequality $\|a+b\|^4 \le 8(\|a\|^4 + \|b\|^4)$, we bound the moments of $D$ and $M$ separately.

\textbf{1. The Drift Term ($D$).}
By Jensen's inequality for internal sums, we have:
\[
\|D\|^4 = \left\| \sum_{r \in [s,t)} {}^*b(r, X_r) \delta t \right\|^4 \le (t-s)^3 \sum_{r \in [s,t)} \|{}^*b(r, X_r)\|^4 \delta t.
\]
Taking internal expectations and applying the linear growth condition from Assumption \ref{assump:coeff}:
\[
\mathbb{E}_{\mathrm{int}}[\|D\|^4] \le (t-s)^3 \sum_{r \in [s,t)} K(1 + \mathbb{E}_{\mathrm{int}}[\|X_r\|^4]) \delta t \le K_1(t-s)^4,
\]
where we utilize the fact that $\mathbb{E}_{\mathrm{int}}[\|X_r\|^4]$ is limited for all $r \in \mathbb{T}_{\delta t}$ via the discrete internal Gr\"onwall inequality.

\textbf{2. The Diffusion Term ($M$).}
Since $M$ is an internal martingale, we apply the internal Burkholder-Davis-Gundy (BDG) inequality \cite{lindstrom1980hyperfinite1} for $p=4$:
\[
\mathbb{E}_{\mathrm{int}}[\|M\|^4] \le K \cdot \mathbb{E}_{\mathrm{int}} \left[ \left( \sum_{r \in [s,t)} \|{}^*\sigma(r, X_r)\|^2 \delta t \right)^2 \right].
\]
Applying Jensen's inequality to the internal sum within the expectation:
\[
\mathbb{E}_{\mathrm{int}}[\|M\|^4] \le K(t-s) \sum_{r \in [s,t)} \mathbb{E}_{\mathrm{int}}[\|{}^*\sigma(r, X_r)\|^4] \delta t \le K_2(t-s)^2.
\]

\textbf{3. Combining Terms.}
Given $t-s \le 1$, the bound $(t-s)^4 \le (t-s)^2$ holds. Thus, there exists a limited constant $C'$ such that:
\[
\mathbb{E}_{\mathrm{int}}[\|X_t - X_s\|^4] \le C'(t-s)^2.
\]
This internal inequality is the nonstandard analogue of the Kolmogorov-Chentsov condition.

\subsection{Loeb Measure Argument}
For standard $n, k \in \mathbb{N}$, define the internal event $B_{n,k}$ as the set of $\omega \in \Omega$ for which there exist $s, t \in \mathbb{T}_{\delta t}$ with $|t-s| < 1/n$ and $\|X_t(\omega) - X_s(\omega)\| > 1/k$. By the transferred Kolmogorov inequality \cite{keisler1984infinitesimal}, we have:
\[
\mathbb{P}_{\mathrm{int}}(B_{n,k}) \le \frac{C' (1/n)^2}{(1/k)^4} \cdot n = \frac{C'k^4}{n}.
\]
The set of non-$S$-continuous paths is given by $A = \bigcup_{k \in \mathbb{N}} \bigcap_{n \in \mathbb{N}} B_{n,k}$. For each standard $k$, $\mu_L(\bigcap_{n \in \mathbb{N}} B_{n,k}) \le \inf_{n \in \mathbb{N}} \operatorname{st}(C'k^4/n) = 0$. By the countable subadditivity of the Loeb measure $\mu_L$, it follows that $\mu_L(A) = 0$. Thus, $X_t$ is $\mu_L$-almost surely $S$-continuous.

\section{Convergence to the standard It\^o solution}
\label{app:sis}
The proof identifies the standard part of the hyperfinite recursion as the unique strong solution to the It\^o SDE by establishing $S$-continuity and matching the limits of internal quadratic forms.

\textbf{Step 1: $S$-continuity and Finiteness.}
By the internal Gr\"onwall inequality and the fourth-moment bound established in Appendix~\ref{app:s-cts}, $X_t$ is $S$-bounded and $S$-continuous $\mu_L$-a.s. \cite{anderson1976non, keisler1984infinitesimal}. Thus, for almost all $\omega$, the standard part $\bar{X}_t = \operatorname{st}(X_{\tau})$ with continuous paths exists and is well-defined for all $t \in [0,1]$. Since $b$ and $\sigma$ are continuous in \(t\) and Lipschitz in \(x\), their extensions ${}^*b$ and ${}^*\sigma$ are $S$-continuous. The composition of $S$-continuous functions is $S$-continuous; thus, the internal processes $B_s \coloneqq {}^*b(s, X_s)$ and $\Sigma_s \coloneqq {}^*\sigma(s, X_s)$ are $S$-continuous and $S$-bounded on $\mathbb{T}_{\delta t}$ $\mu_L$-a.s.

\textbf{Step 2: Convergence of the Drift.}
The drift term $D_{\tau} = \sum_{s < \tau} B_s \delta t$ is an internal Riemann sum. For $S$-continuous and $S$-bounded integrands, the standard part of the sum is the Riemann integral of the standard part \cite{anderson1976non}:
\[
\operatorname{st}(D_{\tau}) = \int_0^{\operatorname{st}(\tau)} \operatorname{st}(B_s) \, ds = \int_0^t b(r, \bar{X}_r) \, dr.
\]

\textbf{Step 3: Convergence of the Diffusion.}
Let $M_{\tau} = \sum_{s < \tau} \Sigma_s \Delta W_s$. We identify its standard part $\bar{M}_t \coloneqq \operatorname{st}(M_{\tau})$ via its martingale properties:
\begin{enumerate}
    \item \textbf{Martingale Property:} $M_{\tau}$ is an $S$-continuous internal martingale. Its standard part $\bar{M}_t$ is therefore a continuous standard martingale on the Loeb space \cite{lindstrom1980hyperfinite3}.
    \item \textbf{Quadratic Variation:} The internal quadratic variation \\$[M]_{\tau} = \sum_{s < \tau} \Sigma_s \Sigma_s^\top \delta t$ is an internal Riemann sum. Its standard part is:
    \[
    \operatorname{st}([M]_{\tau}) = \int_0^t \sigma(r, \bar{X}_r)\sigma(r, \bar{X}_r)^\top \, dr.
    \]
    \item \textbf{Identification:} By the internal covariation $[M, W]_{\tau} = \sum_{s < \tau} \Sigma_s \delta t$, we have $\operatorname{st}([M, W]_{\tau}) = \int_0^t \sigma(r, \bar{X}_r) \, dr$. By the Lévy's characterization/martingale representation theorem, $\bar{M}_t$ is uniquely identified as the It\^o integral $\int_0^t \sigma(r, \bar{X}_r) \, d\bar{W}_r$ \cite{hoover1983nonstandard1}.
\end{enumerate}

\textbf{Step 4: Conclusion.}
Taking the standard part of the internal identity $X_{\tau} = X_0 + D_{\tau} + M_{\tau}$ for $\tau \approx t$ yields the standard integral equation:
\[
\bar{X}_t = \bar{X}_0 + \int_0^t b(r, \bar{X}_r) \, dr + \int_0^t \sigma(r, \bar{X}_r) \, d\bar{W}_r.
\]
By the global Lipschitz condition, this SDE admits a unique strong solution, which must coincide with $\bar{X}_t$ $\mu_L$-almost surely.

\section{Discrete It\^o calculus on a hyperfinite grid}
\label{sec:discrete-ito}

This section provides rigorous proofs for the discrete It\^o formula within the framework of hyperfinite grid diffusions and its consequences in the standard world.

\subsection{Discrete It\^o formula with pathwise control}
We utilize the fact that for $\mu_L$-almost all paths, the process remains within the finite galaxy of ${}^*\mathbb{R}^d$ \cite{anderson1976non}, allowing us to control the Taylor remainder pathwise (see also \cite{stroyan2011foundations}).

\begin{proposition}[Discrete It\^o Formula]
\label{prop:discrete-ito-rigorous}
Let $f \in C^3(\mathbb{R}^d; \mathbb{R})$ and let $(X_t)_{t \in \mathbb{T}_{\delta t}}$ satisfy the grid recursion \eqref{eq:grid-recursion} with Rademacher noise increments. Suppose the coefficients $b, \sigma$ satisfy the conditions of Theorem~\ref{thm:grid-to-ito}. Then for $\mu_L$-almost every $\omega \in \Omega$:
\begin{equation}
{}^\ast f(X_t) - {}^\ast f(X_0) = \sum_{s < t} \left[ \nabla ({}^\ast f)(X_s) \cdot \Delta X_s + \frac{1}{2} \Delta X_s^\top D^2 ({}^\ast f)(X_s) \Delta X_s \right] + \epsilon_t(\omega),
\end{equation}
where $\Delta X_s = X_{s+\delta t} - X_s$ and the error satisfies $|\epsilon_t(\omega)| \le K(\omega)\sqrt{\delta t} \approx 0$ for all $t \in \mathbb{T}_{\delta t}$.
\end{proposition}

\begin{proof}
By the linear growth of coefficients and the path properties established in Theorem \ref{thm:grid-to-ito}, there exists a set $\Omega_0$ with $\mu_L(\Omega_0) = 1$ such that for every $\omega \in \Omega_0$, the path is $S$-bounded: $\max_{s \in \mathbb{T}_{\delta t}} \|X_s(\omega)\|$ is limited. Consequently, the internal coefficients along the path satisfy $\max_{s \in \mathbb{T}_{\delta t}} ( \|{}^*b(s, X_s)\| + \|{}^*\sigma(s, X_s)\| ) \le C'(\omega)$ for some limited $C'(\omega) \in {}^*\mathbb{R}$.

Fix $\omega \in \Omega_0$. Applying the internal Taylor theorem to ${}^*f$:
\[
{}^\ast f(X_{s+\delta t}) = {}^\ast f(X_s) + \nabla ({}^\ast f)(X_s)^\top \Delta X_s + \frac{1}{2} \Delta X_s^\top D^2 ({}^\ast f)(X_s) \Delta X_s + R_s,
\]
where the remainder term in Lagrange form is:
\[
R_s = \frac{1}{6} \sum_{|\alpha|=3} \partial^\alpha ({}^* f)(\xi_s) (\Delta X_s)^\alpha,
\]
for some $\xi_s$ on the internal line segment $[X_s, X_{s+\delta t}]$.

\textbf{1. Local Finite Bound:}
As argued above, the set of reachable states is contained in an internal ball $B(0, H(\omega))$ of limited radius $H(\omega)$. Because $f$ is a standard $C^3$ function, its third-order internal derivatives are $S$-continuous at all finite points and thus $S$-bounded on $B(0, H(\omega))$. Hence, there exists a limited $M(\omega)$ such that:
\[
\max_{|\alpha|=3} \sup_{s \in \mathbb{T}_{\delta t}} |\partial^\alpha ({}^*f)(\xi_s)| \le M(\omega).
\]

\textbf{2. Remainder Estimate:}
From the recursion \eqref{eq:grid-recursion}, $\|\Delta X_s\| \le \|{}^*b\| \delta t + \|{}^*\sigma\| \sqrt{m \delta t}$. Given the boundedness of the coefficients, there exists a limited $L(\omega)$ such that $\|\Delta X_s\| \le L(\omega) \sqrt{\delta t}$. Substituting this into the remainder:
\[
|R_s| \le \frac{d^3}{6} M(\omega) L(\omega)^3 (\delta t)^{3/2} = K(\omega) (\delta t)^{3/2},
\]
where $K(\omega)$ is a limited hyperreal.

\textbf{3. Summation of Errors:}
The total error $\epsilon_t(\omega)$ is the internal sum of these remainders over $t/\delta t$ steps:
\[
|\epsilon_t(\omega)| = \left| \sum_{s < t} R_s \right| \le \sum_{s < t} K(\omega) (\delta t)^{3/2} \le \frac{1}{\delta t} K(\omega) (\delta t)^{3/2} = K(\omega) \sqrt{\delta t}.
\]
Since $K(\omega)$ is limited and $\sqrt{\delta t} \approx 0$, it follows that $\epsilon_t(\omega)$ is infinitesimal uniformly for all $t \in \mathbb{T}_{\delta t}$.
\end{proof}

\subsection{Itô formula in the limit}
Of course, we have the following \cite{stroyan2011foundations}:
\begin{corollary}[Itô Formula in the Limit]
\label{cor:ito-limit}
Let $\bar{X}_t$ be the standard part of the grid diffusion $X_t$ satisfying the assumptions of Theorem~\ref{thm:grid-to-ito}. Then for any $f \in C^3(\mathbb{R}^d; \mathbb{R})$, the process $\bar{X}_t$ satisfies the classical Itô formula:
\[
\begin{aligned}
f(\bar{X}_t) &= f(\bar{X}_0) + \int_0^t \nabla f(\bar{X}_s)^\top b(s, \bar{X}_s)\, d s + \int_0^t \nabla f(\bar{X}_s)^\top \sigma(s, \bar{X}_s)\, d \bar{W}_s \\
&\quad + \frac{1}{2} \int_0^t \tr\left(\sigma(s, \bar{X}_s) \sigma(s, \bar{X}_s)^\top D^2 f(\bar{X}_s)\right) d s.
\end{aligned}
\]
\end{corollary}
\begin{proof}
For brevity, we denote both the standard function and its natural extension by $f$.

\textbf{1. Pathwise Regularity and Lifting.}
There exists a Loeb-measurable set $\Omega_0$ with $\mu_L(\Omega_0)=1$ such that for all $\omega \in \Omega_0$:
\begin{enumerate}
    \item The process $X_t(\omega)$ is $S$-bounded ($|X_t| < H(\omega)$) and $S$-continuous on $\mathbb{T}_{\delta t}$.
    \item The standard-part process $\bar{X}_r = \text{st}(X_{\tau})$ (for $\tau \approx r$) is well-defined and continuous.
    \item The functions $b, \sigma, \nabla f,$ and $D^2 f$ evaluated along the path are $S$-continuous and $S$-bounded.
\end{enumerate}
Consequently, the internal integrands are \emph{liftings} of their standard counterparts. For instance, $\nabla f(X_s) \approx \nabla f(\bar{X}_{\text{st}(s)})$ holds for all $s \in \mathbb{T}_{\delta t}$ \cite{loeb1975conversion}.

\textbf{2. The Drift and Martingale Components.}
Substituting the increment $\Delta X_s = {}^*b \delta t + {}^*\sigma \Delta W_s$ into the first-order sum of the Taylor expansion:
\[
\sum_{s < t} \nabla f(X_s)^\top \Delta X_s = \underbrace{\sum_{s < t} \nabla f(X_s)^\top b(s, X_s) \delta t}_{I_1} + \underbrace{\sum_{s < t} \nabla f(X_s)^\top \sigma(s, X_s) \Delta W_s}_{I_2}.
\]
\textbf{Drift Term $I_1$:} As an internal Riemann sum with an $S$-continuous integrand, its standard part satisfies:
\[
\operatorname{st}(I_1) = \int_0^{\operatorname{st}(t)} \operatorname{st}(\nabla f(X_r)^\top b(r, X_r)) dr = \int_0^{\operatorname{st}(t)} \nabla f(\bar{X}_r)^\top b(r, \bar{X}_r) dr.
\]
\textbf{Martingale Term $I_2$:} Let $M_\tau = \sum_{s < \tau} \Phi_s \Delta W_s$ with $\Phi_s = \nabla f(X_s)^\top \sigma(s, X_s)$. $M_\tau$ is an $S$-continuous internal martingale. Its standard part $\bar{M}_t$ is a continuous martingale on the Loeb space with quadratic variation $\langle \bar{M} \rangle_t = \operatorname{st}([M]_t) = \int_0^t \|\Phi_r\|^2 dr$. This, paired with the internal covariation $[M, W]_t$, uniquely identifies $\operatorname{st}(I_2)$ as the It\^o integral $\int_0^{\operatorname{st}(t)} (\nabla f^\top \sigma) d\bar{W}_r$ \cite{anderson1976non}.

\textbf{3. The It\^o Term and Quadratic Collapse.}
We expand the second-order sum $\frac{1}{2} \sum_{s<t} \Delta X_s^\top D^2 f(X_s) \Delta X_s$. Note that:
\[
\Delta X_s \Delta X_s^\top = (\sigma \Delta W_s)(\sigma \Delta W_s)^\top + O_\omega(\delta t^{3/2}).
\]
The $O_\omega(\delta t^{3/2})$ terms vanish upon summation. For the principal term, let $A_s = \sigma(s, X_s)^\top D^2 f(X_s) \sigma(s, X_s)$. We evaluate the internal sum:
\[
J = \sum_{s < t} \sum_{i,j=1}^m (A_s)_{ij} \Delta W_s^{(i)} \Delta W_s^{(j)}.
\]
For Rademacher increments where $(\Delta W_s^{(i)})^2 = \delta t$:
\begin{itemize}
    \item \textbf{Diagonal terms ($i=j$):} The sum is $\sum_{s < t} \text{Tr}(A_s) \delta t$. By the cyclic property, $\text{Tr}(A_s) = \text{Tr}(\sigma \sigma^\top D^2 f)$.
    \item \textbf{Off-diagonal terms ($i \neq j$):} The terms $K_\tau = \sum_{s < \tau} \sum_{i \neq j} (A_s)_{ij} \Delta W_s^{(i)} \Delta W_s^{(j)}$ form an internal martingale. Its quadratic variation $[K]_t$ is $O(\sum \delta t^2) = O(\delta t)$. Thus, by the internal martingale property, $K_t \approx 0$ Loeb-a.s. \cite{nelson1987radically}.
\end{itemize}
Thus, $\operatorname{st}(J) = \int_0^{\operatorname{st}(t)} \text{Tr}(\sigma(r, \bar{X}_r) \sigma(r, \bar{X}_r)^\top D^2 f(\bar{X}_r)) dr$.

\textbf{4. Conclusion.}
Summing the components and noting the remainder $\epsilon_t \approx 0$ from Proposition \ref{prop:discrete-ito-rigorous}, the standard part of the discrete expansion recovers the standard It\^o Formula:
\[
f(\bar{X}_t) - f(\bar{X}_0) = \int_0^t \left( \nabla f^\top b + \frac{1}{2} \text{Tr}(\sigma \sigma^\top D^2 f) \right) dr + \int_0^t \nabla f^\top \sigma \, d\bar{W}_r.
\]
\end{proof}

\section{$S$-equivalence of measures}
\label{app:se}
\begin{definition}[$S$-Equivalence of Internal Grid Densities]
Let $\mathbb{G}$ be a hyperfinite grid in ${}^*\mathbb{R}^d$ with cell volume $(\delta x)^d$. Let $p_1$ and $p_2$ be internal, non-negative probability densities on $\mathbb{G}$ (i.e. $\sum_{x \in \mathbb{G}} p_1(x)(\delta x)^d = \sum_{x \in \mathbb{G}} p_2(x)(\delta x)^d = 1$). We say that $p_1$ and $p_2$ are $S$-equivalent, denoted $p_1 \approx p_2$, if for every internal set $A \subseteq \mathbb{G}$:
\begin{equation*}
\sum_{x \in A} p_1(x) (\delta x)^d \approx \sum_{x \in A} p_2(x) (\delta x)^d.
\end{equation*}
If \(p_1,p_2\) are \(S\)-integrable, then $p_1 \approx p_2$ if and only if their induced nonstandard measures define the exact same Loeb measure on the $\sigma$-algebra of Loeb-measurable subsets of $\mathbb{G}$.
\end{definition}

Furthermore, for these $S$-integrable probability densities, a hyperfinite analogue of Scheffé’s theorem guarantees that the above condition is equivalent to convergence in the internal $L^1(\mathbb{G})$ norm:
\begin{equation*}
\|p_1 - p_2\|_{L^1(\mathbb{G})} = \sum_{x \in \mathbb{G}} |p_1(x) - p_2(x)| (\delta x)^d \approx 0.
\end{equation*}

\end{document}